\documentclass[11pt]{article}
\usepackage[margin=1in]{geometry}

\RequirePackage[colorlinks,citecolor=blue,urlcolor=blue]{hyperref}

\usepackage{amsmath,amsfonts,amssymb,amsthm}
\usepackage{latexsym,mathtools}
\usepackage{enumitem}
\usepackage{graphicx}

\numberwithin{equation}{section}
\theoremstyle{plain}
\newtheorem{theorem}{Theorem}[section]
\newtheorem{lemma}[theorem]{Lemma}
\newtheorem{proposition}[theorem]{Proposition}
\newtheorem{corollary}[theorem]{Corollary}

\theoremstyle{remark}

\newtheorem{properties}{Property}[section]
\newtheorem{remark}{Remark}[section]

\newcommand{\Reals}{\mathbb R}
\newcommand{\RealsP}{\Reals_+}

\newcommand{\bP}{\mathbb{P}}
\newcommand{\Prob}[1]{\mathbb{P} \left( #1 \right)}
\newcommand{\E}{\mathbb{E}}
\newcommand{\Var}{\text{Var}}

\newcommand{\poly}{{\sf poly}}
\newcommand{\Ind}{\mathbb{I}}
\newcommand{\Cov}{\text{\normalfont Cov}}

\DeclareMathOperator*{\argmin}{arg\,min}
\newcommand{\sgn}{\text{\normalfont sign}}

\newcommand{\dist}{\text{\normalfont {\sf dist}}}
\newcommand{\bias}{\text{\normalfont {\sf bias}}}
\newcommand{\meas}{\text{\normalfont {\sf meas}}}

\newcommand{\cA}{\mathcal{A}}
\newcommand{\cB}{\mathcal{B}}

\newcommand{\cF}{\mathcal{F}}
\newcommand{\cG}{\mathcal{G}}
\newcommand{\cH}{\mathcal{H}}

\newcommand{\cQ}{\mathcal{Q}}

\newcommand{\cS}{\mathcal{S}}

\newcommand{\cU}{\mathcal{U}}
\newcommand{\cV}{\mathcal{V}}

\newcommand{\cX}{\mathcal{X}}

\newcommand{\tN}{\tilde{N}}
\newcommand{\tW}{\tilde{W}}
\newcommand{\tX}{\tilde{X}}

\newcommand{\MSE}{\text{\normalfont MSE}}

\newcommand{\extra}{\kappa}
\newcommand{\radius}{t} 

\newcommand{\Ap}{A^\prime}
\newcommand{\Mp}{{M_{A}}}
\newcommand{\Mpp}{{M_{B}}} 

\newcommand{\Mppp}{{M_2}} 

\newcommand{\cEp}{{\Omega_1}}
\newcommand{\cEpp}{{\Omega_1}} 

\newcommand{\cEppp}{{\Omega_2}} 

\newcommand{\pert}{\varepsilon}
\newcommand{\bpert}{\boldsymbol{\pert}}

\begin{document}
	\title{Robust Max Entrywise Error Bounds for Tensor Estimation from Sparse Observations via Similarity Based Collaborative Filtering}
	\author{Devavrat Shah  \footnote{Massachusetts Institute of Technology, Cambridge, MA;
			\texttt{devavrat@mit.edu}}
			\,\, Christina Lee Yu \footnote{Cornell University, Ithaca NY;
			\texttt{cleeyu@cornell.edu}}}
	\date{}
	\maketitle

	\begin{abstract}
		Consider the task of estimating a 3-order $n \times n \times n$ tensor from noisy observations of randomly chosen entries in the sparse regime. We introduce a similarity based collaborative filtering algorithm for estimating a tensor from sparse observations and argue that it achieves sample complexity that nearly matches the conjectured computationally efficient lower bound on the sample complexity for the setting of low-rank tensors. Our algorithm uses the matrix obtained from the flattened tensor to compute similarity, and estimates the tensor entries using a nearest neighbor estimator. We prove that the algorithm recovers a finite rank tensor with maximum entry-wise error (MEE) and mean-squared-error (MSE) decaying to $0$ as long as each entry is observed independently with probability $p = \Omega(n^{-3/2 + \extra})$ for any arbitrarily small $\extra > 0$. More generally, we establish robustness of the estimator, showing that when arbitrary noise bounded by $\bpert \geq 0$ is added to each observation, the estimation error with respect to MEE and MSE degrades by ${\sf poly}(\bpert)$. Consequently, even if the tensor may not have finite rank but can be approximated within $\bpert \geq 0$ by a finite rank tensor, then the estimation error converges to ${\sf poly}(\bpert)$. Our analysis sheds insight into the conjectured sample complexity lower bound, showing that it matches the connectivity threshold of the graph used by our algorithm for estimating similarity between coordinates.
	\end{abstract}

\section{Introduction}

Tensor estimation involves the task of predicting underlying structure in a high-dimensional tensor structured dataset given only a sparse subset of observations. 
We call this ``tensor estimation'' rather than the conventional ``tensor completion'' as the goal is not only to fill missing entries but also to estimate entries whose noisy observations are available.
Whereas matrices represent data associated to two modes, rows and columns, tensors represent data associated to general $d$ modes. For example, a datapoint collected from a user-product interaction an e-commerce platform may be associated to a user, product, and date/time, which could be represented in a 3-order tensor where the three modes would correspond to users, products, and date/time. Image data is also naturally represented in a 3-order tensor format, with two modes representing the location of the pixel, and the third mode representing the RGB color components. Video data furthermore introduces a fourth mode indexing the time. 
Dynamic network data can also be represented in a tensor with one mode indexing the time and the other two modes indexing the nodes in the network.

There are many applications in which the dataset inherently has a lot of noise or is very sparsely observed. For example, e-commerce data is typically very sparse as the typical number of products a user interacts with is very small relative to the entire product catalog; furthermore the timepoints at which the user interacts with the platform may be sparse. When the dataset can be represented as a matrix, equivalent to a 2-order tensor, there has been a significant amount of research in designing practical algorithms and studying statistical limits for matrix estimation, a critical step in data pre-processing. Under conditions on uniform sampling and incoherence, the minimum sample complexity for estimation has been tightly characterized and achieved by simple algorithms. It is a natural and relevant question then to consider whether the techniques developed can extend to higher order tensors as well. 

The previous literature has primarily focused on attaining consistency with respect to the mean squared error (MSE). Unfortunately as this is aggregated over the error in the full tensor, it does not translate to consistent bounds on entrywise error, as the error on a single entry could be very large despite the MSE being small due to averaging over many entries. However, entrywise bounds are important in practice as the results of tensor estimation are often used subsequently for decisions that involve comparisons between the estimates of individual entries.

In this work we focus on attaining consistent max entrywise error bounds by extending similarity based collaborative filtering algorithms to tensor estimation. Similarity based collaborative filtering is widely used in industry due to its simplicity, interpretability, and amenability to distributed and parallelized implementations. In the analysis of our proposed algorithm we show that it achieves a sample complexity that nearly matches a conjectured lower bound for computationally efficient algorithms. Perhaps most notably, our theoretical guarantees provide high probability bounds on the maximum entrywise error of the estimate, which is significantly stronger than the typical mean squared error style bounds found in the literature for other algorithms. We also provide error bounds under arbitrary bounded noise, which has implications towards approximately low rank settings. 


\subsection{Related Literature}

Algorithms for analyzing sparse low rank matrices (equivalent to 2-order tensors) where the observations are sampled uniformly randomly have been well-studied. The algorithms consist of spectral decomposition or matrix factorization \cite{KeshavanMontanariOh10a,KeshavanMontanariOh10b,Chatterjee15}, nuclear norm minimization \cite{CandesRecht09,CandesPlan10,CandesTao10,Recht11,NegahbanWainwright11,mazumder2010spectral}, gradient descent \cite{KeshavanMontanariOh10a,KeshavanMontanariOh10b,ChenWainwright15,sun2016guaranteed,ge2016matrix}, alternating minimization \cite{JainNetrapalliSanghavi13,hardt2014understanding}, and nearest neighbor style collaborative filtering \cite{goldberg92,song2016blind,li2020blind,BorgsChayesLeeShah17,borgs2017iterative}. These algorithms have been shown to be provably consistent as long as the number of observations is $\Omega(r n \,\poly (\log n))$ for the noiseless
setting where $r$ is the rank and $n$ is the number of rows and columns \cite{KeshavanMontanariOh10a,CandesRecht09}; similar results have been attained under additive Gaussian noise \cite{KeshavanMontanariOh10b,CandesPlan10} and generic bounded noise \cite{Chatterjee15, BorgsChayesLeeShah17}. Lower bounds show that $\Omega(rn)$ samples are necessary for consistent estimation, and $\Omega(rn \log(n))$ samples are necessary for exact recovery \cite{CandesPlan10,CandesTao10}, implying that the proposed algorithms are nearly sample efficient order-wise up to the information theoretic lower bounds. 

There are results extending matrix estimation algorithms to higher order tensor estimation, assuming the tensor is low rank and that observations are sampled uniformly at random. The earliest approaches simply flatten or unfold the tensor to a matrix and subsequently apply matrix estimation algorithms \cite{liu2013tensor,gandy2011tensor,tomioka2010estimation,tomioka2011statistical}. A $d$-order tensor where each dimension is length $n$ would be unfolded to a $n^{\lfloor d/2 \rfloor} \times n^{\lceil d/2 \rceil}$ matrix, resulting in a sample complexity of $O(n^{\lceil d/2 \rceil} \,\poly(\log n) )$, significantly larger than the natural statistical lower bound that is linear with $n$ due to the model being parameterized by linear in $n$ latent variables. When $d$ is odd, for example $d = 3$ the sample complexity for this naive approach scales as $O(n^2 \,\poly(\log n))$. 

Subsequent works have improved upon this sample complexity, requiring only $\Omega(n^{3/2} \,\poly(\log n) )$ observed entries for a 3-order tensor \cite{jain2014provable, bhojanapalli2015new,yuan2016tensor,BarakMoitra16, PotechinSteurer17,MontanariSun18,xia2017polynomial, xia2017statistically}. \cite{jain2014provable, bhojanapalli2015new} analyzes the alternating minimization algorithm for exact recovery of the tensor given noiseless observations and finite rank $r = \Theta(1)$. 
\cite{BarakMoitra16, PotechinSteurer17} use the sum of squares (SOS) method, and \cite{MontanariSun18} introduces a spectral method. Both of these latter algorithms can handle noisy observations and overcomplete tensors where the rank is larger than the dimension. \cite{xia2017polynomial, xia2017statistically} furthermore characterize the minimax optimal rate for the MSE and achieve it using spectral initialization followed by power iteration.
For a general $d$-order tensor these results translate into a sample complexity scaling as $O(n^{d/2})$, improving upon $O(n^{\lceil d/2 \rceil}$.
\cite{yuan2017incoherent} prove that tensor nuclear norm minimization can recover the underlying low-rank $d$-order tensor with $O(n^{3/2} \,\poly(\log n))$ samples  in the noiseless setting; however, the algorithm is not efficiently computable as computing tensor nuclear norm is NP-hard \cite{friedland2014nuclear}. 

\cite{BarakMoitra16} conjecture that any polynomial time estimator for a 3-order tensor must require $\Omega(n^{3/2})$ samples, based on a reduction between tensor estimation for a rank-1 tensor to the random 3-XOR distinguishability problem. They argue that if using the sum of squares hierarchy to construct relaxations for tensor rank, any result that achieves a consistent estimator with fewer than $n^{3/2}$ samples will violate a conjectured hardness of random 3-XOR distinguishability. Information theoretic bounds imply that one needs at least $\Omega(d r n)$ observations to recover a $d$-order rank $r$ tensor, consistent with the degrees of freedom or number of parameters in the model. Interestingly, this implies a conjectured gap between the computational and statistically achievable sample complexities, highlighting how tensor estimation is distinctly more difficult than matrix estimation.

The majority of the results in tensor estimation provide bounds on the mean squared error, which aggregates errors across entries. In contrast our results will also provide bounds on the maximum entrywise error. There has been recent interest on developing matrix estimation methods that provide max entrywise bounds using a leave one out analysis, cf. \cite{abbe2020entrywise}, \cite{chen2019spectral}, \cite{zhong2018near}, \cite{cai2019subspace}, \cite{ma2018implicit} and \cite{DingChen20}. Subseqeuntly \cite{cai2019nonconvex, cai2020uncertainty} extended the leave one out analysis to the tensor setting to obtain entrywise error bounds for a gradient descent algorithm with spectral initialization. The analysis and algorithm in our paper is significantly different than their work, as it results from showing high probability guarantees on the similarity computation, which is akin to a spectral algorithm, followed by a nearest neighbor analysis. Our results suggests that a combination of spectral analysis and nearest neighbor smoothing can achieve entrywise consistent estimates without further gradient descent refinements. As nearest neighbor methods are still widely used in industry, understanding their theoretical performance is of interest.

\subsection{Contribution} Our results answer the following unresolved questions in the literature. 
\begin{itemize}
	\item[1.] Is there a computationally efficient estimator that can provide a consistent estimation 
	of low-rank tensor with respect to maximum entry-wise error (MEE) with minimal sample complexity 
	of $\Omega(n^{\frac32})$ in the presence of noise?
	\item[2.] Is there an extension of matrix estimation collaborative filtering algorithm for the setting of 
	tensor estimation that can provide consistent estimation with such minimal sample complexity? 
	\item[3.] Can the estimator be robust to adversarial bounded noise in the observations? 
\end{itemize}

To begin with, we 
propose an algorithm for a symmetric $3$-order tensor estimation which generalizes a 
nearest neighbor collaborative filtering algorithm for sparse matrix estimation introduced in \cite{BorgsChayesLeeShah17}. 
Na\"ively applying the matrix estimation algorithm in \cite{BorgsChayesLeeShah17} to the $n \times n^2$ matrix obtained by unfolding the $3$-order tensor would require $\Omega(n^2)$ samples, far more than the desired sample complexity of $\Omega(n^{\frac32})$. However, we argue that such a matrix obtained
from the unfolded tensor can be used, after non-trivial modification, to compute the similarities between rows accurately using $\Omega(n^{\frac32 + \extra})$ samples for any positive $\kappa > 0$. After computing these 
similarities we can achieve consistent estimation via a nearest neighbor estimator by additionally using the tensor structure. 

Specifically, we establish that the mean squared error (MSE) in the estimation converges to $0$ as long as $\Omega(n^{3/2 + \extra})$ random samples are observed for any $\extra > 0$ for tensor with rank $r = \Theta(1)$. 
We further establish a stronger guarantee that the maximum entry-wise error (MEE) converge to $0$ with high probability with similar sample
complexity of $\Omega(n^{3/2 + \extra})$. Thus, this simple iterative collaborative filtering algorithm nearly achieves the conjectured computational sample complexity lower bound of  $\Omega(n^{3/2})$ for tensor estimation. While we present the results for symmetric tensors, our method and analysis can extend to asymmetric tensors, which we discuss in Section \ref{ssec:result.disc}. 

Beyond low-rank tensors, our results hold for tensors with potentially countably infinite rank as long as they can be well approximated by a 
low-rank tensor. Specifically, if the tensor can be approximated with $\bpert \geq 0$ with respect to max-norm by a rank $r = \Theta(1)$
tensor, then the MSE converges to $\poly(\bpert)$ and MEE converges to  $\poly(\bpert)$ with high probability 
as long as $\Omega(n^{3/2 + \extra})$ random samples are observed for any $\extra > 0$. This follows as a 
consequence of the robustness property of the algorithm that we establish: 
if  arbitrary noise bounded by $\bpert \geq 0$ is added to each observation, then the 
estimation error with respect to MEE and MSE degrades by ${\sf poly}(\bpert)$.



To establish our results, the key analytic tool is utilizing certain concentration properties of a bilinear form 
arising from the local neighborhood expansion of any given coordinate for an asymmetric matrix with 
dimensions $n \times n^2$. This generalizes the analysis of a similar property 
for symmetric matrices in the prior work of \cite{BorgsChayesLeeShah17}. Specifically, establishing 
the desired concentration requires handling dependencies arising in the local neighborhood expansion
of the $3$-order tensor that was absent in the matrix setting considered in \cite{BorgsChayesLeeShah17}.
Subsequently, we require a novel analytic method compared to the prior work. In particular we believe that the proof techniques in Lemma \ref{lemma:NMN_conc} may be useful to other settings in which one may desire a tighter concentration on sums of sparse random variables. As a consequence, we also establish performance guarantees for matrix estimation for asymmetric matrices having dimensions of different order, generalizing beyond of \cite{BorgsChayesLeeShah17}.

The algorithm and analysis also sheds insight on the conjectured lower bound for $3$-order tensor. 
In particular, the threshold of $n^{3/2}$ is precisely the density of observations needed for the 
connectivity in the associated graph that is utilized to calculate similarities. If the graph is disconnected, the 
similarities can not be computed, while if the graph is connected, we are able to show that similarity calculations 
yield an excellent estimator. Understanding this relationship further remains an interesting open research direction. 

A benefit of our algorithm is that it can be implemented in a parallelized manner where the similarities between 
pair of indices are computed in parallel. This lends itself to a distributed, scalable implementation. A naive bound 
on the computational complexity of our algorithm for 3-order tensor is at most $p n^6$. As discussed in Section \ref{ssec:result.disc}, 
with use of approximate nearest neighbors, these can be further improved and made truly implementable.
\footnote{A weaker abbreviated version of this result appeared in \cite{shah2019iterative} without any proofs or discussion. Since the preliminary results, the convergence rates of the error have improved, and we have new results showing a perturbation analysis under arbitrary bounded noise, which extends our results to the approximately low rank setting. We also present a modification of the algorithm that significantly improves the overall computational complexity of the algorithm.}
\section{Preliminaries}

Tensor estimation from sparse observations hinges on an assumption that the true model exhibits low dimensional structure despite the high dimensional representation. However, there is not a unique definition of rank in the tensor setting, as natural generalizations of matrix rank lead to different quantities when extended to higher order tensors. We will focus on two commonly used definitions of tensor rank, the CP rank and the Tucker or multilinear rank.

For a $d$-order tensor $F \in \Reals^{n^d}$, we can decompose $F$ into a sum of rank-1 tensors. For example if $d = 3$, then 
\[F = \sum_{k=1}^r u_{k} \otimes v_k \otimes w_k,\]
where $\{u_k, v_k, w_k\}_{k \in [r]}$ is a collection of length $n$ vectors. The CP-rank is the minimum number $r$ such that $F$ can be written as a sum of $r$ rank-1 tensors, which we refer to as a CP-decomposition. The CP-rank may in fact be larger than the dimension $n$, and furthermore the latent vectors need not be orthogonal as is the case in the matrix setting.

An alternate notion of tensor rank is defined according to the dimension of subspaces corresponding to each mode. Let $F_{(y)}$ denote the unfolded tensor along the $y$-th mode, which is a matrix of dimension $n \times n^{d-1}$. Let columns of $F_{(y)}$ be referred to as mode $y$ fibers of tensor $F$ as depicted in Figure \ref{fig:unfolding}.
\begin{figure*} 
\centering
\includegraphics[width=5in]{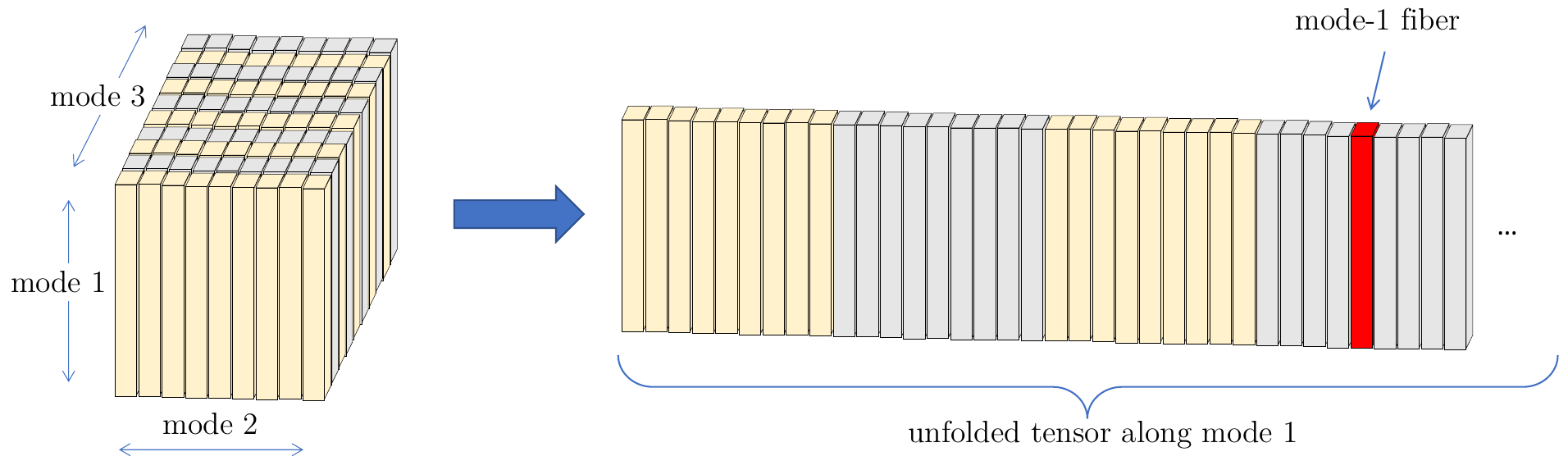}
\caption{Depicting an unfolding of a 3rd order tensor along mode 1. The columns of the resulting matrix are referred to as the mode-1 fibers of the tensor.} \label{fig:unfolding}
\end{figure*}
The Tucker rank, or multilinear rank, is a vector $(r_1, r_2, \dots r_d)$ such that for each mode $\ell \in [d]$, $r_{\ell}$ is the dimension of the column space of $F_{(y)}$. The Tucker rank is also the minimal values of $(r_1, r_2, \dots r_d)$ such that the tensor $F$ can be decomposed according to a multilinear multiplication of a core tensor $\Lambda \in \Reals^{r_1 \times r_2 \times \dots r_d}$ with latent factor matrices $Q_1 \dots Q_d$ for $Q_{\ell} \in \Reals^{n_{\ell} \times r_{\ell}}$, denoted as
\begin{align}
F = (Q_1 \otimes \cdots Q_d) \cdot (\Lambda) 
:= \sum_{{\bf k} \in [r_1]\times[r_2] \cdots \times[r_d]} \Lambda({\bf k}) Q_1(\cdot, k_1) \otimes Q_2(\cdot, k_2) \cdots \otimes Q_d(\cdot, k_d), \label{eq:tucker_decomp}
\end{align}
and depicted in Figure \ref{fig:decomposition}. The higher order SVD (HOSVD) specifies a unique Tucker decomposition in which the factor matrices $Q_1 \dots Q_d$ are orthonormal and correspond to the left singular vectors of the unfolded tensor along each mode \cite{de2000multilinear}. 
\begin{figure*}
\centering
\includegraphics[width=5in]{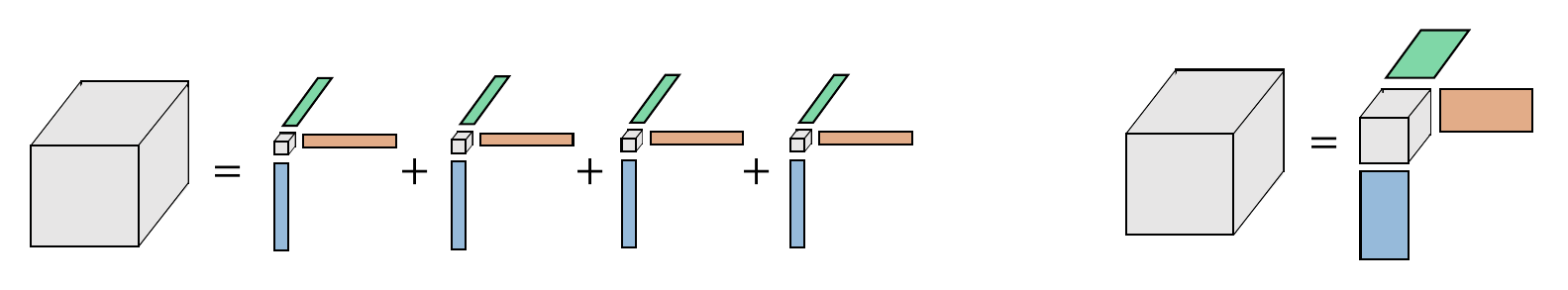}
\caption{(Left) The tensor CP-rank admits a decomposition corresponding to the sum of $r$ rank-1 tensors. \\(Right) The Tucker rank or multilinear rank $(r_1, r_2, \dots r_d)$ admits a decomposition corresponding to a multilinear multiplication of a core tensor of dimensions $(r_1, r_2, \dots r_d)$ with latent factor matrices associated to each mode.}
\label{fig:decomposition}
\end{figure*}

If the CP-rank is $r$, the Tucker-rank is bounded above by $(r,r,\dots r)$ by constructing a superdiagonal core tensor. If the Tucker rank is $(r_1, r_2, \dots r_d)$, the CP-rank is bounded by the number of nonzero entries in the core tensor, which is at most $r_1 r_2 \cdots r_d / (\max_{\ell} r_{\ell})$ \cite{de2000multilinear}. While the latent factors of the HOSVD are orthogonal, the latent factors corresponding to the minimal CP-decomposition may not be orthogonal.
%
For simplicity of presentation, we will consider a limited setting where there exists a decomposition of the tensor into the sum of orthogonal rank-1 tensors. This is equivalent to enforcing that the core tensor $\Lambda$ associated to the Tucker decomposition is superdiagonal, or equivalently enforcing that the latent factors in the minimal CP-decomposition are orthogonal. There does not always exist such an orthogonal CP-decomposition, however this class still includes all rank 1 tensors which encompasses the class of instances used to construct the hardness conjecture in \cite{BarakMoitra16}. Our results also extend beyond to general tensors as well, though the presentation is simpler in the orthogonal setting.

\section{Problem Statement and Model}

Consider an $n \times n \times n$ symmetric tensor $F$ generated as follows: For each $u \in [n]$, sample $\theta_u \sim U[0,1]$ independently. Let the true underlying tensor $F$ be described by a Lipschitz function $f$ evaluated over the latent variables, $F(u,v,w) = f(\theta_u, \theta_v, \theta_w)$ for $u, v, w \in [n]$. Without 
loss of generality, we shall assume that $\sup_{u, v, w \in [0,1]} |f(\theta_u, \theta_v, \theta_w)| \leq 1$. For example, if the coordinates of one mode of the tensor represent users or products in an e-commerce platform, one can view the latent variables associated to a coordinate as the unknown “type” of the user/product, which can be thought of as sampled i.i.d. from an underlying population distribution. The latent function $f$ would then describe the expected observed interaction between units of type $\theta_u, \theta_v$, and $\theta_w$.

Let $M$ denote the observed symmetric data tensor, and let $\Omega \subseteq [n]^3$ denote the set of observed indices. Due to the symmetry, it is sufficient to restrict the index set to triplets $(u,v,w)$ such that $u \leq v \leq w$, as the datapoint is identical for all other permutations of the same triplet. The datapoint at each of these distinct triplets $\{(u,v,w): u \leq v \leq w\}$ is observed independently with probability $p \in (0,1]$, where we assume the observation is corrupted by mean zero independent additive noise terms. For $(u,v,w) \in \Omega$,
\begin{align}\label{eq:meas.def}
M(u,v,w) & = F(u,v,w) + \epsilon_{uvw},
\end{align}
and for $(u,v,w) \notin \Omega$, $M(u,v,w) = \star$\footnote{The notation of $\star$ is used to denote the missing observation. When convenient, we shall replace $\star$ by $0$ for the purpose of computation. }. We shall assume that $|M(u,v,w)| \leq 1$ with probability $1$. We allow $\epsilon_{uvw}$ to have different distributions for different distinct triplets $(u,v,w)$ as long as it is uniformly bounded so as to satisfy the boundedness constraints on $|M(u,v,w)|$ and $|F(u,v,w)|$. When the observations $M(u,v,w)$ are binary, this model is equivalent to the $3$-uniform simple lipschitz hypergraphon in \cite{balasubramanian2021nonparametric}, which states a generative model for hypergraphs where the hyperedges consist of size 3 vertex sets. In this setting, $F(u,v,w) \in [0,1]$ would represent the probability of observing the hyperedge $(u,v,w)$, and $M(u,v,w) \in \{0,1\}$ would indicate the presence of the hyperedge $(u,v,w)$. The noise term $\epsilon_{uvw}$ is clearly bounded since the observations are binary. For any application in which the observations $M(u,v,w)$ are bounded, then $F(u,v,w) = \E[M(u,v,w)]$ would also be bounded, such that the boundedness on the noise $\epsilon_{uvw}$ would be reasonable. However, for applications in which $M(u,v,w)$ may not be bounded, or the almost sure bound is very large, we can extend our analysis to allow for sub-Gaussian noise $\epsilon_{uvw}$ rather than uniformly bounded noise, which is further discussed in section \ref{sec:proof}.

The goal is to recover the underlying tensor $F$ from the incomplete noisy observation $M$ so that the mean squared error (MSE) is small, where MSE for an estimate $\hat{F}$ is defined as
\begin{align} \label{eq:MSE_def}
\text{MSE}(\hat{F}) := \E\left[\tfrac{1}{n^3} \textstyle\sum_{(u,v,w) \in [n]^3} (\hat{F}(u,v,w) - F(u,v,w))^2\right].
\end{align}
We will also be interested in the maximum entry-wise error (MEE) defined as 
\begin{align} \label{eq:MEE_def}
\|F - \hat{F}\|_{\max} := \max_{(u,v,w) \in [n]^3} |\hat{F}(u,v,w) - F(u,v,w)|.
\end{align}

\subsection{Finite spectrum}
Consider the setting where the function $f$ has finite spectrum. That is, 
\[f(u, v, w) = \textstyle\sum_{k=1}^r \lambda_k q_k(\theta_u) q_k(\theta_v) q_k(\theta_w),\] 
where $r = \Theta(1)$ and $q_k(.)$ denotes the orthonormal $\ell_2$ eigenfunctions, satisfying
$\int_0^1 q_k(\theta)^2 d\theta = 1$
and
$\int_0^1 q_k(\theta) q_h(\theta) d\theta = 0 \text{ for } k \neq h$.
Assume that the eigenfunctions are bounded, i.e. $|q_k(\theta)| \leq B$ for all $k \in [r]$.

Let $\Lambda$ denote the diagonal $r \times r$ matrix where $\Lambda_{kk} = \lambda_k$.
 Let $Q$ denote the $r \times n$ matrix where $Q_{ka} = q_k(\theta_a)$. Let $\cQ$ denote the $r \times \binom{n}{2}$ matrix where $\cQ_{kb} = q_k(\theta_{b_1}) q_k(\theta_{b_2})$ for some $b \in \binom{n}{2}$ that represents the pair of vertices $(b_1,b_2)$ for $b_1 < b_2$. 
 The finite spectrum assumption for $f$ implies that the sampled tensor $F$ is such that,
\[F = \textstyle\sum_{k=1}^r \lambda_k (Q^T e_k) \otimes (Q^T e_k) \otimes (Q^T e_k).\]
That is, $F$ has CP-rank at most $r$. In above and in the remainder of the paper, $e_k$ 
denotes a vector with all $0$s but $k$th entry being $1$ of appropriate dimension (here it is $r$).

\subsection{Approximately finite spectrum} In general, $f$ may not have finite spectrum, e.g. a generic
analytic function $f$. For such a setting, we shall consider $f$ with approximately finite spectrum. Specifically, 
a function $f: [0,1]^3 \to \Reals$, it is said to have $\bpert$-approximate finite spectrum with rank $r$ for $\bpert \geq 0$ 
if there exists a symmetric function $f_r: [0,1]^3 \to \Reals$ such that 
\begin{align}\label{eq:eps.rank}
&\sup_{\theta_u, \theta_v, \theta_w \in [0,1]} | f(\theta_u, \theta_v, \theta_w) - f_r(\theta_u, \theta_v, \theta_w)| \leq \bpert \nonumber \\
&F_r(u,v,w) = f_r(\theta_u, \theta_v, \theta_w) = \textstyle\sum_{k=1}^r \lambda_k q_k(\theta_u) q_k(\theta_v) q_k(\theta_w),
\end{align}
where $r = \Theta(1)$ and $q_k(.)$ denotes the orthonormal $\ell_2$ eigenfunctions as before. That is, 
they satisfy $\int_0^1 q_k(\theta)^2 d\theta = 1$,  $\int_0^1 q_k(\theta) q_h(\theta) d\theta = 0 \text{ for } k \neq h$
and $|q_k(\theta)| \leq B$ for all $k \in [r]$.

The above describe property of $f$ implies that the sampled tensor $F$ is has $\bpert$-approximate rank $r$ such that $F_r =  \textstyle\sum_{k=1}^r \lambda_k (Q^T e_k) \otimes (Q^T e_k) \otimes (Q^T e_k)$ and
\begin{align*}
 \| F -F_r \|_{\max} & \leq \bpert.
 \end{align*}

\subsection{Extensions Beyond Orthogonal CP-rank}

The orthogonality conditions on our latent variable decomposition imply that the tensor $F$ can be written as a sum of $r$ rank-1 tensors, where the latent factors are approximately orthogonal. Alternately, this would suggest a Tucker decomposition of the tensor where the core tensor is superdiagonal. Not all tensors admit an orthogonal CP-rank decomposition, but this assumption has also been used in the literature as in \cite{jain2014provable}. This assumption of the existence of an orthogonal CP-rank decomposition can be relaxed as the main property that our algorithm and analysis use is the orthogonal decomposition of the unfolded tensors along each mode. Our algorithm and analysis will still extend to tensors with Tucker rank bounded by $(r,r,r)$. For a general $(r,r,r)$ Tucker rank tensor, we would instead carry out the analysis with respect to the latent orthogonal factors corresponding to the SVD of the unfolded tensor into a matrix, and the algorithm would use the same procedure to estimate similarities along each of the three modes separately. The presentation is stated for the orthogonal symmetric setting for simplicity.

\subsection{Comparision of Assumptions with Literature} 
In the decomposition of the model $f$ when it has finite spectrum, we assume that the functions $q_k$ are orthonormal. This induces a
decomposition of tensor $F$ in terms of $Q \in \Reals^{r \times n}$ with respect to the sampled latent features $\theta \sim U[0,1]$. 
The rank $r$ of the underlying decomposition is assumed to be $\Theta(1)$. We compare and contrast these with those assumed in the tensor estimation literature. 

Most literature on tensor estimation do not impose a distribution on the underlying latent variables, but instead assume deterministic `incoherence' style conditions on the latent singular vectors associated to the underlying tensor decomposition. This plays a similar role to our combined assumption of $q_k$ being orthonormal and the latent variables sampled from a uniform distribution so that the mass in the singular vector matrix is roughly uniformly spread. 
For example, the notion of incoherence used in \cite{BarakMoitra16} imposes that the entries of the latent factors are bounded by a constant when the norm of the latent factor vectors scales as $\Theta(\sqrt{n})$. As $\theta_u \sim U[0,1]$, it holds that the latent factor vector $(q_k(\theta_1), q_k(\theta_2), \dots q_k(\theta_n))$ will have norm scaling as $\sqrt{n}$. Due to the the boundedness assumption that $|q_k(\theta)| \leq B$, our model will satisfy incoherence as defined in \cite{BarakMoitra16}.
Some of the literature on tensor estimation allows for overcomplete tensors, i.e. $r > n$. While our finite spectrum 
setup requires $r = \Theta(1)$, the approximately finite spectrum can allow for potentially countably infinite spectrum but
with sharply decaying spectrum so that it has $\bpert$-approximate rank being $r = \Theta(1)$. 

In order to establish our result for the approximately finite spectrum setting, we  perform a perturbation analysis wherein each observed entry is perturbed arbitrarily bounded by $\bpert$ in magnitude: we shall establish that the resulting estimation error is changed by ${\sf poly}(\bpert)$, both with respect to the MSE and max-norm. That is, with respect to arbitrary bounded noise in the observations, we are able to characterize the error induced by our method, which is
of interest in its own right. 

We remark on the Lipschitz property of $f$: the Lipschitz assumption implies that the tensor is ``smooth'', and thus there are sets of rows and columns that are similar to one another. As our algorithm is based on a nearest neighbor style approach we need that for any coordinate $u$, there is a significant mass of other coordinates $a$ that are similar to $u$ with respect to the function behavior. Other regularity conditions beyond Lipschitz that would also guarantee sufficiently many ``nearest neighbors'' would lead to similar results for our algorithm. Lipschitzness also implies approximate low rankness as a Lipschitz function can be approximated by a piecewise constant function, where the number of pieces would then upper bound the rank.
\section{Algorithm}


The algorithm is a nearest neighbor style  in which the first phase is to estimate a distance function between coordinates, denoted $\dist(u,a)$ for all $(u,a) \in [n]^2$. Given the similarities, for some threshold $\eta$, the algorithm estimates by averaging datapoints from coordinates $(a,b,c)$ for which $\dist(u,a) \leq \eta$, $\dist(v,b) \leq \eta$, and $\dist(w,c) \leq \eta$. 

The entry $F(a,b,c)$ depends on a coordinate $a$ through its representation in the eigenspace, given by $Q e_a$. Therefore $f(a,b,c) \approx f(u,v,w)$ as long as $Q e_u \approx Q e_a$, $Q e_v \approx Q e_b$, and $Q e_w \approx Q e_c$. Ideally we would like our distance function $\dist(u,a)$ to approximate $\|Q e_u - Q e_a\|_2$, but these are hidden latent features that we do not have direct access to.

Let's start with a thought experiment supposing that the density of observations were $p = \omega(n^{-1})$ and the noise variance is $\sigma^2$ for all entries. For a pair of coordinates $u$ and $a$, the expected number of pairs $(b,c)$ such that both $(u,b,c)$ and $(a,b,c)$ are observed is on the order of $p^2 n^2 = \omega(1)$. For fixed $\theta_a, \theta_u$, and for randomly sampled $\theta_b, \theta_c$, the expected squared difference between the two corresponding datapoints reflects the distance between $Q e_a$ and $Q e_u$ along with the overall level of noise,
\begin{align*}
&\E[(M(a,b,c) - M(u,b,c))^2 ~|~ \theta_a, \theta_u] \\
&= \E[(F(a,b,c) - F(u,b,c))^2 ~|~ \theta_a, \theta_u] + \E[\epsilon_{abc}^2 + \epsilon_{ubc}^2] \\
&= \E[(\textstyle\sum_k \lambda_k (q_k(\theta_a) - q_k(\theta_u)) q_k(\theta_b) q_k(\theta_c))^2~|~ \theta_a, \theta_u] + 2\sigma^2 \\
&= \E[\textstyle\sum_k \lambda_k^2 (q_k(\theta_a) - q_k(\theta_u))^2 q_k(\theta_b)^2 q_k(\theta_c)^2~|~ \theta_a, \theta_u] + 2\sigma^2 \\
&= \textstyle\sum_k \lambda_k^2 (q_k(\theta_a) - q_k(\theta_u))^2 + \sigma^2 \\
&= \|\Lambda Q (e_a - e_u)\|_2^2 + 2\sigma^2,
\end{align*}
where we use the fact that $q_k(\cdot)$ are orthonormal. This suggests that approximating $\dist(u,a)$ with the average squared difference between datapoints corresponding to pairs $(b,c)$ for which both $(u,b,c)$ and $(a,b,c)$ are observed.

This method does not attain the $p = n^{-3/2}$ sample complexity, as the expected number of pairs $(b,c)$ for which $(a,b,c)$ and $(u,b,c)$ are both observed will go to zero for $p = o(n^{-1})$. 
This limitation arises due to the fact that when $p = o(n^{-1})$, the observations are extremely sparse. Consider the $n \times \binom{n}{2}$ ``flattened'' matrix of the tensor where row $u$ correspond to coordinates $u \in [n]$, and columns correspond to pairs of indices, e.g. $(b, c) \in [n] \times [n]$ with $b \leq c$. For any given row $u$, there are very few other rows that share observations along any column with the given row $u$, i.e. the number of `neighbors' of any row index is few. If we wanted to exploit the intuition of the above simple calculations, we have to somehow enrich the neighborhood. We do so by constructing a graph using the non-zero pattern of the matrix as an adjacency matrix. This mirrors the idea from \cite{BorgsChayesLeeShah17} for matrix estimation, which approximates distances by comparing expanded depth $2\radius+1$ local neighborhoods in the graph representing the sparsity pattern of the unfolded or flattened tensor. In particular, we will construct a statistic $\dist(u,a)$ such that with high probability it concentrates around $d(u,a)$ for 
\begin{align}
	d(\theta_u,\theta_a) & = \|\Lambda^{2\radius+1} Q (e_u - e_a)\|_2^2 = \sum_{k=1}^r \lambda_k^{4\radius+2} (q_k(\theta_u) - q_k(\theta_a))^2.
\end{align}
As $F(u,v,w) = f(\theta_u, \theta_v, \theta_w) = \sum_{k=1}^r \lambda_k q_k(\theta_u) q_k(\theta_v) q_k(\theta_w)$, we can show that if $d(\theta_u, \theta_a), d(\theta_v,\theta_b),$ and $d(\theta_w,\theta_c)$ are small, then $F(u,v,w)$ will be close to $F(a,b,c)$. The remaining challenge thus how to approximate $d(u,a)$. Consider a length $2\radius$ path in the bipartite graph from $u$ to $a$, denoted by $(u,e^1,x^1,e^2,x^2, \dots x^{t-1},e^t,a)$, where $x^1, \dots x^{t-1}$ are distinct coordinates in $[n] \setminus \{u,a\}$, and $e^1, \dots e^t$ consist of pairs in $[n]^2$ such that the coordinates represented in these pairs are distinct from each other as well as $\{u,a,x^1, \dots x^{t-1}\}$. Let us denote the pair $e^i = (e^i_1, e^i_2)$. Then the product of weights along this path in expectation conditioned on $\theta_u,\theta_a$ is equal to
\begin{align}
&\E\left[\left.M(u,e^1_1,e^1_2) \left(\prod_{i=1}^{t-1} M(e^i_1, e^i_2, x^i) M(x^i,e^{i+1}_1, e^{i+1}_2) \right) M(e^t_1, e^t_2,a) ~\right|~ \theta_u,\theta_a\right] \nonumber \\
&=\E\left[\left.f(\theta_u,\theta_{e^1_1}, \theta_{e^1_2}) \left(\prod_{i=1}^{t-1} f(\theta_{e^i_1}, \theta_{e^i_2},\theta_{x^i}) f(\theta_{x^i},\theta_{e^{i+1}_1}, \theta_{e^{i+1}_2}) \right) f(\theta_{e^t_1}, \theta_{e^t_2},\theta_a) ~\right|~ \theta_u,\theta_a\right] \nonumber \\
&=\sum_{k=1}^r \lambda_k^{2t} q_k(\theta_u) q_k(\theta_a) = e_u^T Q^T \Lambda^{2t} Q e_a. \label{eq:exp_prod_wts}
\end{align}
Therefore, the product of weights along the path connecting $u$ to $a$ is a good proxy of quantity $e_u^T Q^T \Lambda^{2t} Q e_a$, for paths which do not revisit coordinates. The algorithm first constructs the local neighborhood of depth $2t$ centered at each coordinate $u$, and then connects these neighborhoods to form paths of length $4t+2$ in total, where $t$ is chosen such that for every $u,a$ there are sufficiently many paths used to construct the statistic $\dist(u,a)$ to guarantee concentration.
The tensor setting requires an important modification of how one constructs the local breadth-first-search (BFS) trees due to the shared latent variables across different modes, as described in step 3 below. 


\subsection{Formal Description} \label{sec:algo_formal}

We provide a formal description of the algorithm below. The crux of the algorithm is to compute similarity between any pair of indices using the matrix
obtained by flattening the tensor, and then using a nearest neighbor estimator using these similarities between indices over the tensor structure. 
Details are as follows. 

\medskip
\noindent{\bf Step 1: Sample Splitting.} Let us assume for simplicity of the analysis that we obtain $2$ independent fresh observation sets of the data, $\cEp$ and $\cEppp$. Tensors $M_1$ and $\Mppp$ contain information from the subset of the data in $M$ associated to $\cEp$ and $\cEppp$ respectively. $M_1$ is used to compute pairwise similarities between coordinates, and $\Mppp$ is used to average over datapoints for the final estimate. Furthermore, we take the coordinates $[n]$ and split it into two sets, $[n] = \{1,2, \dots, n/2\} \cup \{n/2 + 1, n/2 + 2, \dots n\}$. Without loss of generality, let's assume that $n$ is even. Let $\cV_A$ denote the set of coordinate pairs within set 1 consisting of distinct coordinates, i.e. $\cV_A = \{(b,c)\in[n/2]^2 ~s.t.~ b < c\}$. Let $\cV_B$ denote the set of coordinate pairs within set 2 consisting of distinct coordinates, i.e. $\cV_B = \{(b,c)\in ([n] \setminus [n/2])^2 ~s.t.~ b < c\}$. The sizes of $|\cV_A|$ and $|\cV_B|$ are both equal to $\binom{n/2}{2}$. We define $\Mp$ to be the $n$-by-$\binom{n/2}{2}$ matrix taking values $\Mp(a,(b,c)) = M_1(a,b,c)$, where each row corresponds to an original coordinate of the tensor, and each column corresponds to a pair of coordinates $(b,c) \in \cV_A$ from the original tensor. We define $\Mpp$ to be the $n$-by-$\binom{n/2}{2}$ matrix taking values $\Mpp(a,(b,c)) = M_1(a,b,c)$, where each row corresponds to an original coordinate of the tensor, and each column corresponds to a pair of coordinates $(b,c) \in \cV_B$ from the original tensor. A row-column pair in the matrix corresponds to a triplet of coordinates in the original tensor. We will use matrices $\Mp$ and $\Mpp$ to compute similarities or distances between coordinates, and we use tensor $\Mppp$ to compute the final estimates via nearest neighbor averaging. The data in $\Mp$ is used to construct depth $2t$ local neighborhoods rooted at each coordinate $u$, and $\Mpp$ is used to connect the neighborhoods to form $2t+2$ length paths between any two coordinates $u$ and $a$, which are then used to estimate the similarity between $u$ and $a$. This approach is akin to the technique of ``sprinkling'' used in random graph analysis, in which we first analyze local neighborhoods formed with the edges in $\Mp$, and then ``sprinkle'' the edges in $\Mpp$ to connect these neighborhoods and argue that there are sufficiently many paths then that connect any two coordinates $u$ and $v$.

\medskip
\noindent{\bf Step 2: Construct Bipartite Graph from $\cEp, \Mp$.}
We define a bipartite graph corresponding to the flattened matrix $\Mp$.  Construct a graph with vertex set $[n] \cup \cV_A$. There is an edge between vertex $a \in [n]$ and vertex $(b,c) \in \cV_A$ if $(a,b,c) \in \cEp$, and the corresponding weight of the edge is $M_1(a,b,c)$. Recall that we assumed a symmetric model such that triplets that are permutations of one another will have the same data entry and thus the same edge weight in the associated graph. Figure \ref{fig:example}(a) provides a concrete example of a bipartite graph constructed from tensor observations.

\medskip
\noindent{\bf Step 3: Expanding the Neighborhood.}
Consider the graph constructed from $\cEp, \Mp$. For each vertex $u \in [n]$, we construct a breadth first  search (BFS) tree rooted at vertex $u$ such that the vertices for each depth of the BFS tree consists only of new and previously unvisited coordinates, i.e. if vertex $a \in [n]$ is first visited at depth 4 of the BFS tree, then no vertex corresponding to $(a,b)$ for any $b \in [n]$ can be visited in any subsequent depths greater than 4. Similarly, if $(a',b')$ is visited in the BFS tree at depth 3, then vertices that include either of these coordinates, i.e. $a'$, $b'$, $(a',c)$, or $(b',c)$ for any $c \in [n]$,  can not be visited in subsequent depths greater than 3. This restriction is only across different depths; we allow $(a,b)$ and $(a,c)$ to be visited at the same depth of the BFS tree.

There may be multiple valid BFS trees due to different ordering of visiting edges at the same depth. For example, if a vertex at depth $s$ has edges to two different vertices at depth $s-1$ (i.e. two potential parents), only one of the edges can be chosen to maintain the tree property, but either choice is equally valid. Let us assume that when there is more than one option, one of the valid edges are chosen uniformly at random. Figure \ref{fig:example}(c) shows valid BFS trees for a bipartite graph constructed from an example tensor. 

The graph is bipartite so that each subsequent layer of the BFS tree alternates between the vertex sets $[n]$ and $\cV_A$. Consider a valid BFS tree rooted at vertex $u \in [n]$ which respects the constraint that no coordinate is visited more than once. We will use $\cU_{u,s} \subseteq \cV_A$ to denote the set of vertices at depth $(2s-1)$ of the BFS tree, and we use $\cS_{u,s} \subseteq [n]$ to denote the set of vertices at depth $2s$ of the BFS tree. Let $\cB_{u,s} \subset [n] \cup \cV_A$ denote the set of vertices which are visited in the first $s$ layers of the BFS tree,
\[\cB_{u,s} = \cup_{h \in \lfloor s/2 \rfloor} \cS_{u,h} \cup_{l \in \lceil s/2 \rceil} \cU_{u,l}.\]
We will overload notation and sometimes use $\cB_{u,s}$ to denote the subset of coordinates in $[n]$ visited in the first $s$ layers of the BFS tree, including both visited single coordinate vertices or coordinates in vertices $\cV_A$, i.e.
\begin{align*}
\cB_{u,s} = \cup_{h \in \lfloor s/2 \rfloor} \cS_{u,h} 
\cup \left\{x \in [n] ~\text{s.t.}~ \exists (y,z) \in \cup_{l \in \lceil s/2 \rceil} \cU_{u,l} \text{ satisfying } x \in \{y,z\}\right\}.
\end{align*}
Let $\cG(\cB_{u,s})$ denote all the information corresponding to the subgraph restricted to the first $s$ layers of the BFS tree rooted at $u$. This includes the vertex set $\cB_{u,s}$, the latent variables $\{\theta_a\}_{a \in \cB_{u,s}}$ and the edge weights $\{M_1(a,b,c)\}_{a,(b,c) \in \cB_{u,s}}$.

We define neighborhood vectors which represent the different layers of the BFS tree. Let $N_{u,s} \in [0,1]^n$ be associated to set $\cS_{u,s}$, where the $a$-th coordinate is equal to the product of weights along the path from $u$ to $a$ in the BFS tree for $a \in \cS_{u,s}$. Similarly, let $W_{u,s} \in [0,1]^{\cV_A}$ be associated to set $\cU_{u,s}$, where the $(b,c)$-th coordinate is equal to the product of weights along the path from $u$ to $(b,c)$ in the BFS tree for $(b,c) \in \cU_{u,s}$. For $a \in [n]$, let $\pi_u(a)$ denote the parent of $a$ in the valid BFS tree rooted at vertex $u$. For $(b,c) \in \cV_A$, let $\pi_u(b,c)$ denote the parent of $(b,c)$ in the BFS tree rooted at vertex $u$. We can define the neighborhood vectors recursively,
\begin{align*}
N_{u,s}(a) &= \Mp(a, \pi_u(a)) W_{u,s}(\pi_u(a)) \Ind_{(a \in \cS_{u,s})} \\
W_{u,s}(b,c) &= \Mp(\pi_u(b,c), (b,c)) N_{u,s-1}(\pi_u(b,c)) \Ind_{((b,c) \in \cU_{u,s})}
\end{align*}
and $N_{u,0} = e_u$. Let $\tN_{u,s}$ denote the normalized vector $\tN_{u,s} = N_{u,s} / |\cS_{u,s}|$ and let $\tW_{u,s}$ denote the normalized vector $\tW_{u,s} = W_{u,s} / |\cU_{u,s}|$. 
Figure \ref{fig:example}(d) illustrates the neighborhood sets and vectors for a valid BFS tree. 
Recall from Eq. \eqref{eq:exp_prod_wts} that conditioned on $\theta_u$ and $\theta_a$, $\E[N_{u,s}(a) ~|~\theta_u, \theta_a] = \Prob{a \in \cS_{us}} e_u^T Q^T \Lambda^{2s} Q e_a$. Furthermore the event that $a \in \cS_{us}$ only depends on the presence of the edges as determined by the Bernoulli uniform sampling, and is thus independent from the latent variable $\theta_a$. We will show in a subsequent Lemma \ref{lemma:nhbrhd_vectors} that $e_k^T Q \tN_{u,\radius} \approx e_k^T \Lambda^{2\radius} Q e_{u}$, implying that the neighborhood vector $\tN_{u,\radius}$, which is constructed from products of weights over length $2s$ paths originating at $u$, is a statistic that is approximates $\Lambda^{2\radius} Q e_{u}$. Similar calculations show that $\E[W_{u,s}(b,c) ~|~\theta_u, \theta_b,\theta_c] = \Prob{(b,c) \in \cU_{us}} \sum_{k=1}^r \lambda_k^{2\radius-1} q_k(\theta_u) q_k(\theta_b) q_k(\theta_c)$.

\begin{figure*} 
  \centering
    \includegraphics[width=0.65\textwidth]{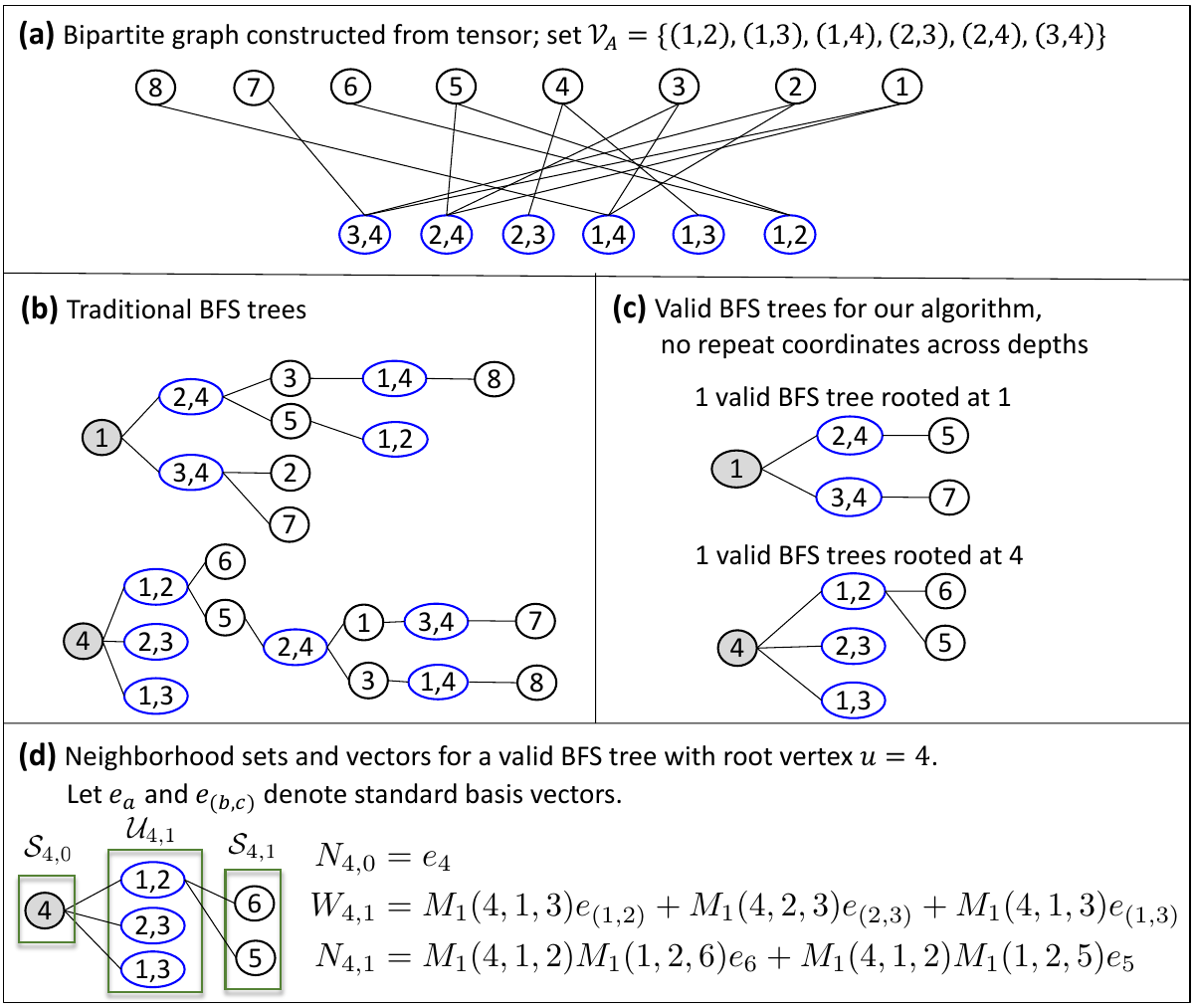}
  \caption[caption]{ Consider a symmetric 3-order tensor with $n=8$,  and the observation set
	$\Omega_1 = \{(1,2,4), (1,2,5), (1,2,6), (1,3,4), (1,4,8), (2,3,4), (2,4,5), $\\$(2,5,6), (3,4,7), (3,5,6)
	\}$. 
Figure {\bf (a)} depicts the bipartite graph constructed from this set of observations. Weights would be assigned to edges based on the value of the observed entry in the tensor $M_1$.
Figure {\bf (b)} depicts the traditional notion of the BFS tree rooted at vertices 1 and 4. Vertices at layer/depth $s$ correspond to vertices with shortest path distance of $s$ to the root vertex. 
Figure {\bf (c)} depicts valid BFS trees for our algorithm, which imposes an additional constraint that coordinates cannot be repeated across depths. For the BFS tree rooted at vertex 1, edges $((2,4),3)$ and $((3,4),2)$ are not valid, as coordinates 2 and 3 have both been visited in layer 2 by the vertices $(2,4)$ and $(3,4)$. For the BFS tree rooted at vertex 4, edge $(5,(2,4))$ is not valid as coordinate 2 has been visited in layer 2 by the vertex $(2,3)$ and coordinate 4 has veen visited in layer 1 by the root vertex 4.
Figure {\bf (d)} depicts the sets $\cS_{u,s}$ and $\cU_{u,s}$ along with the neighborhood vectors $N_{u,s}$ and $W_{u,s}$ for a specific valid BFS tree rooted at vertex $u = 4$.} \label{fig:example}
\end{figure*}

\medskip
\noindent{\bf Step 4: Computing the distances using $\Mpp$.} Let 
\begin{align}\label{eq:t}
\radius & = \Big\lceil \frac{\ln(n)}{2\ln(p^2 n^3)} \Big\rceil.
\end{align}

A heuristic for the distance would be 
\begin{align}\label{eq:dist_heur}
\dist(u,v) &\approx \frac{1}{|\cV_B| p^2} (\tN_{u,t} - \tN_{v,t}) \Mpp \Mpp^T (\tN_{u,t} - \tN_{v,t}) \\
&= \frac{1}{|\cV_B| p^2} \sum_{(\alpha,\beta) \in \cV_B} \sum_{a, b \in [n]^2} (\tN_{u,t}(a) - \tN_{v,t}(a)) \Mpp(a,\alpha,\beta)  \Mpp(b,\alpha,\beta) (\tN_{u,t}(b) - \tN_{v,t}(b)) \nonumber
\end{align}

For technical reasons that facilitate cleaner analysis, we use the following distance calculations. There are two deviations from the equation in \eqref{eq:dist_heur}. First we exclude $a = b$ from the summation. Second we exclude coordinates for $\alpha$ or $\beta$ that have been visited previously in $\cB_{u,2t}$ or $\cB_{v,2t}$. 
Define distance as
\begin{align}
\dist(u,v) &= (Z_{uu} + Z_{vv} - Z_{uv} - Z_{vu}), \label{eq:dist}\\ 
Z_{uv} &= \frac{1}{|\cV_B(u,v,t)| p^2 |\cS_{u,t}| |\cS_{v,t}|} \sum_{(\alpha, \beta) \in \cV_B(u,v,t)} T_{uv}(\alpha, \beta), \nonumber \\
\cV_B(u,v,t) &= \{(\alpha,\beta) \in \cV_B ~s.t.~ \alpha \notin \cB_{u,2t} \cup \cB_{v,2t}, \beta \notin \cB_{u,2t} \cup \cB_{v,2t}\}, \nonumber \\
T_{uv}(\alpha, \beta) &= \sum_{a \neq b \in [n]} N_{u,t}(a) N_{v,t}(b)  \Mpp(a, (\alpha, \beta)) \Mpp(b, (\alpha, \beta)). \label{eq:dist_T}
%
%
\end{align}

We will show in Lemma \ref{lemma:dist} that $\dist(u,v) \approx d(u,v) := \|\Lambda^{2\radius+1} Q (e_u - e_v)\|_2^2$. The estimate is constructed by averaging over the product of weights over paths from $u$ to $v$, where the term $N_{u,t}(a) N_{v,t}(b)  \Mpp(a, (\alpha, \beta)) \Mpp(b, (\alpha, \beta))$ in Eq. \eqref{eq:dist_T} is the product of weights over the path that goes from $u$ to $a$ to $(\alpha,\beta)$ to $b$ to $v$, as $N_{u,t}(a)$ represents the products of weights on the path from $u$ to $a$ and $N_{v,t}(b)$ represents the products of weights on the path from $b$ to $v$. The parameter $t$ is chosen such that there are sufficiently many paths that we are averaging over in order to reduce the noise. In particular, the choice of $t \geq \ln(n)/2\ln(p^2n^3)$ implies that $|\cS_{u,t}| \geq (p^2 n^3)^t = \Omega(n^{1/2})$, such that the number of paths the estimator averages over is approximately $n^2 p^2 |\cS_{u,t}| |\cS_{v,t}| = \Omega(p^2 n^3)$. This rough calculation highlights that $p$ must be $\omega(n^{-3/2})$ to guarantee that the number of paths used to compute $\dist(u,v)$ is increasing with $n$.

\medskip
\noindent{\bf Step 5: Averaging datapoints to produce final estimate.}
Let $\cEppp_{uvw}$ denote the set of indices $(a,b,c)$ such that $a \leq b \leq c$, $(a,b,c) \in \cEppp$, and the estimated distances {$\dist(u,a)$, $\dist(v,b)$, $\dist(w,c)$} are all less than some chosen threshold parameter $\eta$. The final estimate averages the datapoints corresponding to indices in $\cEppp_{uvw}$, 
\begin{align}\label{eq:estimate}
\hat{F}(u,v,w) & = \tfrac{1}{|\cEppp_{uvw}|} \textstyle\sum_{(a,b,c) \in \cEppp_{uvw}} \Mppp(a,b,c).
\end{align}

\subsection{Difference between tensor and matrix setting}

The modifications in the construction of the breadth-first-search (BFS) tree for the tensor setting relative to the matrix setting are critical to the analysis. If we simply considered the classical construction of a BFS tree in the associated bipartite graph (as the matrix setting uses), this would lead to higher variance and bias due to the correlations of vertices sharing common latent variables associated to the same underlying coordinates of the tensor. Alternatively, if one constructed a BFS tree by not allowing any coordinate of the tensor to be visited more than once, this would also lead to suboptimal results as it would throw away too many entries, limiting the computed statistic to only order $n$ data points. Our final algorithm, which allows for vertices with shared coordinates in the same depth of the BFS but not across different depths, is carefully chosen in order to break dependencies across different depths of the BFS tree, while still allowing for sufficient expansion in each depth. 

To extend the algorithm to $d$-order tensors for $d > 3$, we will compute similarities between $u,v$ via a similar computation as described in Steps 2-4 above, except it would be applied to the $n\times n^{d-1}$ matrix and associated bipartite graph corresponding to an unfolding of the tensor. Given the similarity estimates for pairwise coordinates, the final estimate would result from a standard nearest neighbor estimator over the high dimensional tensor. The primary part of the proof that would need to be modified is the analysis of the neighborhood vectors $N_{u,s}$ and $W_{us}$ in step 3, which may involve constraining the growth of the BFS trees such that they extend deep enough before exhausting the visited coordinates.

\section{Main Result} \label{sec:thm}

We provide an upper bound on the mean squared error (MSE) as well as the max entry-wise error (MEE) for the algorithm, showing that both the MSE and the MEE converge to zero as long as $p = n^{-3/2 + \extra}$ for some $\extra > 0$.  Our result implies that the simple variant of collaborative filtering algorithm based on estimating similarities produces a consistent estimator when 
the tensor latent function has finite spectrum or low rank. Further we show that it is robust to arbitrary, additive perturbation in that the estimation error increases gracefully in the amount of perturbation. To the best of our knowledge, such robustness to arbitrary bounded additive noise with respect to
max-norm estimation is first of its kind in the literature on tensor estimation. 

\subsection{Finite spectrum} We establish consistency of our estimator with respect to \text{MSE} and max-norm error of the algorithm
when the underlying $f$ has finite spectrum, i.e. rank $r$ model with $r = \Theta(1)$. 

\begin{theorem}\label{main:thm}
We assume that the function $f$ is rank $r$, $L$-Lipschitz and that $\theta \sim U[0,1]$.
Assume that $p=n^{-3/2 + \extra}$ for some $\extra \in (0,\frac12)$. Let $\radius$ be defined according to \eqref{eq:t}. 
For any arbitrarily small $\psi \in (0, \min(\extra,\frac38))$, choose the threshold
\begin{align*}
\eta &= \Theta\left(n^{- (\extra - \psi)}\right).
\end{align*}
The algorithm produces estimates so that, 
\begin{align*}
\text{MSE} &= O(n^{ - (\extra - \psi)}) = O\left(\frac{n^{\psi}}{(p^2 n^3)^{1/2}}\right),
\end{align*}
and
\begin{align*}
\|F - \hat{F}\|_{\max} & = O(n^{ - (\extra - \psi)/2}), 
\end{align*}
with probability $1 - O\left(n^4 \exp(-\Theta(n^{2\psi}))\right)$. 
\end{theorem}

The error bounds in Theorem \ref{main:thm} imply that our estimator is consistent as long as $p = n^{-3/2 + \extra}$ for some $\extra > 0$, with a MSE that scales as $O(1/p n^{3/2})$. The threshold of $p = \Omega(n^{-3/2})$ is optimal, and furthermore this requirement is precisely the threshold at which the constructed bipartite graph in Step 2 of the algorithm is fully connected. Below the connectivity threshold, the graph will be disconnected into small components with insufficient information to recover the expected value of edges across disconnected components.

In comparison to the literature, \cite{xia2017statistically} prove that the minimax optimal MSE is $O(1/p n^2)$, which is achieved via spectral initialization followed by power iteration \cite{xia2017statistically} or gradient descent \cite{cai2019nonconvex} as long as $p = \Omega(n^{-3/2})$. While our result achieves a similar sample complexity threshold, our MSE rate is suboptimal by a factor of $\sqrt{n}$. A limitation of neighborhood smoothing is that we do not achieve exact recovery under the noiseless setting, and we do not achieve the minimax optimal rates. It is unclear whether the gap is due to a limitation in the analysis or the algorithm. A benefit of our analysis in contrast to the literature is that the neighborhood smoothing approach can more easily deal with approximate low rank settings as arise under smoothness, as presented in Theorem \ref{main:thm.pert}. While low rank is often a useful modeling concept for real world datasets, in reality most real-world applications are likely only approximately low rank rather than exactly low rank.

\subsection{Approximately finite spectrum} For approximate rank $r$ model, we establish a natural perturbation 
result for the algorithm. Specifically, if the underlying model has $\bpert$-approximate rank $r$, then we argue
that the result of Theorem \ref{main:thm} remain true, both with respect to \text{MSE} and max-norm error,  with
perturbation amount of ${\sf poly}(\bpert)$.

\begin{theorem}\label{main:thm.pert}
We assume that the function $f$ has $\bpert$-approximate rank $r$, $L$-Lipschitz and that $\theta \sim U[0,1]$.
Assume that $p=n^{-3/2 + \extra}$ for some $\extra \in (0,\frac12)$. Choosing $\radius$ according to \eqref{eq:t}, it follows that $\radius= \lceil \frac{1}{4\extra} \rceil$. 
For any arbitrarily small $\psi \in (0, \min(\extra,\frac38))$, choose the threshold
\begin{align*}
\eta &=\Theta\left(n^{-(\extra - \psi)} + \radius \bpert (1+\bpert)^{2\radius-1} +  \radius^2 \bpert^2 (1+\bpert)^{4\radius-2}\right).
\end{align*}
The algorithm produces estimates so that, 
\begin{align*}
\text{MSE} &= O(n^{ - (\extra - \psi)} +\radius \bpert (1+\bpert)^{2\radius-1} +  \radius^2 \bpert^2 (1+\bpert)^{4\radius-2}) \\
& = O\left(\frac{n^{\psi}}{(p^2 n^3)^{1/2}} + \radius \bpert (1+\bpert)^{2\radius-1} +  \radius^2 \bpert^2 (1+\bpert)^{4\radius-2} \right),
\end{align*}
and
\begin{align*}
\|F - \hat{F}\|_{\max} & = O(n^{ - (\extra - \psi)/2} + \radius \bpert (1+\bpert)^{2\radius-1} + \sqrt{\radius \bpert (1+\bpert)^{2\radius-1}}), 
\end{align*}
with probability $1 - O\left(n^4 \exp(-\Theta(n^{2\psi}))\right) - O(n^{-2})$. 
\end{theorem}

As the entries of $F$ are normalized such that $\|F\|_{\max} \leq 1$, the bound is meaningful when $\bpert < 1$, in which case the dominating term of the additional error due to the perturbation is linear in $\bpert$, as $\radius$ is a constant.
The proof of Theorem \ref{main:thm.pert} relies on the following observation: the distribution of the data under the setting where the latent function $f$ has $\bpert$-rank $r$ is equivalent to the distribution of data generated according to the rank $r$ approximation of $f$ and then adding a deterministic perturbation to each observation accounting for the difference between $f$ and its rank $r$ approximation $f_r$, which is entrywise bounded by $\bpert$. In particular, the proof of Theorem \ref{main:thm.pert} shows that under arbitary deterministic perturbation of a rank $r$ model where the perturbation is bounded by $\bpert$, the estimation error
is perturbed by at most ${\sf poly}(\bpert)$. As a byproduct, our result proves that our estimator is robust to arbitrary deterministic bounded noise in the observations. 

The approximation guarantee depends on the spectral decay. Since the analysis allows any arbitrary adversarial model for the $\bpert$ deviation from a low rank model, the resulting guarantee depending polynomially in $\bpert$ is qualitatively best one can hope for. By definition, the minimal error must be lower bounded by $\bpert$, but determining the best achievable error as a function of $\bpert$ is an important open question for future investigation. It is worth noting that no other prior work addresses such a robust error model.

\subsection{Reducing Computational Complexity}

The computational complexity can be estimated by analyzing steps 3-5 of the algorithm. Step 3 costs $O(pn^4)$, as there are $n$ BFS trees to construct, which each take at most $pn^3$ edge traversals as there are at most order $pn^3$ edges in the constructed graph. 
Step 4 costs $O(p^2 n^6)$ as there are order $n^2$ pairwise distances to compute, and each computed distance involves sums over terms indexed by $a,b,\alpha,\beta \in [n]^4$ where $(a,\alpha,\beta)$ and $(b,\alpha,\beta)$ are in the observation set. As the sparsity of the dataset is $p$, this results in order $p^2 n^4$ nonzero terms in the summation, each of which is the product of 4 quantities, taking $O(1)$ to compute.
Step 5 costs $O(pn^6)$ as there are $\Theta(n^3)$ triplets we need to estimate, and each involves averaging at most 
$O(pn^3)$ datapoints. In summary, the computation cost of the entire method, for $p = n^{-\frac32 + \extra}$ is $O(pn^4 + p^2 n^6 + pn^6)$ where the cost in Step 5 dominates. 

This computation cost can be improved drastically. For example, as explained in \cite{BorgsChayesLeeShah17}, by 
use of `representative' or 'anchor' vertices chosen as random, the 
algorithm can instead cluster the vertices with respect to these anchor vertices and learn a block constant estimate, significantly reducing the involved computation. If there are $y$ anchor vertices, then Step 4 reduces to only computing pairwise distances between $\binom{y}{2} + ny$ pairs of vertices, as non-anchor vertices are only compared to the small set of $y$ anchor vertices. Step 5 reduces to only estimating $\binom{y}{3}$ entries of the tensor corresponding to combinations of the anchor vertices, and then extrapolating the estimate to other vertices assigned to the same cluster. This would result in a computational cost of 
$O(pn^4 + (y^2 + ny) p^2 n^4 + y^3 pn^3)$. When $p =n^{-\frac32 + \extra}$, our proof indicates that by choosing $y = \Theta((p^2 n^3)^{1/4})=\Theta(n^{\extra/2})$, the corresponding block constant estimator would achieve the same rates on the MSE and MEE as presented in Theorem \ref{main:thm}, while requiring a reduced computational complexity of $O(n^{5/2 + \extra} + n^{2 + 5 \extra/2})$.

\begin{corollary}\label{cor:anchor_vertices}
	We assume that the function $f$ is rank $r$, $L$-Lipschitz and that $\theta \sim U[0,1]$.
	Assume that $p=n^{-3/2 + \extra}$ for some $\extra \in (0,\frac12)$. Let $\radius$ be defined as per \eqref{eq:t}. 
	For any arbitrarily small $\psi \in (0, \min(\extra,\frac38))$, choose the threshold
	\begin{align*}
		\eta &= \Theta\left(n^{- (\extra - \psi)}\right).
	\end{align*}
	The modified algorithm which subsamples $y = \Omega((p^2 n^3)^{1/4})=\Omega(n^{\extra/2})$ anchor vertices at random and uses them to cluster the vertices to learn a block constant estimate will achieve 
	\begin{align*}
		\text{MSE} &= O(n^{ - (\extra - \psi)}) = O\left(\frac{n^{\psi}}{(p^2 n^3)^{1/2}}\right),
	\end{align*}
	and
	\begin{align*}
		\|F - \hat{F}\|_{\max} & = O(n^{ - (\extra - \psi)/2}), 
	\end{align*}
	with probability $1 - O\left(n^4 \exp(-\Theta(n^{2\psi}))\right)$. 
\end{corollary}

\subsection{Discussion of Assumptions}\label{ssec:result.disc}

We assumed in our algorithm and analysis that we had two fresh samples of the dataset, $M_1$ and $M_2$. The dataset $M_1$ is used to estimate distances between coordinates, and the dataset $M_2$ is used to compute the final nearest neighbor estimates. Given only a single dataset, the same theoretical results can also be shown by simply splitting the samples uniformly into two sets, one used to estimate distances and one used to compute the nearest neighbor estimates. as we are considering the sparse regime with $p = n^{-3/2 + \extra}$ for $\extra \in (0,\frac12)$, the two subsets after sample splitting will be nearly independent, such that the analysis only needs to be slightly modified. This is formally handled in the paper on collaborative filtering for matrix estimation in \cite{BorgsChayesLeeShah17}. 

Our model and analysis assumes that the latent variables $\{\theta_u\}_{u \in [n]}$ are sampled uniformly on the unit interval, and that the function $f$ is Lipschitz with respect to $\theta$. This assumption can in fact be relaxed significantly, as it is only used in the final step of the proof in analyzing the nearest neighbor estimator. Proving that the distance estimates concentrate well does not require these assumptions, in particular it primarily uses the low rank assumption. Given that the distance estimate concentrates well, the analysis of the nearest neighbor estimator depends on the local measure, i.e. what fraction of other coordinates have similar function values so that the estimated distance is small. We used Lipschitzness and uniform distribution on the unit interval in order to lower bound the fraction of nearby coordinates, however many other properties would also lead to such a bound. The dependence of the noisy nearest neighbor estimator on the local measure is discussed in detail in \cite{song2016blind}. Similar extensions as presented in \cite{song2016blind} would apply for our analysis here, leading to consistency and convergence rate bounds for examples including when
\begin{itemize}
\item the latent space has only finitely many elements, or equivalently the distribution of $\theta$ has finite support;
\item the latent space is the unit hypercube in a finite dimensional space and the latent function is Lipschitz;
\item the latent space is a complete, separable metric space, i.e. Polish space, with bounded diameter and the latent function is Lipschitz.
\end{itemize}

Although our stated results assume a symmetric tensor, the results naturally extend to asymmetric $(n_1 \times n_2 \times n_3)$ tensors as long as $n_1, n_2,$ and $n_3$ are proportional to one another. Our analysis can be modified for the asymmetric setting, or one can reduce the asymmetric tensor to a $(n\times n \times n)$ symmetric tensor where $n = n_1 +n_2 + n_3$, and the coordinates of the new tensor consists of the union of the coordinates in all three dimensions of the asymmetric tensor. The results applied to this larger tensor would still hold with adjustments of the model allowing for piecewise Lipschitz functions.

In the proof sketch that follows below, we show that for the 3-order tensor, the sample complexity threshold of $p = \omega(n^{-3/2})$ directly equals the density of observations needed to guarantee the bipartite graph is connected with high probability. 
Although our stated results assume a 3-order tensor, our algorithm and analysis can be likely extended to general $d$-order tensors. The proof would instead require analysis of the $n\times n^{d-1}$ matrix and associated bipartite graph corresponding to the unfolding of the tensor. The bipartite graph would consist of vertex sets $[n]$ and $[n]^{d-1}$.
	
\begin{remark}
In order for the vertices in $[n]$ to be fully connected to each other, $p$ needs to be $\Omega(n^{d/2})$, which can be proved using the standard branching process analysis as is used to prove the connectivity threshold of an Erdos Renyi graph \cite{alon2016probabilistic}. Let $X_u$ denote the set of vertices $v \in [n]$ such that the distance between $u$ and $v$ in the bipartite graph is 2. It follows that $\Prob{v \in X_u} = 1 - (1-p^2)^{n^{d-1}}$, such that $\E[|X_u|] = (n-1) (1 - (1-p^2)^{n^{d-1}}) = \Theta(p^2 n^d)$. As a result, if $p = o(n^{d/2})$, it follows that $\Prob{X_u = \emptyset} = 1- \Prob{|X_v| \geq 1} \geq 1 - \E[|X_u|] = 1 - o(1)$, such that with probability tending to 1 the vertex $u$ will be isolated, i.e. not connected to any other vertex $v \in [n]$. To prove that the graph is connected for $p = \tilde{\Omega}(n^{d/2})$, one would relate the growth of a local neighborhood in the graph to an appropriate branching process, where the expected number of descendents alternates between $pn$ and $pn^{d-1}$. To formalize the argument, one would then argue that the branching process survives to infinity if $p^2 n^d$ is sufficiently large, e.g $\text{polylog}(n)$, and also argue that the branching process is a reasonable approximation for the neighborhood growth of the graph until a linear number of vertices are visited. The requirement for graph connectivity in our algorithm and analysis arises from the similarity computation $\dist(u,v)$, which involves products of weights over paths in the graph that connect $u$ and $v$. The fact that $n^{d/2}$ also corresponds to the connectivity threshold in the corresponding bipartite graph sheds light on the computational lower bound for tensor completion in \cite{BarakMoitra16}, giving another way to explain the $\Omega(n^{d/2})$ lower bound.
\end{remark} 




\section{Proof} \label{sec:proof}

In this section, we present the proof for Theorem \ref{main:thm}. The proof outline is similar to the matrix setting in \cite{BorgsChayesLeeShah17}, in that the core of the analysis is proving that the distance function as defined in \eqref{eq:dist} concentrates appropriately and captures an appropriate notion of distance that enables the classical ``nearest neighbor" algorithm to be effective. However, due to high-dependencies across latent factors associated with columns that share tensor coordinates, the concentration of the BFS neighborhood expansion in section \ref{sec:conc_quadratic_1} requires a new argument beyond the simple martingale argument in the matrix setting. This involves a careful application of the concentration of U-statistics. Furthermore, the concentration of the distance calculation in Eq \eqref{eq:dist_T} as analyzed in section \ref{sec:conc_quadratic_3} requires a new argument relating the computed statistic to a thresholded variant more amenable to analysis. This is due to both the dependencies in the latent factors along with the lopsidedness in the dimensions so that straightforward applications of standard concentration results are too weak and insufficient to drive the error to zero. 

While the proof is stated for bounded observations, i.e. bounded noise, the result can be extended to sub-Gaussian noise rather than uniformly bounded noise. This would involve showing that the norms of the neighborhood vectors $N_{us}$ and $W_{us}$ are well-controlled such that the application of Hoeffding's inequality used in Lemmas \ref{lem:martingale_diff_1} and \ref{lem:martingale_diff_2} still hold. Additionally the proof of \ref{lemma:NMN_conc} naively bounds the products of weights over a path in absolute value by 1; if the noise were not bounded but sub-Gaussian, one would have to additionally argue that the product of the weights would be sufficienty controlled with high probability.

The critical lemma that the proof hinges on shows that the computed similarities, i.e. $\dist(u,a)$, concentrates around the function $d(u,a) =  \|\Lambda^{2\radius+1} Q (e_u - e_a)\|_2^2$. This then implies that if $d(u,a)$, $d(v,b)$ and $d(w, c)$ are small, the function value $F(u,v,w)$ would be close to $F(a,b,c)$. Additionally, we use Lipschitzness of the latent function $f$ along with the assumption that $\theta_u$ are sampled independently from $U[0,1]$ to argue that for any $u$, there is a sufficiently large set of other coordinates $a$ such that $d(u,a)$ is small. If these properties hold, then a simple analysis of nearest neighbor averaging using $\dist(u,a)$ to determine the neighbors will result in a bias variance tradeoff that can be tuned to show our final results. As such, the complexity of the proof revolves around showing that the computed $\dist(u,a)$ concentrates around $d(u,a)$. This involves a delicate analysis which involves first arguing that the normalized neighborhood vectors $\tN_{u,\radius}$ satisfy $e_k^T Q \tN_{u,\radius} \approx e_k^T \Lambda^{2\radius} Q e_{u}$, which involves martingale concentration as well as concentration of appropriately defined U-statistics. Subsequently we need to argue that the statistic 
\begin{align*}
Z_{uv} &= \frac{1}{|\cV_B(u,v,t)| p^2 |\cS_{u,t}| |\cS_{v,t}|}  \sum_{(\alpha, \beta) \in \cV_B(u,v,t)} \sum_{a \neq b \in [n]} N_{u,t}(a) N_{v,t}(b)   \Mpp(a, (\alpha, \beta)) \Mpp(b, (\alpha, \beta)) \\
&\approx \tN_{u,t}^T Q^T \Lambda^2 Q \tN_{v,t}.
\end{align*}
A challenge in the analysis is that each term in the sum has very small probability of being nonzero such that the sum is sparse enough that the standard concentration inequalities are not tight enough. Thus we relate the sum to a thresholded variant and use a tighter approximation of the binomial cdf to obtain the desired bound.

\subsection{Analyzing Noisy Nearest Neighbors} 
We start by stating an important Lemma \ref{lemma:nearest_neighbor}, adapted from \cite{BorgsChayesLeeShah17} that characterizes the error of the noisy nearest neighbor algorithm. 
Recall that our algorithm estimates $F(u, v, w)$, i.e. $f(\theta_u, \theta_v, \theta_w)$, according to \eqref{eq:estimate}, which simply averages over data-points $\Mppp(a, b, c)$ corresponding to tuples $(a, b, c)$ for which $a$ is close to $u$, $b$ is close to $v$ and $c$ is close to $w$ according to the estimated distance function. The choice of parameter $\eta$ allows for tradeoff between bias and variance of the algorithm.

We first argue that the data-driven distance estimates $\dist$ will concentrate around an 
ideal data-independent distance $d(\theta_u,\theta_v)$ for $d: [0,1]^2 \to \mathbb{R}_+$. 
We subsequently argue that the nearest neighbor estimate produced by \eqref{eq:estimate} 
using $d(\theta_u,\theta_v)$ in place of $\dist(u,v)$ will yield a good estimate by properly 
choosing the threshold $\eta$ to tradeoff between bias and variance. The bias will depend 
on the local geometry of the function $f$ relative to the distances defined by $d$. The variance 
depends on the measure of the latent variables $\{\theta_u\}_{u \in [n]}$ relative to the distances 
defined by $d$, i.e. the number of observed tuples $(a, b, c) \in \cEppp$ such that 
$d(\theta_u,\theta_a) \leq \eta$, $d(\theta_v,\theta_{b}) \leq \eta$ and $d(\theta_w,\theta_{c}) \leq \eta$ 
needs to be sufficiently large. We formalize the above stated desired properties. 

\begin{properties}[Good Distance]\label{ass:good_distances.1}
We call an ideal distance function $d: [0,1]^2 \to \RealsP$ to be a $\bias$-good distance function for some $\bias: \RealsP \to \RealsP$ if for any given $\eta > 0$ it follows that 
$|f(\theta_a, \theta_b, \theta_c) - f(\theta_u, \theta_v, \theta_w)| \leq \bias(\eta)$
for all $(\theta_a,\theta_b,\theta_c, \theta_u,\theta_v, \theta_w) \in [0,1]^4$ such that 
$d(\theta_u,\theta_a) \leq \eta$, $d(\theta_v,\theta_b) \leq \eta$ and $d(\theta_w,\theta_c) \leq \eta$.
\end{properties}

\begin{properties}[Good Distance Estimation]\label{ass:good_distances.2}
For some $\Delta > 0$, we call distance $\hat{d}: [n]^2 \to \RealsP$ a $\Delta$-good estimate for ideal distance $d: [0,1]^2 \to \RealsP$, if $|d(\theta_u,\theta_a) - \hat{d}(u,a)| \leq \Delta$ for all $(u, a) \in [n]^2$.
\end{properties}

\begin{properties}[Sufficient Representation]\label{ass:good_distances.3}
The collection of coordinate latent variables $\{\theta_u\}_{u \in [n]}$ is called $\meas$-represented for some
$\meas: \RealsP \to \RealsP$ if for any $u \in [n]$ and $\eta' > 0$, 
$\frac{1}{n} \sum_{a \in [n]} \Ind_{(d(u, a) \leq \eta')} \geq \meas(\eta')$. 
\end{properties}
\begin{lemma} \label{lemma:nearest_neighbor}
Assume that property \ref{ass:good_distances.1} holds with probability $1$, property \ref{ass:good_distances.2} holds
for any given pair $u, a \in [n]$ with probability $1-\alpha_1$, and property  \ref{ass:good_distances.3} holds with 
probability $1-\alpha_2$ for some $\eta, \Delta,$ and $\eta' = \eta - \Delta$; in particular $d$ is a 
$\bias$-good distance function, $\hat{d} = \dist$ as estimated from $\Mp$ is a $\Delta$-good distance 
estimate for $d$, and $\{\theta_u\}_{u \in [n]}$ is $\meas$-represented. Then noisy nearest neighbor estimate 
$\hat{F}$ computed according to \eqref{eq:estimate} satisfies 
\begin{align*}
\MSE(\hat{F}) \leq &\bias^2(\eta + \Delta) + \frac{\sigma^2}{(1 - \delta) p \left(\meas(\eta-\Delta) n\right)^3 }  + \exp\left(-\frac{\delta^2 p \left(\meas(\eta-\Delta) n\right)^3}{2}\right) + 3 n \alpha_1 + \alpha_2,
\end{align*}
for any $\delta \in (0,1)$. Furthermore, for any $\delta' \in (0,1)$ and $(u, v, w) \in [n]^3$, 
\[ |\hat{F}(u, v, w) - f(\theta_u,\theta_v, \theta_w)| \leq \bias(\eta + \Delta) + \delta',\]
with probability at least 
\begin{align*}
& 1 -  \exp\left(-\tfrac{1}{2}\delta^2 p \left(\meas(\eta-\Delta) n\right)^3\right)  -  \exp\left(-\delta'^2 (1 - \delta) p \left(\meas(\eta-\Delta) n\right)^3\right)  - 3 n \alpha_1 - \alpha_2.
\end{align*}
\end{lemma}

The proof of Lemma \ref{lemma:nearest_neighbor} is a modification from \cite{BorgsChayesLeeShah17} and is included in the Appendix.

\subsection{Proofs of Theorems \ref{main:thm} and  \ref{main:thm.pert}} 
\proof
We prove that as long as $p = n^{-3/2+\extra}$ for any $\extra \in (0,\frac12)$, with high probability, properties \ref{ass:good_distances.1}-\ref{ass:good_distances.3} hold for an appropriately chosen function $d$, and for distance estimates $\hat{d} = \dist$ computed according to \eqref{eq:dist} with $t$ defined in \eqref{eq:t}. 
We subsequently use Lemma \ref{lemma:nearest_neighbor} to conclude Theorem \ref{main:thm} and Theorem \ref{main:thm.pert}.  The proofs 
of Properties \ref{ass:good_distances.1} and \ref{ass:good_distances.3} are identical in Theorem \ref{main:thm} and Theorem \ref{main:thm.pert}, while
that of property \ref{ass:good_distances.2} differ. For Theorem \ref{main:thm}, we utilize Lemma  \ref{lemma:dist} while for 
Theorem \ref{main:thm.pert}, we utilize Lemma  \ref{lemma:dist.pert}. The proof of Theorem \ref{main:thm.pert} follows nearly the same argument, where $f$ will be replaced by the rank $r$ approximation $f_r$, c.f. \eqref{eq:eps.rank}. 

\medskip
\noindent{\em Good distance $d$ and Property \ref{ass:good_distances.1}.} 
We start by defining the ideal distance $d$ as follows. For all $(u, v) \in [n]^2$, let
\begin{align}
d(\theta_u,\theta_v) & = \|\Lambda^{2\radius+1} Q (e_u - e_v)\|_2^2 = \sum_{k=1}^r \lambda_k^{2(2\radius+1)} (q_k(\theta_u) - q_k(\theta_v))^2. \label{eq:ideal.dist} 
\end{align}
Recall that $\radius$ is defined in \eqref{eq:t}. Since $p = n^{-3/2 + \extra}$ and $\extra \in (0,\frac12)$, we have
that 
\begin{align}
\radius & = \Bigg\lceil\frac{\ln(n)}{2\ln(p^2 n^3)}\Bigg\rceil ~=~\Bigg\lceil \frac{1}{4\extra} \Bigg\rceil.
\end{align}

We want to show that there exists $\bias: \RealsP \to \RealsP$ so that $|(f(\theta_a, \theta_b, \theta_c) - f(\theta_u, \theta_v, \theta_w)) |  \leq \bias(\eta)$ for any $\eta > 0$ and $(u,a,v,b,w, c) \in [n]^3$ such that $d(\theta_u,\theta_a) \leq \eta$, $d(\theta_v,\theta_b) \leq \eta$ and $d(\theta_w,\theta_c) \leq \eta$. Consider 
\begin{align}
|f(\theta_u,\theta_v, \theta_w) - f(\theta_a, \theta_b, \theta_c)| 
& \leq |f(\theta_u,\theta_v, \theta_w) - f(\theta_a, \theta_v, \theta_w)| + |f(\theta_a,\theta_v, \theta_w) - f(\theta_a, \theta_b, \theta_w)| \nonumber \\
& \qquad  +  |f(\theta_a,\theta_b, \theta_w) - f(\theta_a, \theta_b, \theta_c)|. \label{eq:prf.1}
\end{align}
Now 
\begin{align}
|f(\theta_u,\theta_v, \theta_w) - f(\theta_a, \theta_v, \theta_w)| 
&= |\sum_k \lambda_k (q_k(\theta_u) - q_k(\theta_a)) q_k(\theta_v) q_k(\theta_w)|  \nonumber \\ 
& \stackrel{(a)}{\leq} B^2 | \sum_k \lambda_k (q_k(\theta_u) - q_k(\theta_a))|  \nonumber \\ 
& = B^2 \| \Lambda Q (e_u - e_a)\|_1 \nonumber \\
& \leq B^2 \sqrt{r} \| \Lambda Q (e_u - e_a)\|_2 \nonumber \\
& \leq B^2 \sqrt{r} |\lambda_r|^{-2\radius} \|\Lambda^{2\radius+1} Q (e_u - e_a)\|_2 \nonumber \\
& = B^2 \sqrt{r} |\lambda_r|^{-2\radius} \sqrt{d(\theta_u,\theta_a)}.
\end{align}
In above, (a) follows from the $\|q_k(\cdot)\|_\infty \leq B$ for all $k$. Repeating this 
argument to bound the other terms in \eqref{eq:prf.1}, we obtain that 
\begin{align}
|f(\theta_u,\theta_v, \theta_w) - f(\theta_a, \theta_b, \theta_c)|
&\leq 3 B^2 \sqrt{r} |\lambda_r|^{-2\radius} 
\max\big(\sqrt{d(\theta_u,\theta_a)}, \sqrt{d(\theta_v,\theta_b)}, \sqrt{d(\theta_w,\theta_c)}\big) \nonumber \\
& \leq 3 B^2 |\lambda_r|^{-2\radius}  \sqrt{r \eta}~\equiv~\bias(\eta).  \label{eq:ideal.bias}
\end{align}
In summary, property \ref{ass:good_distances.1} is satisfied for distance function $d$ defined 
according to \eqref{eq:ideal.dist} and $\bias(\eta) = 3 B^2 |\lambda_r|^{-2\radius} \sqrt{r \eta}$. 

\medskip
\noindent{\em Good distance estimate $\hat{d}$ and Property \ref{ass:good_distances.2}.} 
We state the following Lemma whose proof is delegated to Section \ref{sec:proofoflemma}.

\begin{lemma}\label{lemma:dist}
Given $f$ with rank $r$, assume that $p = n^{-3/2+\extra}$ for $\extra \in (0,\frac12)$. Let $\hat{d} = \dist$ as defined in \eqref{eq:dist}.  
Then for any $(u, a) \in [n]^2$, for any $\psi \in (0, \extra)$,
\begin{align*}
|d(\theta_u,\theta_a) - \hat{d}(u,a)|
&= O\left(r \lambda_{\max}^{4t}  n^{-(\extra - \psi)}\right), 
\end{align*}
with probability at least 
$1 - O\Big( \exp(-n^{2\psi}(1-o(1)))\Big)$.
\end{lemma}
Lemma \ref{lemma:dist} implies that property \ref{ass:good_distances.2} holds with probability $1 - o(1)$ for  
$\Delta = \Theta\left(r \lambda_{\max}^{4t} n^{- (\extra - \psi)}\right)$ when $f$ has rank $r$.

\begin{lemma}\label{lemma:dist.pert}
Given $f$ with $\bpert$-approximate rank $r$ for $\bpert \geq 0$,  
assume that $p = n^{-3/2+\extra}$ for $\extra \in (0,\frac12)$. Let $\hat{d} = \dist$ as defined in \eqref{eq:dist}. 
Then for any $(u, a) \in [n]^2$, for any $\psi \in (0, \extra)$,
\begin{align*}
|d(\theta_u,\theta_a) - \hat{d}(u,a)| 
=& ~O\left(r \lambda_{\max}^{4t} n^{-(\extra - \psi)}\right) + O\left( \radius \bpert (1+\bpert)^{2\radius-1} +  \radius^2 \bpert^2 (1+\bpert)^{4\radius-2} \right), 
\end{align*}
with probability at least $1 - O\Big( \exp(-n^{2\psi}(1-o(1)))\Big) - O\Big(n^{-6}\Big)$.
\end{lemma}
Lemma \ref{lemma:dist.pert} implies that property \ref{ass:good_distances.2} holds with probability $1 - o(1)$ for  
$$\Delta = \Theta\left(r \lambda_{\max}^{4t} n^{- (\extra - \psi)} + \radius \bpert (1+\bpert)^{2\radius-1} +  \radius^2 \bpert^2 (1+\bpert)^{4\radius-2} \right),$$ 
when $f$ has $\bpert$-approximate rank $r$.

\medskip
\noindent{\em Sufficient representation and Property \ref{ass:good_distances.3}.} Since $f$ is $L$-Lipschitz, the distance $d$ as defined in \eqref{eq:ideal.dist} is bounded above by the squared $\ell_2$ distance: 
\begin{align}
d(\theta_u,\theta_v)
& = \|\Lambda^{2t+1} Q (e_u - e_v)\|_2^2 \nonumber \\
& \leq |\lambda_1|^{4\radius} \|\Lambda Q (e_u - e_v)\|_2^2 \nonumber \\
& = |\lambda_1|^{4\radius} \Big(\sum_{k=1}^r \lambda_k^2 (q_k(\theta_u) - q_k(\theta_v))^2 \nonumber \\
& = |\lambda_1|^{4\radius} \sum_{k=1}^r \lambda_k^2 (q_k(\theta_u) - q_k(\theta_v))^2 \big(\int_{0}^1 q_k(\theta_a)^2 d\theta_a\big) \nonumber \\
&\qquad \qquad \qquad \times \big(\int_{0}^1 q_k(\theta_b)^2 d\theta_b\big)  \nonumber \\
& = |\lambda_1|^{4\radius} \sum_{k=1}^r \lambda_k^2 \int_{[0,1]^2}(q_k(\theta_u) q_k(\theta_a) q_k(\theta_b) - q_k(\theta_v)q_k(\theta_a) q_k(\theta_b) )^2 d\theta_a d\theta_b \Big) \nonumber \\
& \stackrel{(a)}{=} |\lambda_1|^{4\radius} \int_0^1 \int_0^1 (f(\theta_u, \theta_a, \theta_b) - f(\theta_v, \theta_a, \theta_b))^2 d\theta_a d\theta_b \nonumber \\
&\leq |\lambda_1|^{4\radius} L^2 |\theta_u - \theta_v|^2, \label{eq:ideal.dist1.a}
\end{align}
where in (a) we have used the fact that $q_k(\cdot), k \in [r]$ are orthonormal with respect to uniform distribution over
$[0,1]$. We assumed that the latent parameters $\{\theta_u\}_{u \in [n]}$ are sampled i.i.d. uniformly over $[0,1]$. Therefore, for any $\theta_u \in [0,1]$, for any $v \in [n]$ and $\eta' > 0$,  
\begin{align}
\Prob{d(\theta_u,\theta_v) \leq \eta' ~\big|~ \theta_u} &\geq \Prob{|\lambda_1|^{4\radius} L^2 |\theta_u - \theta_v|^2 \leq \eta' ~\big|~ \theta_u} \nonumber \\
& = \Prob{|\theta_u - \theta_v| \leq \frac{\sqrt{\eta'}}{|\lambda_1|^{2\radius} L} ~\big|~ \theta_u} \nonumber \\
& \geq \min\Big(1, \frac{\sqrt{\eta'}}{|\lambda_1|^{2\radius} L}\Big). \label{eq:nn_probablity}
\end{align}
Let us define 
\begin{align}
\meas(\eta') =  \frac{(1-\delta) \sqrt{\eta'}}{|\lambda_1|^{\radius} L}\label{eq:meas_def}
\end{align}
for all $\eta' \in (0, |\lambda_1|^{4\radius} L^2)$. By an application of Chernoff's bound and a simple majorization argument, it follows that for all $\eta' \in (0,|\lambda_1|^{4\radius} L^2)$ and $\delta \in (0,1)$,   
\begin{align*}
\Prob{\frac{1}{n-1} \sum_{a \in [n] \setminus u} \Ind_{(d(u,a) \leq \eta' )} \leq \meas(\eta') ~\Bigg|~ \theta_u } 
\leq \exp\left(- \frac{\delta^2 (n-1) \sqrt{\eta'}}{2 |\lambda_1|^{2\radius} L}\right).
\end{align*}
By using union bound over all $n$ indices, it follows that for any $\eta' \in (0, |\lambda_1|^{4\radius} L^2)$, with probability at least $1 -  n \exp\left(- \frac{\delta^2 (n-1) \sqrt{\eta'}}{2 |\lambda_1|^{2\radius} L}\right)$, property \ref{ass:good_distances.3} is
satisfied with $\meas$ as defined in \eqref{eq:meas_def}.

\medskip
\noindent{\em Concluding Proof of Theorem \ref{main:thm}.}
In summary, property \ref{ass:good_distances.1} holds
with probability $1$, by Lemma \ref{lemma:dist} property \ref{ass:good_distances.2} holds for a given tuple $(u, a) \in [n]^2$ with probability
$1-\alpha_1$ where $\alpha_1 = O\Big( \exp(-n^{2\psi}(1-o(1)))\Big)$ for $\psi \in (0, \min(\extra,\frac38))$ and $\extra \in (0,\frac12)$, 
property \ref{ass:good_distances.3} holds with probability
$1-\alpha_2$ where $\alpha_2 = n \exp\left(- \frac{\delta^2 (n-1) \sqrt{\eta - \Delta}}{2 |\lambda_1|^{2\radius} L}\right)$ with
distance estimate $\hat{d} = \dist$ defined in \eqref{eq:dist} with 
\begin{align}
d(\theta_u, \theta_v) &= \|\Lambda^{2\radius+1} Q (e_u - e_v)\|_2^2, \nonumber \\
\bias(\eta) & = 3 B^2 |\lambda_r|^{-2\radius} \sqrt{r \eta}, \nonumber \\
\Delta & = \Theta(r \lambda_{\max}^{4t} n^{ - (\extra - \psi) }), \nonumber \\
\meas(\eta') & = \frac{(1-\delta) \sqrt{\eta'}}{|\lambda_1|^{2\radius} L}, \label{eq:summary.1}
\end{align}
for any $\eta > 0$, $\delta \in (0,1)$ and $\eta' = \eta - \Delta \in (0,|\lambda_1|^{4\radius} L^2)$.
By substituting the expressions for $\bias$, $\meas$, and $\alpha$ into Lemma \ref{lemma:nearest_neighbor}, it follows that 

\begin{align*}
\MSE(\hat{F}) 
&\leq 9 B^4 |\lambda_r|^{-4\radius} r (\eta+\Delta) + \frac{ \sigma^2 L^3 |\lambda_1|^{6\radius}}{(1 - \delta)^4 p \left(\sqrt{\eta-\Delta} n\right)^3 } + n O\Big( \exp(-n^{2\psi}(1-o(1)))\Big)\\
& \quad + \exp\left(-\frac{\delta^2 (1-\delta)^3 p \left(\sqrt{\eta-\Delta} n\right)^3}{2L^3|\lambda_1|^{6\radius}}\right) + n \exp\left(- \frac{\delta^2 (n-1) \sqrt{\eta - \Delta}}{2 |\lambda_1|^{2\radius} L}\right).
\end{align*}
Additionally, for any $\delta' \in (0,1)$, 
\begin{align}
|\hat{F}(u, v, w) - f(\theta_u,\theta_v, \theta_w)| \leq 3 B^2 |\lambda_r|^{-2\radius} \sqrt{r (\eta + \Delta)} + \delta'  \label{eq:high_prob_bd1}
\end{align}
with probability at least
\begin{align*}
& 1 -  \exp\left(-\frac{\delta^2 (1-\delta)^3 p \left(\sqrt{\eta-\Delta} n\right)^3}{2L^3|\lambda_1|^{6\radius}}\right) - \exp\left(-\frac{\delta'^2 (1-\delta)^4 p \left(\sqrt{\eta-\Delta} n\right)^3}{L^3|\lambda_1|^{6\radius}}\right) \\ 
& \quad -  n O\Big( \exp(-n^{2\psi}(1-o(1)))\Big) - n \exp\left(- \frac{\delta^2 (n-1) \sqrt{\eta - \Delta}}{2 |\lambda_1|^{2\radius} L}\right).
\end{align*}
By selecting $\eta = \Theta\big(\Delta\big) = \Theta(r \lambda_{\max}^{4t} n^{-(\extra - \psi)})$ with a large enough constant so that $\eta - \Delta = \Theta(\eta)$, it follows that by the conditions that $\psi > 0$ and $\extra < \frac12$,
\begin{align*}
\eta \pm \Delta &= \Theta(r \lambda_{\max}^{4t} n^{-(\extra - \psi)}), \\
p (\sqrt{\eta-\Delta} n)^3  &= \Theta(r^{3/2} \lambda_{\max}^{6t} n^{\frac32 - \frac{\extra}{2} + \frac{3\psi}{2}}) = \Omega(n^{\frac54}), \\
n \sqrt{\eta-\Delta} &= \Theta(r^{1/2} \lambda_{\max}^{2t} n^{1-\frac{\extra - \psi}{2}}) = \Omega(n^{\frac34}).
\end{align*}
By substituting this choice of $\eta$ and $\delta = \frac12$, it follows that
\begin{align}
\MSE(\hat{F}) & = O\Big(r^2 (\lambda_{\max}/\lambda_r)^{4t} n^{-(\extra - \psi)}\Big).
\end{align}
By choosing $\delta' = n^{-(\extra-\psi)/2}$ such that $\delta' = \Theta(\sqrt{\eta})$ and $\delta'^2 p(\sqrt{\eta-\Delta} n)^3 = \Omega(n^{\frac34}) = \Omega(n^{2\psi})$ because $\psi < \frac38$. Therefore, by substituting into \eqref{eq:high_prob_bd1}, it follows that for any given $(u, v, w) \in [n]^3$, 
with probability $1 - O(n \exp(-\Theta(n^{2\psi}) )\Big)$, 
\begin{align}
 |\hat{F}(u, v, w) - f(\theta_u,\theta_v, \theta_w)| & = O\Big( r (\lambda_{\max}/\lambda_r)^{2t}  n^{-(\extra - \psi)/2}\Big).
\end{align}
Using union bound over choices of $(u, v, w) \in [n]^3$, it follows that the maximum entry-wise error is bounded above by 
$O\Big(n^{-(\extra - \psi)/2}\Big)$ with probability $1 - O(n^4 \exp(-\Theta(n^{2\psi}) )\Big)$. 
This completes the proof of Theorem \ref{main:thm}. 
\endproof

\medskip
\noindent{\em Concluding Proof of Theorem \ref{main:thm.pert}.} We follow similar line of argument as for proof of
Theorem \ref{main:thm}. As noted earlier,  property \ref{ass:good_distances.1} holds
with probability $1$, by Lemma \ref{lemma:dist.pert} property \ref{ass:good_distances.2} holds for a given tuple $(u, a) \in [n]^2$ with probability
$1-\alpha_1$ where $\alpha_1 = O\Big( \exp(-n^{2\psi}(1-o(1))) + n^{-6} \Big)$ for $\psi \in (0, \min(\extra,\frac38))$ and $\extra \in (0,\frac12)$, 
property \ref{ass:good_distances.3} holds with probability
$1-\alpha_2$ where $\alpha_2 = n \exp\left(- \frac{\delta^2 (n-1) \sqrt{\eta - \Delta}}{2 |\lambda_1|^{2\radius} L}\right)$ with
distance estimate $\hat{d} = \dist$ defined in \eqref{eq:dist} with 
\begin{align}
d(\theta_u, \theta_v) &= \|\Lambda^{2\radius+1} Q (e_u - e_v)\|_2^2, \nonumber \\
\bias(\eta) & = 3 B^2 |\lambda_r|^{-2\radius} \sqrt{r \eta}, \nonumber \\
\Delta & = \Theta(r \lambda_{\max}^{4t} n^{ - (\extra - \psi) } + \radius \bpert (1+\bpert)^{2\radius-1} +  \radius^2 \bpert^2 (1+\bpert)^{4\radius-2}  ), \nonumber \\
\meas(\eta') & = \frac{(1-\delta) \sqrt{\eta'}}{|\lambda_1|^{2\radius} L}, \label{eq:summary.1.pert}
\end{align}
for any $\eta > 0$, $\delta \in (0,1)$ and $\eta' = \eta - \Delta \in (0,|\lambda_1|^{4\radius} L^2)$.
By substituting the expressions for $\bias$, $\meas$, and $\alpha$ into Lemma \ref{lemma:nearest_neighbor}, it follows that for any $\delta' \in (0,1)$, 
\begin{align}
	 |\hat{F}(u, v, w) - f_r(\theta_u,\theta_v, \theta_w)| 
	&\leq 3 B^2 |\lambda_r|^{-2\radius} \sqrt{r (\eta + \Delta)} + \delta', \label{eq:high_prob_bd1.pert}
\end{align}
with probability at least
\begin{align*}
& 1 -  \exp\left(-\frac{\delta^2 (1-\delta)^3 p \left(\sqrt{\eta-\Delta} n\right)^3}{2L^3|\lambda_1|^{6\radius}}\right) - \exp\left(-\frac{\delta'^2 (1-\delta)^4 p \left(\sqrt{\eta-\Delta} n\right)^3}{L^3|\lambda_1|^{6\radius}}\right) \\ 
& \quad -  n O\Big( \exp(-n^{2\psi}(1-o(1)))  + n^{-6} \Big)  - n \exp\left(- \frac{\delta^2 (n-1) \sqrt{\eta - \Delta}}{2 |\lambda_1|^{2\radius} L}\right).
\end{align*}
By selecting 
\begin{align}
	\eta = \Delta + \min\big( \Delta, |\lambda_1|^{4\radius} L^2\big),\label{eq:eta_pert}
\end{align}
it follows by the conditions $\psi > 0$ and $\extra < \frac12$ that
\begin{align*}
&\eta + \Delta = \Theta(r \lambda_{\max}^{4t} n^{-(\extra - \psi)} + \radius \bpert (1+\bpert)^{2\radius-1} +  \radius^2 \bpert^2 (1+\bpert)^{4\radius-2}), \\
&\eta - \Delta = \Omega(n^{-(\extra - \psi)}), \\
&p (\sqrt{\eta-\Delta} n)^3  
					= \Omega(n^{\frac54}), \\
&n \sqrt{\eta-\Delta} 
				=\Omega(n^{\frac34}).
\end{align*}

By choosing $\delta' = n^{-(\extra-\psi)/2}$ such that $\delta' = O(\sqrt{\eta})$ and $\delta'^2 p(\sqrt{\eta-\Delta} n)^3 = \Omega(n^{\frac34}) = \Omega(n^{2\psi})$ because $\psi < \frac38$. Therefore, by substituting into \eqref{eq:high_prob_bd1.pert}, it follows that for any given $(u, v, w) \in [n]^3$, 
with probability $1 - O(n \exp(-\Theta(n^{2\psi}) )\Big) - O(n^{-5})$, 
\begin{align*}
	|\hat{F}(u, v, w) - f(\theta_u,\theta_v, \theta_w)| &\leq |\hat{F}(u, v, w) - f_r(\theta_u,\theta_v, \theta_w)| + |f_r(\theta_u,\theta_v, \theta_w) - f(\theta_u,\theta_v, \theta_w)|\\
	& = O\Big(r (\lambda_{\max}/\lambda_r)^{2t} n^{-(\extra - \psi)/2} + \radius \bpert (1+\bpert)^{2\radius-1} + \sqrt{\radius \bpert (1+\bpert)^{2\radius-1}}\Big),
\end{align*}
where the bias between $f_r$ and $f$ is bounded by $\bpert$, and dominated by the bound between $\hat{F}$ and $f_r$. The final result follows from a union bound over $(u, v, w) \in [n]^3$.

The bound on MSE also follows by substituting $\delta = \frac12$ and the same choice of $\eta$ from \eqref{eq:eta_pert} into Lemma \ref{lemma:nearest_neighbor}, and again noting that the bias between $F$ and $F_r$ is dominated by the error between $\hat{F}$ and $F_r$ such that 
\begin{align*}
	\MSE(\hat{F})
	& = O\left(r^2 (\lambda_{\max}/\lambda_r)^{4t} n^{-(\extra - \psi)}+ \radius \bpert (1+\bpert)^{2\radius-1}\right) + O\left( \radius^2 \bpert^2 (1+\bpert)^{4\radius-2}\right).
\end{align*}
This completes the proof of Theorem \ref{main:thm.pert}. 
\endproof

\subsection{Proof of Corollary \ref{cor:anchor_vertices}}

\proof
The proof follows the same format as the proof of Theorem \ref{main:thm}. Let us denote the set of anchor vertices as $\mathcal{Y}$ such that $|\mathcal{Y}| = y$, and they are assumed to be chosen uniformly at random amongst all vertices. For a pair of vertices $(a,b) \in \mathcal{Y}^2$, the estimate $\hat{F}(a,b)$ follows the same exact computation as described in Section \ref{sec:algo_formal}. As a result it follows from Theorem \ref{main:thm} that with high probability,
\begin{align*}
	\max_{(a,b,c) \in \mathcal{Y}^3} |F(a,b,c) - \hat{F}(a,b,c)| & = O(n^{ - (\extra - \psi)/2}). 
\end{align*}

Next we need to show the error is not degraded for non-anchor vertices $(u,v,w) \in ([n] \setminus \mathcal{Y})^3$. Let $\zeta: [n] \to \mathcal{Y}$ denote the function that maps from each vertex to the closest anchor vertex as determined by the true distances $d$, 
\[\zeta(u) = \argmin_{a \in \mathcal{A}} d(\theta_u,\theta_a),\]
and let $\hat{\zeta}: [n] \to \mathcal{Y}$ denote the data-dependent function that maps from each vertex to the closest anchor vertex as determined by the computed distances $\hat{d}$,
\[\hat{\zeta}(u) = \argmin_{a \in \mathcal{A}} \hat{d}(u,a).\]

The estimate for non-anchor vertices is then taken to be the estimate computed for the corresponding closest anchor vertices,
\[\hat{F}(u,v,w) = \hat{F}(\hat{\zeta}(u),\hat{\zeta}(v),\hat{\zeta}(w)),\]
such that 
\begin{align*}
|\hat{F}(u,v,w) - F(u,v,w)|
&\leq |\hat{F}(\hat{\zeta}(u),\hat{\zeta}(v),\hat{\zeta}(w)) - F(\hat{\zeta}(u),\hat{\zeta}(v),\hat{\zeta}(w))| + |F(\hat{\zeta}(u),\hat{\zeta}(v),\hat{\zeta}(w)) - F(u,v,w)|.
\end{align*}
By Theorem \ref{main:thm}, as $(\hat{\zeta}(u), \hat{\zeta}(v),\hat{\zeta}(w)) \in \mathcal{Y}^3$, the first term is bounded by $O(n^{ - (\extra - \psi)/2})$ with high probability.
By property \ref{ass:good_distances.1},
\begin{align*}
	|F(\hat{\zeta}(u),\hat{\zeta}(v),\hat{\zeta}(w)) - F(u,v,w)| 
	& \leq 3 B^2 \sqrt{r} |\lambda_r|^{-2\radius} 
	\sqrt{\max\big(d(\theta_u,\theta_{\hat{\zeta}(u)}), d(\theta_v,\theta_{\hat{\zeta}(v)}), d(\theta_w,\theta_{\hat{\zeta}(w)})\big)}.
\end{align*}

The modified algorithm computes distances using Step 4 of the described algorithm between all pairs of anchor vertices, as well as all pairs $(u,a)$ such that $u \in [n]$ and $a \in \mathcal{Y}$. For each computed distance between a pair $(u,a)$, 
by Lemma \ref{lemma:dist}, property \ref{ass:good_distances.2} holds for $\Delta = \Theta(n^{-\kappa+\psi})$ with probability
$1-\alpha_1$ where $\alpha_1 = O\Big( \exp(-n^{2\psi}(1-o(1)))\Big)$ for $\psi \in (0, \min(\extra,\frac38))$ and $\extra \in (0,\frac12)$.

In order to bound $\max_{u\in[n]} d(\theta_u,\theta_{\hat{\zeta}(u)})$, we argue that for every $u \in [n]$, with high probability
\begin{align*}
d(\theta_u,\theta_{\hat{\zeta}(u)})
&\overset{(a)}{\leq} \hat{d}(u,\hat{\zeta}(u)) + \Delta \\
&\overset{(b)}{\leq}  \hat{d}(u,\zeta(u)) + \Delta \\
&\overset{(c)}{\leq} d(\theta_u,\theta_{\zeta(u)}) + 2\Delta \\
&\overset{(d)}{=} \min_{a \in \mathcal{A}} d(\theta_u,\theta_a) +2\Delta,
\end{align*}
where (a) and (c) hold with high probability for $\Delta = \Theta(n^{-\kappa+\psi})$ as a result of property  \ref{ass:good_distances.2}, and (b) and (d) follow from the definition of the functions $\hat{\zeta}$ and $\zeta$.

To bound  $\min_{a \in \mathcal{A}} d(\theta_u,\theta_a)$, we use \eqref{eq:nn_probablity} from property \ref{ass:good_distances.3} to show that
for any $u \in [n]$, $\eta = \Theta(n^{-\kappa+\psi})$, and 
$y = \Omega((p^2 n^3)^{1/4})=\Omega(n^{\extra/2})$
\begin{align*}
\Prob{\min_{a \in \mathcal{Y}} d(\theta_u,\theta_a) > \eta ~\big|~ \theta_u} 
&= \prod_{a \in \mathcal{Y}} \Prob{d(\theta_u,\theta_a) > \eta ~\big|~ \theta_u} \\ 
&\leq \left(1 - \frac{\sqrt{\eta}}{|\lambda_1|^{2\radius} L}\right)^y \\
&\leq \exp(- \frac{y \sqrt{\eta'}}{|\lambda_1|^{2\radius} L}) = \exp(-\Theta(n^{\psi/2})).
\end{align*}

As a result, the max entrywise error is bounded by $O(n^{ - (\extra - \psi)/2})$ with high probability, which can be used to show the MSE bound of $O(n^{ - (\extra - \psi)})$.
\endproof

\section{Proving distance estimate is close}\label{sec:proofoflemma}

In this section we argue that the distance estimate as defined in \eqref{eq:dist} is close to an ideal distance
as claimed in the Lemma \ref{lemma:dist}.

\subsection{Regular enough growth of breadth-first-search (BFS) tree}

The distance estimation algorithm of interest constructs a specific BFS tree for each 
vertex $u \in [n]$ with respect to the bipartite graph between vertices $[n]$ and 
$\cV_A$ where recall that $\cV_A = \{(b,c)\in[n/2]^2 ~\text{s.t.}~ b < c\}$. The BFS tree construction
is done so that vertices at different levels do not share coordinates, i.e. if vertex $a \in [n]$ 
is visited in an earlier layer of the BFS tree, then no vertex corresponding to $(a,b)$ for any 
$b \in [n]$ can be visited subsequently. Similarly, if $(a,b)$ is visited in the BFS tree, then no 
subsequent vertices including either coordinates $a$ or $b$ can be visited. The restriction is placed across different depths, whereas pairs of vertices $(a,b)$ and $(a,c)$ can be visited in the same depth. Amongst various 
valid BFS trees, the algorithm chooses one arbitrarily (for example, see Figure \ref{fig:example}(c)). 

We recall some notations. Consider a valid BFS tree rooted at vertex $u \in [n]$ which 
respects the constraint that no coordinate is visited more than once. Recall that
for any $s \geq 1$, $\cU_{u,s} \subseteq \cV_A$ denotes the set of vertices at depth $(2s-1)$ and   
$\cS_{u,s} \subseteq [n]$ denotes the set of vertices at depth $2s$ of the 
BFS tree, $\cB_{u,s} = \cup_{l \in \lceil s/2 \rceil} \cU_{u,l} \cup_{h \in \lfloor s/2 \rfloor} \cS_{u,h}$, 
 $\cG(\cB_{u,s})$ denotes all the information corresponding to the subgraph restricted 
 to the first $s$ layers of the BFS tree which includes $\cB_{u,s}$, the latent variables 
 $\{\theta_a\}_{a \in \cB_{u,s}}$ and the edge weights $\{M_1(a,b,c)\}_{a,(b,c) \in \cB_{u,s}}$.
The vector $N_{u,s} \in [0,1]^n$ is such that the $a$-th coordinate is equal to the product 
of weights along the path from $u$ to $a$ in the BFS tree for $a \in \cS_{u,s}$, and the
vector $W_{u,s} \in [0,1]^{\cV_A}$ is such that the $(b,c)$-th coordinate is equal to the product 
of weights along the path from $u$ to $(b,c)$ in the BFS tree for $(b,c) \in \cU_{u,s}$. 
The normalized vectors are $\tN_{u,s} = N_{u,s} / |\cS_{u,s}|$ and 
$\tW_{u,s} = W_{u,s} / |\cU_{u,s}|$ for $u \in [n], ~s \geq 1$. 

In a valid BFS tree rooted at vertex $u$, $\pi_u(a)$ denotes the parent of $a \in [n]$, 
and $\pi_u(b,c)$ denotes the parent of $(b,c) \in \cV_A$. The neighborhood vectors 
satisfy recursive relationship,
\begin{align*}
N_{u,s}(a) &= \Mp(a, \pi_u(a)) W_{u,s}(\pi_u(a)) \Ind_{(a \in \cS_{u,s})} \\
W_{u,s}(b,c) &= \Mp(\pi_u(b,c), (b,c)) N_{u,s-1}(\pi_u(b,c)) \Ind_{((b,c) \in \cU_{u,s})}
\end{align*}
with $N_{u,0} = e_u$. We state the following result regarding regularity in the growth
of the BFS tree. 
\begin{lemma}\label{lem:growth}
Let $p = n^{-3/2+\extra}$ for $\extra \in (0,\frac12)$. Let $\radius$ be as defined in \eqref{eq:t}. 
For a given $\delta \in (0,\frac12)$ and for any $u \in [n]$, 
with probability $1 - O\Big( n \exp\big(-\Theta(n^{2\extra}) \big) \Big)$,
or all $s \in [\radius-1]$, 
\begin{align}
|\cS_{u,s}| & \in \left[(1-\delta)^{2s} 2^{-3s} n^{2\extra s} (1-o(1)), (1+\delta)^{2s} 2^{-s} n^{2\extra s} \right], \label{eq.lem1.a}
\end{align}
for $s = \radius$,
\begin{align}
|\cS_{u,\radius}| \in \left[(1-\delta)^{2\radius} 2^{-3\radius-1} n^{2\extra \radius} (1-o(1)), (1+\delta)^{2\radius} 2^{-\radius} n^{2\extra \radius} \right],\label{eq.lem1.a2}
\end{align}
and for $s \in [\radius]$,
\begin{align}
|\cU_{u,s}| & \in \left[(1-\delta)^{2s-1} 2^{-3s} n^{\frac12 + \extra (2s-1)} (1-o(1)), (1+\delta)^{2s-1} 2^{-s} n^{\frac12 + \extra (2s-1)}\right]. \label{eq.lem1.b}
\end{align}
The set of single coordinate vertices visited within depth $2\radius$ is $o(n)$,
\begin{align}
|\cup_{\ell=0}^{\radius} \cS_{u,\ell}| = o(n).\label{eq.lem1.c}
\end{align}
\end{lemma}

\begin{proof}
First observe that if $\radius$ is as defined in \eqref{eq:t} with $\extra \in (0,\frac12)$, then 
\[\radius = \Big\lceil \frac{\ln(n)}{2\ln(p^2 n^3)} \Big\rceil = \Big\lceil \frac{1}{4\extra} \Big\rceil \] 
such that
\begin{align}
\frac{1}{4 \extra} \leq \radius < \frac{1}{4 \extra} + 1.
\end{align}
Note that $\radius$ is constant with respect to $n$.

For any $s \in [\radius]$, we study the growth of $|\cS_{u,s}|$ and $|\cU_{u,s}|$ conditioned on 
$\cB_{u,2s-1} \cup \cU_{u,s}$ and $ \cB_{u,2(s-1)} \cup \cS_{u,s-1}$ respectively. 
To that end, conditioned on the set $\cB_{u,2s-1}$ and the set $\cU_{u,s}$, 
any vertex $i \in [n] \setminus \cB_{u,2s-1}$ is in $\cS_{u,s}$ independently 
with probability $(1-(1-p)^{|\cU_{u,s}|})$. Thus the number of vertices in 
$\cS_{u,s}$ is distributed as a binomial random variable. By Chernoff's bound,
\begin{align}
&\Prob{|\cS_{u,s}| \notin (1 \pm \delta) (|[n] \setminus \cB_{u,2s-1}|) (1-(1-p)^{|\cU_{u,s}|}) ~|~ \cB_{u,2s-1}, \cU_{u,s}} \nonumber\\
&\leq 2 \exp\left(-\frac13 \delta^2 (|[n] \setminus |\cB_{u,2s-1}|) (1-(1-p)^{|\cU_{u,s}|})\right) \label{eq.lg.1}.
\end{align}
Similarly, conditioned on the sets $\cB_{u,2(s-1)}$ and $\cS_{u,s-1}$, 
the set of vertices in $\cU_{u,s}$ is equivalent to the number of edges in a 
graph with vertices $[n/2] \setminus \cB_{u,2(s-1)}$ and an edge between $(i,j)$ 
if there is some $h \in \cS_{u,s-1}$ such that $(i,j,h) \in \cEp$. 
This is an Erdos-Renyi graph, as each edge is independent with probability 
$(1 - (1-p)^{|\cS_{u,s-1}|})$. By Chernoff's bound,
\begin{align}
&\Prob{|\cU_{u,s}| \notin (1 \pm \delta) \binom{|[n/2] \setminus \cB_{u,2(s-1)}|}{2} (1-(1-p)^{|\cS_{u,s-1}|}) ~|~ \cB_{u,2(s-1)}, \cS_{u,s-1}} \nonumber\\
&\leq 2 \exp\left(-\frac13 \delta^2 \binom{|[n/2] \setminus \cB_{u,2(s-1)}|}{2} (1-(1-p)^{|\cS_{u,s-1}|})\right) \label{eq.lg.2}. 
\end{align}
Let us define the events
\begin{align}
\cA^1_{u,s}(\delta) & =  \Big\{|\cS_{u,s}| \in (1 \pm \delta) (|[n] \setminus \cB_{u,2s-1}|) (1-(1-p)^{|\cU_{u,s}|})\Big\}, \\
\cA^2_{u,s}(\delta) & = \Big\{|\cU_{u,s}| \in (1 \pm \delta) \binom{|[n/2] \setminus \cB_{u,2(s-1)}|}{2} (1 - (1-p)^{|\cS_{u,s-1}|})\Big\}.
\end{align}
Since $ p \in (0,1)$ and hence $1-(1-p)^x \leq px$ for all $ x \geq 1$, we have that under 
events $\cA^1_{u,s}(\delta) \cap \cA^2_{u,s}(\delta)$,
\begin{align*}
|\cS_{u,s}| & \leq (1 + \delta) n p |\cU_{u,s}| ~\text{ and }~
|\cU_{u,s}|  \leq (1 + \delta) \binom{n/2}{2} p |\cS_{u, s-1}|,
\end{align*}
which together implies that
conditioned on event $\cap_{h=1}^{s} \big( \cA^1_{u,h}(\delta) \cap \cA^2_{u,h}(\delta) \big)$, for all
$s \in [\radius]$ 
\begin{align}
|\cS_{u,s}| &\leq \left((1 + \delta)^2 \frac{p^2 n^3}{8} \right)^s ~=~(1 + \delta)^{2s} 2^{-3s} n^{2\extra s}, \label{eq.lg.3}
\end{align}
and 
\begin{align}
|\cU_{u,s}| 
&\leq (1 + \delta) \frac{pn^2}{8} \left((1 + \delta)^2 \frac{p^2 n^3}{8} \right)^{s-1} 
 =~(1+\delta)^{2s-1} 2^{-3s} 
n^{\frac12+\extra(2s-1)}. \label{eq.lg.4}
\end{align}
Therefore, for any $s \in [\radius-1]$ such that $s \leq \frac{1}{4\extra}$ by the definition of $\radius$, 
\begin{align}
|\cB_{u,2s}| 
&\leq 1 + \sum_{\ell=1}^s (2|\cU_{u,\ell}| + |\cS_{u,\ell}|) \nonumber \\
&\leq 1 + \sum_{\ell=1}^s \left(2 (1 + \delta) \frac{pn^2}{8} \left((1 + \delta)^2 \frac{p^2 n^3}{8} \right)^{\ell-1} + \left((1+\delta)^2 \frac{p^2 n^3}{8} \right)^{\ell} \right) \nonumber \\
&= 1 + \left(2 (1 + \delta) \frac{pn^2}{8} + \left((1 + \delta)^2 \frac{p^2 n^3}{8} \right) \right) \sum_{\ell=1}^{s} \left((1+\delta)^2 \frac{p^2 n^3}{8} \right)^{\ell-1} \nonumber \\
&= O\big( pn^2 (p^2 n^3)^{s-1} \big) ~=~O\big(n^{\extra (2s-1) + \frac12}\big)~=~
O\big(n^{1-\extra}\big) ~=~o(n). \label{eq.lg.5}
\end{align}
With a similar argument we can show that 
\begin{align}
\sum_{\ell=0}^t |\cS_{u,\ell}| &\leq \sum_{\ell=0}^{\radius} \left((1+\delta)^2 \frac{p^2 n^3}{8} \right)^{\ell} = O((p^2 n^3)^{\radius}) = O(n^{2\extra \radius}) = o(n). \label{eq.lg.6}
\end{align}
The last step follows from checking that when $\extra \in [\frac14, \frac12)$, $\radius = 1$ such that $n^{2\extra \radius} = o(n)$, and when $\extra \in (0, \frac14)$, from $\radius \leq \frac{1}{4\extra} + 1$, it follows such that $n^{2 \extra \radius} = O(n^{\frac12 + 2\extra}) = o(n)$ as $\extra < \frac14$.
Recall that we split the coordinates such that $\cup_{\ell=1}^t \cU_{u,\ell} \subset \cV_A$, and the coordinates represented in $(a,b) \in \cV_A$ are such that $a \in [n/2]$ and $b \in [n/2]$. Therefore by \eqref{eq.lg.6}, 
\[|[n] \setminus \cB_{u,2\radius-1}| \geq n/2 - \sum_{\ell=0}^{t-1} |\cS_{u,\ell}| = \frac{n}{2} (1-o(1)).\]

Using \eqref{eq.lg.5}, we establish lower bounds on $|\cS_{u,s}|$ and $|\cU_{u,s}|$ next. Note that, for $p \in (0,1)$, $1-p \leq e^{-p}$ and for any $x \in (0,1)$, $e^{-x} \leq 1 - x + x^2$. It follows that $1-(1-p)^x \geq px (1-px)$. For $s \in [\radius]$ we can show that
\begin{align*}
|\cU_{u,s}|
&\geq (1 - \delta) \frac{(n (1-o(1)))^2 (1- o(1))}{8} (1 - (1-p)^{|\cS_{u,s-1}|}) \\
&\geq (1 - \delta) \frac{n^2}{8} p|\cS_{u,s-1}| (1- p|\cS_{u,s-1}|) (1- o(1)) \\
&= (1 - \delta) \frac{n^2}{8} p |\cS_{u,s-1}| (1- o(1)).
\end{align*}
For $s \in [\radius - 1]$ we can show that
\begin{align*}
|\cS_{u,s}|
&\geq (1 - \delta) n (1 - o(1)) (1-(1-p)^{|\cU_{u,s}|}) \\
&\geq (1 - \delta) n (1 - o(1)) p|\cU_{u,s}| (1- p|\cU_{u,s}|) \\
&\geq (1 - \delta) n (1 - o(1)) p|\cU_{u,s}| (1-o(1)) \\
&= (1 - \delta) p n |\cU_{u,s}| (1-o(1)),
\end{align*}
and for $s = t$, 
\[|\cS_{u,t}| \geq (1 - \delta) \frac{n}{2} (1 - o(1)) (1-(1-p)^{|\cU_{u,t}|}).\]
Then for $s \in [\radius]$,
\begin{align}
|\cU_{u,s}|
&\geq (1 - \delta)^2 \frac{p^2 n^3}{8} |\cU_{u,s-1}| (1- o(1)) \nonumber \\
&\geq \left((1 - \delta)^2 \frac{p^2 n^3}{8}\right)^{s-1} |\cU_{u,1}| (1- o(1))  \nonumber \\
&\geq \left((1 - \delta)^2 \frac{p^2 n^3}{8}\right)^{s-1} (1 - \delta) \frac{p n^2}{8} (1- o(1)) \nonumber \\
& =~ (1- \delta)^{2s-1} 2^{-3s} n^{\frac12+\extra(2s-1)}(1-o(1)); \label{eq.lg.8}
\end{align}
for $s \in [\radius-1]$,
\begin{align}
|\cS_{u,s}|
&\geq (1 - \delta)^2 p n \frac{n^2}{8} p |\cS_{u, s-1}| (1- o(1)), \nonumber \\
&\geq \left((1 - \delta)^2 \frac{p^2 n^3}{8} \right)^s (1- o(1)),  \nonumber \\
& =~ (1 - \delta)^{2s} 2^{-3s} n^{2\extra s} (1-o(1)); \label{eq.lg.7}
\end{align}
and for $s =\radius$, 
$|\cS_{u,\radius}| \geq (1 - \delta)^{2\radius} 2^{-3\radius-1} n^{2\extra \radius} (1-o(1))$.

To conclude the proof of Lemma \ref{lem:growth}, we need to argue that 
$\cap_{s=1}^\radius \big( \cA^1_{u,s}(\delta) \cap \cA^2_{u,s}(\delta)\big)$ holds with high probability. To that end, 
\begin{align*}
& \Prob{\neg( \cap_{s=1}^{\radius} (\cA^1_{u,s}(\delta) \cap \cA^2_{u,s}(\delta)))} ~=~\Prob{\cup_{s=1}^{\radius} \neg(\cA^1_{u,s}(\delta) \cap \cA^2_{u,s}(\delta))} \\ & =~ \sum_{s=1}^\radius \Prob{\neg(\cA^1_{u,s}(\delta) \cap \cA^2_{u,s}(\delta)) \cap_{h=1}^{s-1} (\cA^1_{u,h}(\delta) \cap \cA^2_{u,h}(\delta))}  \\
 & \leq~\sum_{s=1}^\radius \Prob{\neg(\cA^1_{u,s}(\delta) \cap \cA^2_{u,s}(\delta)) \, | \, \cap_{h=1}^{s-1} (\cA^1_{u,h}(\delta) \cap \cA^2_{u,h}(\delta))} \\
& \leq \sum_{s=1}^\radius \Prob{\neg\cA^1_{u,s}(\delta)  \, | \, \cap_{h=1}^{s-1} (\cA^1_{u,h}(\delta) \cap \cA^2_{u,h}(\delta))} \\
& \qquad 
+ \sum_{s=1}^\radius \Prob{\neg\cA^2_{u,s}(\delta) \, | \, \cA^1_{u,s}(\delta) \cap_{h=1}^{s-1} (\cA^1_{u,h}(\delta) \cap \cA^2_{u,h}(\delta))}.
\end{align*}
We bound the each of the two summation terms on the right hand side in the last inequality next. 
Using \eqref{eq.lg.1} and \eqref{eq.lg.3}, we have
\begin{align*}
& \sum_{s=1}^\radius \Prob{\neg\cA^1_{u,s}(\delta)  \, | \, \cap_{h=1}^{s-1} (\cA^1_{u,h}(\delta) \cap \cA^2_{u,h}(\delta))} \\
& \leq \sum_{s=1}^{\radius-1} 2 \exp\left(-\frac13 \delta^2 (1 - \delta) \frac{p^2 n^3}{2} \left((1 - \delta)^2 \frac{p^2 n^3}{2}\right)^{s-1} (1- o(1))\right) \\
&\qquad + 2 \exp\left(-\frac13 \delta^2 (1 - \delta) \frac{p^2 n^3}{4} \left((1 - \delta)^2 \frac{p^2 n^3}{2}\right)^{\radius-1} (1- o(1))\right) \\
&\leq 4 \exp\left(-\frac13 \delta^2 (1 - \delta) \frac{p^2 n^3}{2} (1- o(1))\right)~=~O\Big(\exp\big(-\Theta(n^{2\extra})\big)\Big).
\end{align*}
Similarly, using \eqref{eq.lg.2} and \eqref{eq.lg.4}, we have
\begin{align*}
& \sum_{s=1}^\radius \Prob{\neg\cA^2_{u,s}(\delta)  \, | \, \cA^1_{u,s}(\delta) \cap_{h=1}^{s-1} (\cA^1_{u,h}(\delta) \cap \cA^2_{u,h}(\delta))} \\
&\leq \sum_{s=1}^{\radius+1} 2 \exp\left(-\frac13 \delta^2 \frac{n^2 p}{2} \left((1 - \delta)^2 \frac{p^2 n^3}{2} \right)^{s-1} (1-o(1)) \right) \\
&\leq 4\exp\left(-\frac13 \delta^2 \frac{n^2 p}{2} (1-o(1)) \right)~=~O\Big(\exp\big(-\Theta(n^{\frac12 + \extra})\big)\Big).
\end{align*}
Putting it all together, we have that 
\begin{align*}
\Prob{\neg( \cap_{s=1}^{\radius} (\cA^1_{u,s}(\delta) \cap \cA^2_{u,s}(\delta)))} 
&\leq O\Big(\exp\big(-\Theta(n^{2\extra})\big)\Big) + 
O\Big(\exp\big(-\Theta(n^{\frac12 + \extra})\big)\Big) \\
&=~ O\Big(\exp\big(-\Theta(n^{2\extra})\big)\Big),
\end{align*}
since $\extra \in (0,\frac12)$. By union bound over all $u \in [n]$, we obtain the desired bound on the probability of error. This concludes the proof of Lemma \ref{lem:growth}. 
\end{proof}

\subsection{Concentration of Quadratic Form One} \label{sec:conc_quadratic_1}

Let $\cA^3_{u, \radius}(\delta)$ denote the event that \eqref{eq.lem1.a} holds for all $s \in [\radius-1]$, \eqref{eq.lem1.a2} holds, \eqref{eq.lem1.b} holds for all $s \in [\radius]$, and \eqref{eq.lem1.c} holds. 
Lemma \ref{lem:growth} established that this event holds with high probability. 
Conditioned on the event $\cA^3_{u, \radius}(\delta)$, we prove that a specific quadratic form concentrates around its mean. 
This will be used as the key property to eventually establish that the distance estimates are a good approximation to the 
ideal distances. 
\begin{lemma} \label{lemma:nhbrhd_vectors}
Let $p = n^{-3/2+\extra}$ for $\extra \in (0,\frac12)$, $\radius$ as defined in \eqref{eq:t}, 
$\delta \in (0,\frac12)$, and $\psi \in (0,\extra)$.
For any $u \in [n]$, with probability $1 - 2 \exp(-n^{2\psi}(1-o(1)))$, 
\[|e_k^T Q \tN_{u,\radius} - e_k^T \Lambda^{2\radius} Q e_{u}| < \frac{ 16 \lambda_k^{2\radius-2} n^{\psi}} {(1-\delta) n^{\extra}}.\]
\end{lemma}

\proof
Recall that conditioning on event $\cA^3_{u, \radius}(\delta)$ simply imposes the restriction that the neighborhood
of $u \in [n]$ grows at a specific rate. This event is independent from latent parameters $\{\theta_a\}_{a \in [n]}$, the precise entries in $\Omega_1$ as well as associated values, i.e. $M_1$. 

Conditioned on $\cA^3_{u, \radius}(\delta)$, let $\cF_{u,s}$ for $0\leq s \leq 2\radius$  denote the sigma-algebra containing information about the latent parameters, edges and the values  associated with nodes in the bipartite graph up to distance $s$ from $u$, i.e. nodes $\cS_{u, h'}$ for $h' \leq \lfloor s/2\rfloor$, 
$\cU_{u, h^{''}}$ for $h^{''} \leq \lceil s/2\rceil$, associated latent parameters as well as edges of $\Omega_1$. 
Specifically, $\cF_{u, 0}$ contains information about latent parameter $\theta_u$ associated with $u \in [n]$; 
$\cF_{u,s}$ contains information about latent parameters 
$\cup_{h=1}^{\lfloor s/2\rfloor} \{\theta_a\}_{a \in \cS_{u,h}} \cup_{h=1}^{\lceil s/2\rceil} \{\theta_b, \theta_c \}_{(b,c) \in \cU_{u, h}}$ and all the associated edges and observations. This implies that 
 $ \cF_{u,0} \subset \cF_{u, 1} \subset \cF_{u,2}$, etc.

Recall that $Q$ denotes the $r \times n$ matrix where $Q_{ka} = q_k(\theta_a), k \in [r], a \in [n]$. We modify the notation due to the sample splitting, and we let $\cQ$ denote the $r \times \binom{n/2}{2}$ matrix where $\cQ_{kb} = q_k(\theta_{b_1}) q_k(\theta_{b_2})$ for some $b \in \cV_A$ that represents the pair of coordinates $(b_1,b_2)$ for $b_1 < b_2 \in [n/2]$. 

We shall consider a specific martingale sequence with respect to the filtration $\cF_{u, s}$ that will help establish the desired concentration of $e_k^T Q \tN_{u,\radius} - e_k^T \Lambda^{2\radius} Q e_{u}$. 
For $1 \leq s \leq 2 \radius$, define 
\begin{align*}
Y_{u,s} &= \begin{cases}
e_k^T \Lambda^{2\radius-s} Q \tN_{u,s/2} &\text{ if } s \text{ even}\\
e_k^T \Lambda^{2\radius-s} \cQ \tW_{u,(s+1)/2} &\text{ if } s \text{ odd}
\end{cases} \\
D_{u,s} &= Y_{u,s} - Y_{u,s-1}, \\
Y_{u,2\radius} - Y_{u,0} &= e_k^T Q \tN_{u,\radius} - e_k^T \Lambda^{2\radius} Q \tN_{u,0} 
= \sum_{s=1}^{2\radius} D_{u,s}.
\end{align*}
Note that $\tN_{u,0} = e_u$, and $Y_{u,s}$ is measurable with respect to $\cF_{u,s}$ because $e_k^T \Lambda^{2\radius-s} Q \tN_{u,s/2}$ and $e_k^T \Lambda^{2\radius-s} Q \tN_{u,(s+1)/2}$ only depend on observations in the BFS tree within depth $s$.

By Lemmas \ref{lem:martingale_diff_1} and \ref{lem:martingale_diff_2}, it follows that $Y_{u,s}$ is martingale with respect to $\cF_{u,s}$ for $1\leq s \leq \radius$, i.e. 
\begin{align}\label{eq:martingale}
\E[D_{u,s} ~|~ \cF_{u,s-1}] = 0.
\end{align}
Furthermore, for properly chosen $\nu_s$ as specified in Lemmas \ref{lem:martingale_diff_1} and \ref{lem:martingale_diff_2}, 
\[\E[e^{\lambda D_s} ~|~ \cF_{s-1},\cA^3_{u,\radius}(\delta)] \leq e^{\lambda^2 \nu_s^2/2}\]
almost surely for any $\lambda \in \Reals$.

We can then apply Proposition \ref{prop:MG} with any arbitrarily small $\alpha_*$ such that for any $x \geq 0$,
\begin{align*}
\Prob{|e_k^T Q \tN_{u,\radius} - e_k^T \Lambda^{2\radius} Q e_{u}| \geq x ~|~\cA^3_{u,\radius}(\delta)} 
& \leq 
2 \exp\Bigg(- \frac{x^2}{2\sum_{s=1}^{2\radius} \nu_s^2}\Bigg)
\end{align*}
where for $n$ large enough,
\begin{align*}
\sum_{s=1}^{2\radius} \nu_s^2
&= \sum_{s=1}^{\radius} \frac{(1 + 4 \pi) \lambda_k^{4\radius-4s} 2^{3s-1} (1+o(1))}{(1-\delta)^{2s} n^{2\extra s}} + \sum_{s=1}^{\radius} \frac{(1 + 16 \pi) 72 \lambda_k^{4\radius-4s+2} B^4 2^{3s-1} (1+o(1))}{(1-\delta)^{2s-1} n^{\min\{1, \frac12 + \extra (2s-1)\}}} \\
&\leq \frac{ 8 (1 + 4 \pi) \lambda_k^{4\radius-4} (1+o(1))} {(1-\delta)^{2} n^{2\extra}},
\end{align*}
and $1 + 4 \pi \leq 16$.

For $\psi \in (0,\extra)$, we choose $x = \frac{ 16 \lambda_k^{2\radius-2} n^{\psi}} {(1-\delta) n^{\extra}} = o(1)$, such that with probability $1 - 2 \exp(-n^{2\psi}(1-o(1)))$, 
\[|e_k^T Q \tN_{u,\radius} - e_k^T \Lambda^{2\radius} Q e_{u}| < x.\]

\endproof

We recall the following concentration inequality for Martingale difference sequence, cf. \cite[Theorem 2.19]{wainwrightbook}:
\begin{proposition}\label{prop:MG}
Let $\{D_k, \cF_k\}_{k\geq 1}$ be a martingale difference sequence such 
that $\E[e^{\lambda D_k}| \cF_{k-1}] \leq e^{\lambda^2 \nu_k^2/2}$ almost surely for all $\lambda \in \mathbb{R}$. 
Then for all $x > 0$, 
\begin{align}
\Prob{\Big| \sum_{k=1}^n D_k \Big| \geq x} & \leq 
2 \exp\Bigg(- \frac{x^2}{2\sum_{k=1}^n \nu_k^2}\Bigg). 
\end{align}
\end{proposition}

\begin{lemma}\label{lem:martingale_diff_1}
For any $s \in [\radius]$, 
$$\E\left[D_{u,2s} ~|~ \cF_{u,2s-1}, \cA^3_{u,\radius}(\delta) \right] = 0.$$
Let $\nu = \sqrt{\frac{Q(1 + 4 \pi)}{2}}$, and 
\[Q = \frac{B^2\lambda_k^{4\radius-4s} 2^{3s} (1+o(1)) } {(1-\delta)^{2s} n^{2\extra s}}.\]
For any $\lambda \in \Reals$,
\[ \E\left[e^{\lambda D_{u,2s}} ~|~ \cF_{u,2s-1}, \cA^3_{u,\radius}(\delta)\right] \leq \exp\Big(\frac{\lambda^2 \nu^2}{2} \Big).\]
\end{lemma}

\proof
Recall $\cF_{2s-1}$ contains all information in the depth $2s-1$ neighborhood of vertex $u$. In particular this includes the vertex set 
\[\cB_{u,2s-1} = \cup_{l=1}^{s} \cU_{u,l} \cup_{h=1}^{s-1} \cS_{u,h},\]
 the vertex latent variables $\{\theta_i\}_{i \in \cB_{u,2s-1}}$ and the edges and corresponding weights. 
 Let us additionally condition on the set $\cS_{u,s}$. As $2s$ is even, 
\begin{align*}
D_{u,2s} &= Y_{u,2s} - Y_{u,2s-1} \\
&= \lambda_k^{2\radius-2s} \left(e_k^T Q \tN_{u,s} - \lambda_k e_k^T \cQ \tW_{u,s}\right) \\
&= \lambda_k^{2\radius-2s} \left(\frac{1}{|\cS_{u,s}|} \sum_{i \in [n]} N_{u,s}(i) q_k(\theta_i) - \lambda_k e_k^T \cQ \tW_{u,s}\right)\\
&= \lambda_k^{2\radius-2s} \Bigg(\frac{1}{|\cS_{u,s}|} \sum_{i \in \cS_{u,s}} \sum_{a=(a_1, a_2) \in \cU_{u,s}} W_{u,s}(a) \Ind_{(a = \pi(i))} M_1(a_1, a_2,i) q_k(\theta_i) - \lambda_k e_k^T \cQ \tW_{u,s}\Bigg).
\end{align*}
Let us define
\begin{align*}
X_i &= \sum_{a=(a_1, a_2) \in \cU_{u,s}} W_{u,s}(a) \Ind_{(a = \pi(i))} M_1(a_1, a_2,i) q_k(\theta_i) \\
&= \sum_{a=(a_1, a_2) \in \cU_{u,s}} W_{u,s}(a) \Ind_{(a = \pi(i))} (f(\theta_{a_1}, \theta_{a_2}, \theta_i) + \epsilon_{a_1 a_2 i}) q_k(\theta_i).
\end{align*}
The randomness in $X_i$ only depends on $\theta_i, \epsilon_{a_1 a_2 i}, \Ind_{(a = \pi(i))}$. Note that we already conditioned on $\theta_{a_1}, \theta_{a_2}$ for $a \in \cU_{u,s} \subset \cB_{2s-1}$. $X_i$ is independent from $X_j$ because the vertices and edges are disjoint, and $\pi(i)$ is independent from $\pi(j)$ as different vertices are allowed to have the same (or different) parents.
First we compute the mean of $X_i$ (conditioned on $i \in \cS_{u,s}$). For any vertex $i \in \cS_{u,s}$, it must have exactly one parent in $\cU_{u,s}$ due to the BFS tree constraints. The parent is equally likely to be any vertex in $\cU_{u,s}$ due to the symmetry in the randomly sampled observations. Because the additive noise terms are mean zero, the eigenfunctions are orthonormal, and $\pi(i)$ is equally likely to be any $a \in \cU_{u,s}$,
\begin{align*}
\E[X_i ~|~ i \in \cS_{u,s}] 
&= \E\left[\sum_{a =(a_1,a_2) \in \cU_{u,s}} W_{u,s}(a) \Ind_{(a = \pi(i))} (f(\theta_{a_1}, \theta_{a_2}, \theta_i) + \epsilon_{a_1 a_2 i}) q_k(\theta_i)\right] \\
&= \sum_{a =(a_1,a_2)\in \cU_{u,s}} \frac{1}{|\cU_{u,s}|} \E\left[ W_{u,s}(a) \sum_h \lambda_h q_h(\theta_{a_1}) q_h(\theta_{a_2}) q_h(\theta_i) q_k(\theta_i)\right] \\
&= \sum_{a =(a_1,a_2)\in \cU_{u,s}} \tW_{u,s}(a) \lambda_k q_k(\theta_{a_1}) q_k(\theta_{a_2})  \\
&= e_k^T \Lambda \cQ \tW_{u,s}.
\end{align*}
Furthermore, $|X_i| \leq B$ almost surely as we assumed $|q_k(\theta)| \leq B$.
%
%
%
By Hoeffding's inequality, it follows that 
\begin{align*}
\Prob{|D_{u,2s} | \geq z ~|~ \cF_{u,2s-1}, \cS_{u,s}} 
&\leq 2 \exp\left(-\frac{ 2 |\cS_{u,s}| z^2}{\lambda_k^{4\radius-4s} B^2}\right).
\end{align*}
If we condition on the event $\cA^3_{u,\radius}(\delta)$, 
\[|\cS_{u,s}| \geq (1-\delta)^{2s} 2^{-3s-1} n^{2\extra s} (1-o(1))\] 
for $s \in [\radius]$. Therefore, 
\begin{align*}
\Prob{|D_{u,2s} | \geq z ~|~ \cF_{u,2s-1}, \cA^3_{u,\radius}(\delta)} 
&\leq  2 \exp\left(-\frac{2 (1-\delta)^{2s} 2^{-3s-1} n^{2\extra s} (1-o(1)) z^2}{\lambda_k^{4\radius-4s}B^2}\right).
\end{align*}
We finish the proof by using Lemma \ref{lem:subexp} with $c=2$ and $Q = \frac{B^2\lambda_k^{4\radius-4s} 2^{3s} (1+o(1))} {(1-\delta)^{2s} n^{2\extra s}}$.  
\endproof

\begin{lemma}\label{lem:martingale_diff_2}
For any $s \in [\radius]$, 
$$\E\left[D_{u,2s-1} ~|~ \cF_{u,2s-1}, \cA^3_{u,\radius}(\delta)\right] = 0.$$
Let $\nu = \sqrt{\frac{Q(1 + 16 \pi)}{2}}$, and
\[ Q = \frac{72 \lambda_k^{4\radius-4s+2} B^4 2^{3s} (1+o(1))}{(1-\delta)^{2s-1} n^{\min\{1, \frac12 + \extra (2s-1)\}}}.\]
For any $\lambda \in \Reals$,
\[ \E\left[e^{\lambda D_{u,2s-1}} ~|~ \cF_{u,2s-1}, \cA^3_{u,\radius}(\delta)\right] \leq \exp\Big(\frac{\lambda^2 \nu^2}{2}\Big).\]
\end{lemma}

\proof
As $2s-1$ is odd, 
\begin{align*}
D_{u,2s-1} &= Y_{u,2s-1} - Y_{u,2(s-1)} \\
&= \lambda_k^{2\radius-2s+1} \left(e_k^T \cQ \tW_{u,s} - \lambda_k e_k^T Q \tN_{u,s-1}\right).
\end{align*}
Recall $\cF_{2s-2}$ contains all information in the depth $2s-2$ neighborhood of vertex $u$. In particular this includes the vertex set 
\[\cB_{u,2s-2} = \cup_{l \in [s-1]} \cU_{u,l} \cup_{h \in [s-1]} \cS_{u,h},\]
 the vertex latent variables $\{\theta_i\}_{i \in \cB_{u,2(s-1)}}$ and the edges and corresponding weights $\{M_1(i,a)\}_{i,a \in \cB_{u,2(s-1)}}$.
Consider
 \begin{align*}
e_k^T \cQ \tW_{u,s} 
&= \frac{1}{|\cU_{u,s}|} \sum_{i=(i_1, i_2) \in \cU_{u,s}} W_{u,s}(i) q_k(\theta_{i_1}) q_k(\theta_{i_2}) \\
&= \frac{1}{|\cU_{u,s}|} \sum_{i=(i_1, i_2) \in \cU_{u,s}} \sum_{v \in \cS_{u,s-1}} N_{u,s-1}(v) \Ind_{(v = \pi(i))} M_1(v,i)q_k(\theta_{i_1}) q_k(\theta_{i_2}) \\
&= \frac{1}{|\cU_{u,s}|} \sum_{i=(i_1, i_2) \in \cU_{u,s}} X_i,
\end{align*}
where we define for $i =(i_1, i_2) \in  \cU_{u,s}$,
\begin{align*}
X_i &= \sum_{v \in \cS_{u,s-1}} N_{u,s-1}(v) \Ind_{(v = \pi(i))} M_1(v,i) q_k(\theta_{i_1}) q_k(\theta_{i_2}) \\
&= \sum_{v \in \cS_{u,s-1}} N_{u,s-1}(v) \Ind_{(v = \pi(i))} \left(\sum_l \lambda_l q_l(\theta_v) q_l(\theta_{i_1}) q_l(\theta_{i_2}) +\epsilon_{v i_1 i_2}\right) q_k(\theta_{i_1}) q_k(\theta_{i_2}).
\end{align*}
Conditioned on $\cF_{u,2(s-1)}$ the randomness in $X_i$ only depends on $\theta_{i_1}, \theta_{i_2}$, $\epsilon_{\pi(i) i_1 i_2}$, and 
$\Ind_{(v=\pi(i))}$.
Conditioned on $\cU_{u,s}$ and $\{\theta_{i_1}, \theta_{i_2}\}_{i \in \cU_{u,s}}$, the random variables $X_i$ are independent as 
$\epsilon_{\pi(i) i_1 i_2}$ and $\Ind_{(v=\pi(i))}$ are independent. The parent of $i=(i_1,i_2) \in \cU_{u,s}$ is equally likely to be any vertex in $\cS_{u,s-1}$, and the parent of $i=(i_1,i_2)  \in \cU_{u,s}$ is independent from the parent of $j =(j_1,j_2) \in \cU_{u,s}$ with $j \neq i$ 
as different vertices are allowed to have the same (or different) parent.
First we compute the mean of $X_i$ conditioned on $i \in \cU_{u,s}$ and $\theta_{i_1}, \theta_{i_2}$. Because the additive noise terms are mean zero and the parent of $i$ is equally likely to be any $v \in \cS_{u,s-1}$,
\begin{align*}
\E[X_i ~|~ \theta_{i_1}, \theta_{i_2}, i \in \cU_{u,s}] 
&= \E[\sum_{v \in \cS_{u,s-1}} N_{u,s-1}(v) \Ind_{(v = \pi(i))} \left(\sum_l \lambda_l q_l(\theta_v) q_l(\theta_{i_1}) q_l(\theta_{i_2})\right) q_k(\theta_{i_1}) q_k(\theta_{i_2}) ~|~ \theta_{i_1}, \theta_{i_2}] \\
&= \frac{1}{|\cS_{u,s-1}|} \sum_{v \in \cS_{u,s-1}} N_{u,s-1}(v) \left(\sum_l \lambda_l q_l(\theta_v) q_l(\theta_{i_1}) q_l(\theta_{i_2})\right) q_k(\theta_{i_1}) q_k(\theta_{i_2}).
\end{align*}
Furthermore, $|X_i| \leq B^2$ almost surely as we assumed $|q_k(\theta)| \leq B$. 
%
By Hoeffding's inequality, it follows that 
\begin{align} \label{eq:mart_diff_conc1}
&\Prob{|e_k^T \cQ \tW_{u,s} - \E[e_k^T \cQ \tW_{u,s}  ~|~ \{\theta_{i_1}, \theta_{i_2}\}_{i \in \cU_{u,s}}, \cU_{u,s}]| > z ~|~ \cF_{u,2(s-1)}, \cU_{u,s}, \{\theta_{i_1}, \theta_{i_2}\}_{i \in \cU_{u,s}}} \nonumber \\
&\leq 2 \exp\left(-\frac{|\cU_{u,s}| z^2}{B^4}\right).
\end{align}
Next we consider concentration with respect to the random subset $\cU_{u,s}$ out of the $\cV_A \setminus \cB_{u,2(s-1)}$ possible vertices. In particular we would like to argue that with high probability,
\begin{align*}
&\frac{1}{|\cU_{u,s}|} \sum_{i \in \cU_{u,s}} \sum_l \lambda_l q_l(\theta_v) q_l(\theta_{i_1}) q_l(\theta_{i_2}) q_k(\theta_{i_1}) q_k(\theta_{i_2}) \\
&\qquad\approx \frac{1}{|\cV_A \setminus \cB_{u,2(s-1)}|} \sum_{i \in \cV_A \setminus \cB_{u,2(s-1)}} \sum_l \lambda_l q_l(\theta_v) q_l(\theta_{i_1}) q_l(\theta_{i_2}) q_k(\theta_{i_1}) q_k(\theta_{i_2}).
\end{align*}
Next, we formalize it. To that end, conditioned on the size $|\cU_{u,s}|$, the set $\cU_{u,s}$ is a uniform random sample of the possible set of vertices $\cV_A \setminus \cB_{u,2(s-1)}$. The above expression on the left is thus the mean of a random sample $\cU_{u,s}$ without replacement from $\cV_A \setminus \cB_{u,2(s-1)}$. Due to negative dependence, it concentrates around its means no slower than assuming that they were a
sample of the same size from the same population with replacement, cf. \cite[Theorem 4]{hoeffding1963}. Therefore, using 
\begin{align*}
\sum_l \lambda_l q_l(\theta_v) q_l(\theta_{i_1}) q_l(\theta_{i_2}) q_k(\theta_{i_1}) q_k(\theta_{i_2}) 
&= |f(\theta_v, \theta_{i_1}, \theta_{i_2}) q_k(\theta_{i_1}) q_k(\theta_{i_2})|
\leq B^2,
\end{align*}
we can apply Hoeffding's inequality to argue that
\begin{align} \label{eq:mart_diff_conc2}
&\bP\Big(\Big|\frac{1}{|\cU_{u,s}|} \sum_{i \in \cU_{u,s}} q_l(\theta_{i_1}) q_l(\theta_{i_2}) q_k(\theta_{i_1}) q_k(\theta_{i_2}) \nonumber \\
&\quad\quad- \frac{1}{|\cV_A \setminus \cB_{u,2(s-1)}|} \sum_{i \in \cV_A \setminus \cB_{u,2(s-1)}} q_l(\theta_{i_1}) q_l(\theta_{i_2}) q_k(\theta_{i_1}) q_k(\theta_{i_2})\Big| \geq z ~|~ \{\theta_{i_1}, \theta_{i_2}\}_{i \in \cV_A}, |\cU_{u,s}|\Big) \nonumber \\
&\leq 2 \exp\Big(-\frac{|\cU_{u,s}| z^2}{B^4}\Big).
\end{align}
Finally, we need to account for the randomness in $\{\theta_{i_1}, \theta_{i_2}\}_{i \in \cV_A}$, arguing that with high probability
\[\frac{1}{|\cV_A \setminus \cB_{u,2(s-1)}|} \sum_{i \in \cV_A \setminus \cB_{u,2(s-1)}} \sum_l \lambda_l q_l(\theta_v) q_l(\theta_{i_1}) q_l(\theta_{i_2}) q_k(\theta_{i_1}) q_k(\theta_{i_2}) \approx \lambda_k q_k(\theta_v).\]
To formalize this, we start by recalling that 
\[\cV_A \setminus \cB_{u,2(s-1)} = \{(i_1, i_2) ~\text{s.t.}~ i_1 < i_2, \{i_1, i_2\} \subset [n/2] \setminus \cB_{u,2(s-1)}\}.\]
Let $n_{u,s} = |[n/2] \setminus \cB_{u,2(s-1)}|$, then $|\cV_A \setminus \cB_{u,2(s-1)}| = \binom{|n_{u,s}|}{2}$. Then the above summation can be written as a pairwise U-statistic, 
\[U = \frac{1}{|\cV_A \setminus \cB_{u,2(s-1)}|} \sum_{(i_1, i_2) \in \cV_A \setminus \cB_{u,2(s-1)}} g(\theta_{i_1}, \theta_{i_2})\]
where $g$ is a symmetric function and each term $g(\theta_{i_1}, \theta_{i_2})$ is bounded in absolute value by $B^2$. Furthermore, 
\[\E\left[\sum_l \lambda_l q_l(\theta_{i_2}) q_k(\theta_{i_1}) q_k(\theta_{i_2})\right] =\lambda_k q_k(\theta_v)\]
by the orthogonality model assumption. Therefore, by 
Lemma \ref{lem:ustats}
\begin{align} \label{eq:mart_diff_conc3}
&\Prob{\Big|\frac{1}{|\cV_A \setminus \cB_{u,2(s-1)}|} \sum_{i \in \cV_A \setminus \cB_{u,2(s-1)}} \sum_l \lambda_l q_l(\theta_{i_1}) q_l(\theta_{i_2}) q_k(\theta_{i_1}) q_k(\theta_{i_2})  - \lambda_k q_k(\theta_v)\Big| \geq z} \nonumber \\
&\quad\leq 2 \exp\Big(-\frac{n_{u,s} z^2}{8 B^4}\Big).
\end{align}
By putting together all calculations, it also follows that
\[\E[e_k^T \cQ \tW_{u,s}] = e_k^T \Lambda Q \tN_{u,s-1},\]
and for $z_1, z_2, z_3 > 0$, with probability at least
\begin{align*}
1 - 2 \exp\Big(-\frac{|\cU_{u,s}| z_1^2}{B^4}\Big)
- 2 \exp\Big(-\frac{|\cU_{u,s}| z_2^2}{B^4}\Big)
- 2 \exp\Big(-\frac{n_{u,s} z_3^2}{8 B^4}\Big)
\end{align*}
it holds that 
\begin{align*}
&|e_k^T \cQ \tW_{u,s} -  e_k^T \Lambda Q \tN_{u,s-1}| \\
&\leq |e_k^T \cQ \tW_{u,s} - \E[e_k^T \cQ \tW_{u,s}  ~|~ \{\theta_{i_1}, \theta_{i_2}\}_{i \in \cU_{u,s}}, \cU_{u,s}]| \\
&~+ \Big|\frac{1}{|\cS_{u,s-1}|} \sum_{v \in \cS_{u,s-1}} N_{u,s-1}(v) \Big(\frac{1}{|\cU_{u,s}|} \sum_{i \in \cU_{u,s}} \sum_l \lambda_l q_l(\theta_v) q_l(\theta_{i_1}) q_l(\theta_{i_2}) q_k(\theta_{i_1}) q_k(\theta_{i_2}) \\
&\qquad\qquad\qquad\qquad\qquad\qquad\qquad - \frac{1}{|\cV_A \setminus \cB_{u,2(s-1)}|} \sum_{i \in \cV_A \setminus \cB_{u,2(s-1)}} \sum_l \lambda_l q_l(\theta_v)q_l(\theta_{i_1}) q_l(\theta_{i_2}) q_k(\theta_{i_1}) q_k(\theta_{i_2}) \Big)\Big| \\
&~ + \Big|\frac{1}{|\cS_{u,s-1}|} \sum_{v \in \cS_{u,s-1}} N_{u,s-1}(v) \Big(\frac{\sum_{i \in \cV_A \setminus \cB_{u,2(s-1)}} \sum_l \lambda_l q_l(\theta_v) q_l(\theta_{i_1}) q_l(\theta_{i_2}) q_k(\theta_{i_1}) q_k(\theta_{i_2})}{|\cV_A \setminus \cB_{u,2(s-1)}|} - \lambda_k q_k(\theta_v)\Big)\Big| \\
&\leq z_1 + \frac{1}{|\cS_{u,s-1}|} \sum_{v \in \cS_{u,s-1}} |N_{u,s-1}(v)| z_2 
 + \frac{1}{|\cS_{u,s-1}|} \sum_{v \in \cS_{u,s-1}} |N_{u,s-1}(v)| z_3 \\
&\leq z_1 + z_2 + z_3,
\end{align*}
since $\|N_{u,s-1}\|_\infty \leq 1$.  Conditioned on $\cA^3_{u,\radius}(\delta)$, $n_{u,s} = n/2(1-o(1))$, and 
\[|\cU_{u,s}| \in \left[(1-\delta)^{2s-1} 2^{-3s} n^{\frac12 + \extra (2s-1)} (1-o(1)), (1+\delta)^{2s-1} 2^{-s} n^{\frac12 + \extra (2s-1)}\right].\]
As a result, for $z_1 = z_2 = z_3$, the expression in \eqref{eq:mart_diff_conc1} and \eqref{eq:mart_diff_conc2} asymptotically dominate the expression in \eqref{eq:mart_diff_conc3}. It follows that, with appropriate 
choice of $z_1 = z_2 = z_3$ in the above, 
\begin{align*}
\Prob{|D_{u,2s-1}| \geq z ~|~ \cF_{2s-2},\cA^3_{u,\radius}(\delta)} 
&\leq 6 \exp\left(-\frac{(1-\delta)^{2s-1} 2^{-3s} n^{\min\{1, \frac12 + \extra (2s-1)\}} (1-o(1)) z^2}{72 \lambda_k^{4\radius-4s+2} B^4}\right).
\end{align*}
We finish the proof by using Lemma \ref{lem:subexp} with $c = 4$ and 
$Q = \frac{72 \lambda_k^{4\radius-4s+2} B^4 2^{3s} (1+o(1))}{(1-\delta)^{2s-1} n^{\min\{1, \frac12 + \extra (2s-1)\}}}$.

\endproof
\subsection{Concentration of Quadratic Form Two}

Lemma \ref{lemma:nhbrhd_vectors} suggests the following high probability events: 
for any $u \in [n], k \in [r]$, $\radius$ as defined in \eqref{eq:t}, i.e. 
$\radius = \Big\lceil \frac{1}{4 \extra} \Big\rceil$, $\delta \in (0,1)$, and
\[x = \frac{ 16 \lambda_{\max}^{2\radius-2} n^{\psi}} {(1-\delta) n^{\extra}}.\]
define 
\[
\cA^4_{u,k,\radius}(x, \delta) = \Big\{|e_k^T Q \tN_{u,\radius} - e_k^T \Lambda^{2\radius} Q e_u| < x \Big\} 
\cap \cA^3_{u,\radius}(\delta).
\]
Now, we state a useful concentration that builds on the above condition holding. It will be useful step
towards establishing Lemma \ref{lemma:dist}. 

\begin{lemma} \label{lemma:NFN_bound}
Let $p = n^{-3/2+\extra}$ for $\extra \in (0,\frac12)$, 
$\radius$ as defined in \eqref{eq:t}, 
and $\delta \in (0,\frac12)$. For any $u, v \in [n]$, conditioned on 
$\cap_{k=1}^r \big(\cA^4_{u,k,\radius}(x, \delta) \cap \cA^4_{v,k,\radius}(x, \delta)\big)$,
\begin{align*}
\big|\tN_{u,\radius}^T Q^T \Lambda^2 Q \tN_{v, \radius} - e_u^T Q^T \Lambda^{2(2\radius + 1)} Q e_v \big|
& \leq x^2 \Bigg(\sum_{k=1}^r \lambda_k^2\Bigg) + x B \Bigg(\sum_{k=1}^r 2\lambda_k^{2(\radius+1)}\Bigg).
\end{align*}
\end{lemma}
\proof{Proof of Lemma \ref{lemma:NFN_bound}.}
Assuming event $\cap_{k=1}^r \big(\cA^4_{u,k,\radius}(x, \delta) \cap \cA^4_{v,k,\radius}(x, \delta) \big)$ holds,
\begin{align}
|\tN_{u,\radius}^T Q^T \Lambda^2 Q \tN_{v, \radius} - e_u^T Q^T \Lambda^{2(2\radius + 1)} Q e_v| 
&\leq |(\tN_{u,\radius}^T Q^T - e_u^T Q^T \Lambda^{2\radius}) (\Lambda^2 Q \tN_{v,\radius} - \Lambda^{2(\radius+1)} Q e_v)| \nonumber \\
&\qquad + |(\tN_{u,\radius}^T Q^T - e_u^T Q^T \Lambda^{2\radius}) \Lambda^{2(\radius+1)} Q e_v| \nonumber \\
&\qquad+ |e_u^T Q^T \Lambda^{2(\radius + 1)} (Q \tN_{v, \radius} - \Lambda^{2\radius} Q e_v)| \nonumber\\
&\leq \Big|\sum_{k=1}^r (e_k ^T Q \tN_{u,\radius} - e_k^T \Lambda^{2\radius} Q e_u) (e_k^T \Lambda^2 Q \tN_{v,\radius} - e_k^T \Lambda^{2(\radius+1)} Q e_v)\Big|\nonumber \\
&\qquad + \Big|\sum_{k=1}^r (e_k ^T Q \tN_{u,\radius} - e_k^T \Lambda^{2\radius} Q e_u) e_k^T \Lambda^{2(\radius+1)} Q e_v \Big| \nonumber\\
& \qquad  + \Big|\sum_{k=1}^r (e_k^T \Lambda^{2(\radius + 1)} Q e_u) (e_k^T Q \tN_{v,\radius} - e_k^T \Lambda^{2\radius} Q e_v)\Big|.  
\end{align}
In above, we have simply used the fact that for two vectors $a, b \in \mathbb{R}^r$, 
$a^T b = \sum_k a_k b_k = \sum_k (e_k^T a)(e_k^T b)$. Now, consider the first term on the right hand side of the last inequality. 
If $\cap_{k=1}^r \cA^4_{u,k,\radius}(x, \delta)$ holds, then $|(e_k ^T Q \tN_{u,\radius} - e_k^T \Lambda^{2\radius} Q e_u)| \leq x$. And
if $\cap_{k=1}^r  \cA^4_{v,k,\radius}(x, \delta)$ holds, then $|(e_k^T \Lambda^2 Q \tN_{v,\radius} - e_k^T \Lambda^{2(\radius+1)} Q e_v)| \leq \lambda_k^2 x$. Similar application to other terms and the fact that 
$|e_k^T Q e_u|, |e_k^T Q e_v|  \leq \|q_k(\cdot)\|_\infty \leq B$, we conclude that 
\begin{align}
|\tN_{u,\radius}^T Q^T \Lambda^2 Q \tN_{v,\radius} - e_u^T Q^T \Lambda^{2(2\radius + 1)} Q e_v| 
&\leq x^2 \Bigg(\sum_{k=1}^r \lambda_k^2\Bigg) + x B \Bigg(\sum_{k=1}^r 2 \lambda_k^{2(\radius+1)}\Bigg).
\end{align}

\endproof

\subsection{Concentration of Quadratic Form Three} \label{sec:conc_quadratic_3}

We establish a final concentration that will lead us to the proof of good distance function property. 
For any $u \in [n]$, define event
\begin{align}\label{eq:aprime}
\Ap_{u, v, \radius}(x, \delta) & = \cap_{k=1}^r \big(\cA^4_{u,k,\radius}(x, \delta) \cap \cA^4_{v,k,\radius}(x, \delta)\big). 
\end{align}

\begin{lemma} \label{lemma:NMN_conc}
Let $p = n^{-3/2+\extra}$ for $\extra \in (0,\frac12)$, 
$\radius$ as defined in \eqref{eq:t}, 
$\delta \in (0,\frac12)$, and
\[x = \frac{ 16 \lambda_{\max}^{2\radius-2} n^{\psi}} {(1-\delta) n^{\extra}}.\]
Let $S \equiv S_{u,v,\radius} = [n] \backslash (\cB_{u,2\radius} \cup \cB_{v, 2\radius} \cup [n/2])$.   
Then, under event 
$\Ap_{u, v, \radius}(x, \delta)$, 
\begin{align*}
&\left|\frac{1}{\binom{|S|}{2} p^2 |\cS_{u,\radius}| |\cS_{v,\radius}|} \sum_{\alpha < \beta \in S \times S} T(\alpha, \beta)  - \tN_{u,\radius}^T Q^T \Lambda^2 Q \tN_{v,\radius}  \right| 
 = O\left(\frac{n^{\psi}}{(|S|^2 p^2 |\cS_{u,\radius}| |\cS_{v,\radius}|)^{1/2}}\right) + O\left(\frac{n^{\psi}}{|S|^{1/2}}\right)
\end{align*}
with probability at least $1 - 4\exp(-n^{2\psi}(1-o(1))) - O(n^{-6})$ with $\psi \in (0,\extra)$.
\end{lemma}

\proof
 First, note that $\Ap_{u, v, \radius}(x, \delta)$ includes events $\cA^3_{u,\radius}(\delta)$ and $\cA^3_{v,\radius}(\delta)$. This implies that $|S| = \frac{n}{2} - o(n) = \frac{n(1-o(1))}{2}$.
Furthermore, it implies that $|\cS_{u,\radius}|$ and $|\cS_{v,\radius}|$ are both greater than or equal to $(1-\delta)^{2\radius} 2^{-3\radius-1} n^{2\extra \radius} (1-o(1))$. As a result, 
\begin{align}
|S|^2 p^2 |\cS_{u,\radius}| |\cS_{v,\radius}| 
&\geq \frac{n^2}{8} p^2 \left(\frac{(1-\delta)^{2}}{8} n^{2\extra}\right)^{2\radius} (1-o(1)) \\
&\geq n^{2 + 4\extra\radius} n^{2(-\frac32 + \extra)} \frac{1}{8}\left(\frac{(1-\delta)^2}{8}\right)^{2\radius} (1-o(1)) \\
&= \Theta(n^{-1 + 2\extra(2\radius+1)}) \\
&= \Omega(n^{2\extra}). \label{eq:SpS_bd}
\end{align}
The asymptotic relationships follow from the choice of $\radius \geq \frac{1}{4\extra}$, and the fact that $\delta$ and $\radius$ are both constants.



Recall that $\Mpp(a, (\alpha, \beta)) = \Ind_{((a, \alpha, \beta) \in \cEpp)} (F(a,\alpha,\beta) + \epsilon_{a \alpha \beta})$ for $$F(a,\alpha,\beta)= \sum_{k=1}^r \lambda_k q_k(\theta_a) q_k(\theta_{\alpha}) q_k(\theta_{\beta}).$$ There are 3 sources of randomness: the sampling of entries in $\cEpp$, the observation noise terms $\epsilon_{a \alpha \beta}$, and the latent variables $\theta_a, \theta_{\alpha}, \theta_{\beta}$. Since we enforce that $\alpha$ and $\beta$ are in the complement of $\cB_{u,2\radius} \cup \cB_{v, 2\radius}$, the sampling, observations, and latent variables involved in ${\Mpp}$ are independent from $N_{u,\radius}$ and $N_{v, \radius}$. 

Let us define the quantity 
\begin{align*}
	\tilde{T}(\alpha,\beta) &= \min(\max(T(\alpha,\beta),-\phi^2), \phi^2) \\
	&= \sgn(T(\alpha,\beta))\min(|T(\alpha,\beta)|,\phi^2)
\end{align*}
for $\phi = \lceil 16/(1-2\extra)\rceil$ where recall that
\[T(\alpha, \beta) = \sum_{a \neq b \in [n]} N_{u, \radius}(a) N_{v, \radius}(b)  \Mpp(a, (\alpha,0 \beta)) \Mpp(b, (\alpha, \beta)).\] 
Trivially, due to this thresholding, $|\tilde{T}(\alpha, \beta)| \leq \phi^2$ such that $|\tilde{T}(\alpha, \beta) - \E[\tilde{T}(\alpha,\beta)]| \leq 2\phi^2$.

To begin with, $N_{u,  \radius}(a) = 0$ if $a \notin \cS_{u, \radius} \subset \cB_{u,2 \radius}$ and $N_{v,  \radius}(b) = 0$ if $b \notin \cS_{v, \radius} \subset \cB_{v,2\radius}$. Further, conditioned on event  $\Ap_{u, v,\radius}(x, \delta)$, all the information associated with 
$\cB_{u, 2\radius}$ and $\cB_{v, 2\radius}$ is revealed; however, information about $[n] \backslash (\cB_{u, 2\radius} \cup \cB_{v, 2\radius})$ is not. Let $\cF(u, v, \radius, x, \delta)$ denote all the information revealed such that event $\Ap_{u, v, \radius}(x, \delta)$ holds.

Let's prove concentration in two steps. In step one, we condition on $\cF(u,v,\radius,x,\delta)$ and the latent variables $\{\theta_i\}_{i \in [n]}$. The sampling process (edges in $\cEpp$) and the observation noise are independent for distinct pairs $(\alpha, \beta)$ and $(\alpha',\beta')$. As a result, $T(\alpha,\beta)$ and $T(\alpha',\beta')$ are conditionally independent as long as $\{\alpha,\beta\} \cap \{\alpha',\beta'\} \neq 2$, i.e. they are not the exact same pair. The correlations across $T(\alpha,\beta)$ and $T(\alpha',\beta')$ are due only to the latent variables if $\alpha,\beta,\alpha',\beta'$ share any values. We will bound the variance of $T(\alpha,\beta)$ in Lemma \ref{lem:VarT}, and by combining it with the conditional independence property across $T(\alpha,\beta)$, it follows that (using notation $\cF = \cF(u, v, \radius, x, \delta)$)
\begin{align*}
\Var\left[\left.\sum_{\alpha < \beta \in S \times S} T(\alpha, \beta)  ~\right|~ \cF, \{\theta_i\}_{i\in [n]}\right] 
&= \sum_{\alpha < \beta \in S \times S} \Var\left[T(\alpha, \beta)  ~|~ \cF, \{\theta_i\}_{i\in [n]}\right] \\
&\leq 2 \binom{|S|}{2} p^2 |\cS_{u,\radius}| |\cS_{v,\radius}| (1+ o(1)).
\end{align*}

The variables $\tilde{T}(\alpha,\beta)$ are also independent across $(\alpha,\beta)$ conditioned on the latent variables $\{\theta_i\}_{i\in [n]}$, and their variance is bounded above by the corresponding variances of $T(\alpha, \beta)$. Using the boundedness of $\tilde{T}(\alpha,\beta)$, by applying Bernstein's inequality with the choice of $z = 2 n^{\psi} \left(\binom{|S|}{2} p^2 |\cS_{u,\radius}| |\cS_{v,\radius}|\right)^{1/2}$ for $\psi \in (0,\extra)$, it follows that
\begin{align}
&\Prob{\Big|\sum_{\alpha < \beta \in S \times S} (\tilde{T}(\alpha, \beta) - \E[\tilde{T}(\alpha,\beta) ~|~ \cF, \{\theta_i\}_{i\in [n]}])\Big| \geq z ~\Bigg|~ \cF(u, v, \radius, x, \delta), \{\theta_i\}_{i\in [n]}} \nonumber \\ 
&\qquad\leq 2 \exp\left(- \frac{\frac{z^2}{2}}{ 2 \binom{|S|}{2} p^2 |\cS_{u,\radius}| |\cS_{v,\radius}| (1 + o(1)) + \frac{2 \phi^2 z}{3}}\right) \nonumber \\
&\qquad= 2  \exp(-n^{2\phi}(1-o(1))). \label{eq:T_tilde_Berstein}
\end{align}
The last equality arises from the observation that $\radius$ is chosen such that conditioned on $\cF$, we can plug in \eqref{eq:SpS_bd} to show that for our choice of $z$, it holds that 
\[z = o(|S|^2 p^2 |\cS_{u,\radius}| |\cS_{v,\radius}|).\]

In Lemma \ref{lem:tail_T}, we will show a bound on $\Prob{|T(\alpha,\beta)|\geq \phi^2}$, which translates to a bound on $\E[\Ind_{|T(\alpha,\beta)| \geq \phi^2} (|T(\alpha,\beta)| - \phi^2) ~|~ \cF]$, which then upper bound the difference between the conditional expectations of $T$ and $\tilde{T}$ according to
\begin{align*}
|\E[T(\alpha,\beta) ~|~ \cF] - \E[\tilde{T}(\alpha,\beta) ~|~ \cF]| 
&\leq \E[ |T(\alpha,\beta) - \tilde{T}(\alpha,\beta)| ~|~ \cF] \\
&= \E[|T(\alpha,\beta)| - \min(|T(\alpha,\beta)|,\phi^2) ~|~ \cF] \\
&= \E[\Ind_{|T(\alpha,\beta)| \geq \phi^2} (|T(\alpha,\beta)| - \phi^2) ~|~ \cF]
\end{align*}

Using this bound from Lemma \ref{lem:tail_T} along with the conditions from $\cF$ that guarantee $|S| = \Theta(n)$ and naively $|\cS_{u,\radius} \cup \cS_{v,\radius}| = O(n)$, it follows that
\begin{align}
&\Big|\sum_{\alpha < \beta \in S \times S} (\E[\tilde{T}(\alpha,\beta)~|~ \cF, \{\theta_i\}_{i\in [n]}] - \E[T(\alpha,\beta)~|~ \cF, \{\theta_i\}_{i\in [n]}])\Big| \nonumber \\
&\qquad\leq (1 + o(1)) \binom{|S|}{2} \frac{2 \phi}{\ln(|\cS_{u,\radius} \cup \cS_{v,\radius}|^{-1} p^{-1})} (|\cS_{u,\radius} \cup \cS_{v,\radius}| p)^{\phi} \nonumber \\
&\qquad= O\left(|S|^2 \left(|\cS_{u,\radius} \cup \cS_{v,\radius}| p\right)^{\phi} \right) \nonumber \\
&\qquad= O\left(n^2 \left(n^{-(\frac12 - \extra)}\right)^{\phi} \right). \label{eq:exp_T_diff}
\end{align}
We choose $\phi = \lceil \frac{16}{1-2\extra} \rceil \geq \frac{16}{1-2\extra}$ so that this difference between the expectations of $T$ and $\tilde{T}$ is $O(n^{-6})$.

By plugging in our choice of $\phi$ into Lemma \ref{lem:tail_T}, it also follows that 
\begin{align} \label{eq:T_eq_T_tilde}
\Prob{\cup_{\alpha,\beta} \{\tilde{T}(\alpha,\beta) \neq T(\alpha,\beta)\} ~|~ \cF} \leq O(n^{-6}).
\end{align}

By combining \eqref{eq:T_tilde_Berstein}, \eqref{eq:exp_T_diff}, and \eqref{eq:T_eq_T_tilde}, with probability at least $1 - 2\exp(-n^{2 \psi}(1-o(1))) - O(n^{-6})$,
\begin{align}
& \left|\sum_{\alpha < \beta \in S \times S} \left(T(\alpha, \beta) - \E\left[T(\alpha,\beta)~|~ \cF, \{\theta_i\}_{i\in [n]}\right]\right)\right|\leq 2 n^{\psi} \left(\binom{|S|}{2} p^2 |\cS_{u,\radius}| |\cS_{v,\radius}|\right)^{1/2} + O(n^{-6}), \label{eq:conc_T_tilde_ET}
\end{align}
where the first term will dominate the second term.

Finally we want to show concentration of the following expression with respect to the latent variables,
\[\frac{1}{\binom{|S|}{2} p^2 |\cS_{u,\radius}| |\cS_{v,\radius}|} \E\left[\left.\sum_{\alpha < \beta \in S \times S} T(\alpha, \beta) ~\right|~ \cF, \{\theta_i\}_{i\in [n]}\right].\]
The expression can be written as a pairwise U-statistic,
\[U = \frac{1}{\binom{|S|}{2}} \sum_{\alpha < \beta \in S \times S} g(\theta_\alpha, \theta_\beta),\]
where $g$ is a symmetric function, and 
\begin{align*}
g(\theta_\alpha, \theta_\beta) 
&= \frac{1}{p^2 |\cS_{u,\radius}| |\cS_{v,\radius}|} \E\left[ T(\alpha, \beta) ~|~ \cF, \{\theta_i\}_{i\in [n]}\right] \\
&= \frac{1}{p^2 |\cS_{u,\radius}| |\cS_{v,\radius}|} \sum_{a \neq b \in [n]} N_{u,\radius}(a) N_{v,\radius}(b) \times 
\E\left[ \Mpp(a, (\alpha, \beta)) \Mpp(b, (\alpha, \beta)) ~|~ \cF, \{\theta_i\}_{i\in [n]}\right] \nonumber \\
& = \frac{1}{|\cS_{u,\radius}| |\cS_{v,\radius}|} \sum_{a \neq b \in [n]} N_{u,\radius}(a) N_{v,\radius}(b) F(a,\alpha,\beta) F(b,\alpha,\beta).
\end{align*}
It follows by boundedness of entries in $F$ and the fact that $\|N_{u,\radius}\|_{\infty} \leq 1$ and $\|N_{u,\radius}\|_0 = |\cS_{u,\radius}|$, that $|g(\theta_\alpha, \theta_\beta)| \leq 1$ almost surely. Therefore, by Lemma \ref{lem:ustats} and choosing $z = \sqrt{8} n^{\psi}|S|^{-1/2}$, 
\begin{align}
	&\Prob{\left|\frac{1}{\binom{|S|}{2} p^2 |\cS_{u,\radius}| |\cS_{v,\radius}|} \sum_{\alpha < \beta \in S \times S} \left(\E\left[T(\alpha, \beta) ~|~ \cF, \{\theta_i\}_{i\in [n]}\right]  - \E\left[T(\alpha, \beta) ~|~ \cF\right]\right)\right| \geq z} \nonumber \\
	&\quad\leq 2 \exp\Big(-\frac{|S| z^2}{8}\Big) = 2 \exp(n^{2\psi}). \label{eq:conc_ET_latentVar}
\end{align}

The expected value with respect to the randomness in the latent variables is 
\begin{align*}
&\frac{1}{\binom{|S|}{2} p^2 |\cS_{u,\radius}| |\cS_{v,\radius}|} \E\left[\left.\sum_{\alpha < \beta \in S \times S} T(\alpha, \beta) ~\right|~ \cF\right] \\
&\qquad= \frac{1}{\binom{|S|}{2}} \sum_{\alpha < \beta \in S \times S} \frac{1}{|\cS_{u,\radius}| |\cS_{v,\radius}|} \sum_{a \neq b \in [n]} N_{u,\radius}(a) N_{v,\radius}(b) \E[F(a,\alpha,\beta) F(b,\alpha,\beta)] \\
&\qquad= \frac{1}{\binom{|S|}{2}} \sum_{\alpha < \beta \in S \times S} \frac{1}{|\cS_{u,\radius}| |\cS_{v,\radius}|} \sum_{a \neq b \in [n]} N_{u,\radius}(a) N_{v,\radius}(b) \sum_{k} \lambda_k^2 q_k(\theta_a) q_k(\theta_b) \\
&\qquad= \tN_{u,\radius}^T Q^T \Lambda^2 Q \tN_{v,\radius} - \sum_{a \in [n]} \tN_{u,\radius}(a) \tN_{v,\radius}(a)\sum_{k} \lambda_k^2 q^2_k(\theta_a).
\end{align*}
Furthermore, 
\begin{align}
\left|\sum_{a \in [n]} \tN_{u,\radius}(a) \tN_{v,\radius}(a) \sum_k \lambda_k^2 q^2_k(\theta_a)\right|
&\leq \frac{B^2 \big(\sum_k \lambda_k^2 \big)}{\max(|\cS_{u, \radius}|,|\cS_{v, \radius}|)} \nonumber \\
&= O((|\cS_{u, \radius}||\cS_{v, \radius}|)^{-1/2}). \label{eq:diag_terms_bd}
\end{align}
By combining \eqref{eq:conc_T_tilde_ET}, \eqref{eq:conc_ET_latentVar}, and \eqref{eq:diag_terms_bd}, it follows that conditioned on $\cF(u, v, s, \ell, x, \delta)$, with probability 
\[1 - 2\exp\left(-n^{2\psi}(1-o(1))\right) - 2 \exp(n^{2\psi}) - O(n^{-6}),\]
it holds that 
\begin{align*}
&\left|\frac{1}{\binom{|S|}{2} p^2 |\cS_{u,\radius}| |\cS_{v,\radius}|} \sum_{\alpha < \beta \in S \times S} T(\alpha, \beta)  - \tN_{u,\radius}^T Q^T \Lambda^2 Q \tN_{v,\radius}  \right| \\
 &\qquad\leq 2 n^{\psi} \left(\binom{|S|}{2} p^2 |\cS_{u,\radius}| |\cS_{v,\radius}|\right)^{-1/2} + O(n^{-6}) + \frac{\sqrt{8} n^{\psi}}{|S|^{1/2}}\\
 &\qquad\quad+ o\left(\left(\binom{|S|}{2} p^2 |\cS_{u,\radius}| |\cS_{v,\radius}|\right)^{-1}\right) 
 + \frac{B^2 \big(\sum_k \lambda_k^2 \big)}{\max(|\cS_{u, \radius}|,|\cS_{v, \radius}|)} \\
 &\qquad\leq O\left(\frac{n^{\psi}}{(|S|^2 p^2 |\cS_{u,\radius}| |\cS_{v,\radius}|)^{1/2}}\right) + O\left(\frac{n^{\psi}}{|S|^{1/2}}\right) + O\left(\frac{1}{(|\cS_{u, \radius}||\cS_{v, \radius}|)^{1/2}}\right).
\end{align*}
Note that the third term is dominated by the first term as $|S| p = o(1)$. This completes the proof of Lemma \ref{lemma:NMN_conc}. 

\endproof

\begin{lemma} \label{lem:VarT}
Let $\cF = \cF(u, v, \radius, x, \delta)$ denote all the information revealed such that event $\Ap_{u, v, \radius}(x, \delta)$ holds.
\begin{align*}
\normalfont{\Var}[T(\alpha, \beta)  ~|~ \cF, \{\theta_i\}_{i\in [n]}]
&\leq 2 p^2 |\cS_{u,\radius}| |\cS_{v,\radius}| (1+ o(1)).
\end{align*}
\end{lemma}

\proof
To compute the variance of $T(\alpha,\beta)$ conditioned on $\cF, \{\theta_i\}_{i\in [n]}$, note that there is correlation in the terms within the sum of $T(\alpha,\beta)$ as there may be pairs $(a,b)$ and $(a', b')$ that share coordinates. In particular because the observation noise and sampling randomness for $\Mpp(a,b,c)$ is independent across different entries $(a,b,c)$, then conditioned on $\{\theta_i\}_{i\in [n]}$, for $a \neq b$ and $a' \neq b'$, if all four coordinates $\{a,b,a',b'\}$ are distinct,
\[\Cov[\Mpp(a, (\alpha, \beta)) \Mpp(b, (\alpha, \beta)), \Mpp(a', (\alpha, \beta)) \Mpp(b', (\alpha, \beta))] = 0;\]
if $|\{a,b\} \cap \{a',b'\}| = 2$, i.e. $(a',b') = (a,b)$ or $(a',b') = (b,a)$,
\begin{align*}
&\left| \Cov[\Mpp(a, (\alpha, \beta)) \Mpp(b, (\alpha, \beta)), \Mpp(a', (\alpha, \beta)) \Mpp(b', (\alpha, \beta))] \right|\\
& \qquad= \Var[\Mpp(a, (\alpha, \beta)) \Mpp(b, (\alpha, \beta))]\\
& \qquad\leq \E[\Mpp^2(a, (\alpha, \beta)) \Mpp^2(b, (\alpha, \beta))] \leq p^2; 
\end{align*}
and if $\{a,b\} \cup \{a',b'\} = \{x,y,z\}$ such that $\{a,b\} \cap \{a',b'\} = \{x\}$, then 
\begin{align*}
&\left|\Cov[\Mpp(a, (\alpha, \beta)) \Mpp(b, (\alpha, \beta)), \Mpp(a', (\alpha, \beta)) \Mpp(b', (\alpha, \beta))  ~|~ \cF, \{\theta_i\}_{i\in [n]}]\right| \\
&\qquad= \left|\Var[\Mpp(x, (\alpha, \beta))] \E[\Mpp(y, (\alpha, \beta))] \E[\Mpp(z, (\alpha, \beta))]\right| \\
&\qquad\leq \left|\E[\Mpp^2(x, (\alpha, \beta))] \E[\Mpp(y, (\alpha, \beta))] \E[\Mpp(z, (\alpha, \beta))]\right| \\
&\qquad\leq p^3.
\end{align*}
The inequalities follow from the property that every entry of $\Mpp$ has absolute value bounded by 1, and takes value 0 with probability $(1-p)$ in the event it is not observed.

We use this to expand the variance calculation, and use the properties that for every entry $a$, $|N_{u,\radius}(a)| \leq \Ind_{(a \in \cS_{u,\radius})}$. We have dropped the conditioning notation due to the length of the expressions.
\begin{align*}
&\Var[T(\alpha, \beta)  ~|~ \cF, \{\theta_i\}_{i\in [n]}] \\
&= \sum_{a \neq b \in [n]} \Big(N_{u,\radius}^2(a) N_{v,\radius}^2(b) + N_{u,\radius}(a) N_{v,\radius}(b) N_{u,\radius}(b) N_{v,\radius}(a)\Big)
 \Var[\Mpp(a, (\alpha, \beta)) \Mpp(b, (\alpha, \beta))] \\
&\quad + \sum_{a \neq b \in [n]} \sum_{c \notin \{a,b\}} \Big(N_{u,\radius}^2(a) N_{v,\radius}(b) N_{v,\radius}(c) + N_{v,\radius}^2(a) N_{u,\radius}(b) N_{u,\radius}(c)  + N_{u,\radius}(a) N_{u,\radius}(b) N_{v,\radius}(a) N_{v,\radius}(c)  \\
&\qquad \qquad \qquad \qquad \qquad 
+ N_{u,\radius}(a) N_{u,\radius}(c) N_{v,\radius}(a) N_{v,\radius}(b) \Big) \Var[\Mpp(a, (\alpha, \beta))] \E[\Mpp(b, (\alpha, \beta))] \E[\Mpp(c, (\alpha, \beta))] \\
&\leq p^2 \sum_{a \neq b \in [n]} (\Ind_{(a \in \cS_{u,\radius}, b \in \cS_{v,\radius})}  + \Ind_{(\{a,b\} \subset \cS_{u,\radius} \cap \cS_{v,\radius})}) \\
&\quad + p^3 \sum_{a \neq b \in [n]} \sum_{c \notin \{a,b\}} \Big(
\Ind_{(a \in \cS_{u,\radius}, \{b,c\} \subset \cS_{v,\radius})} 
+\Ind_{(a \in \cS_{v,\radius}, \{b,c\} \subset \cS_{u,\radius})}
+\Ind_{(a \in \cS_{u,\radius} \cap \cS_{v,\radius}, b \in \cS_{u,\radius}, c \in \cS_{v,\radius})} 
+\Ind_{(a \in \cS_{u,\radius} \cap \cS_{v,\radius}, c \in \cS_{u,\radius}, b \in \cS_{v,\radius})} 
\Big) \\
&\leq 2 p^2 |\cS_{u,\radius}| |\cS_{v,\radius}| + 2 p^3 |\cS_{u,\radius}| |\cS_{v,\radius}|^2 + 2 p^3 |\cS_{u,\radius}|^2 |\cS_{v,\radius}| \\
&= 2 p^2 |\cS_{u,\radius}| |\cS_{v,\radius}| (1+ o(1)).
\end{align*}
The first term dominates because $p |\cS_{u,\radius}| \leq pn = o(1)$ and $p |\cS_{v,\radius}| = o(1)$.
\endproof

\begin{lemma} \label{lem:tail_T}
\[\Prob{|T(\alpha,\beta)| \geq z  ~|~ \cF} \leq
\left(|\cS_{u,\radius} \cup \cS_{v,\radius}| p\right)^{\lceil\sqrt{z}\rceil} (1 + o(1).\]
As a result,
\[\Prob{\cup_{\alpha,\beta} \{|T(\alpha,\beta)| \geq z\}  ~|~ \cF} \leq \binom{|S|}{2} \left(|\cS_{u,\radius} \cup \cS_{v,\radius}| p\right)^{\lceil\sqrt{z}\rceil} (1 + o(1)),\]
and
\begin{align*}
&\E[\Ind_{|T(\alpha,\beta)| \geq \phi^2} (|T(\alpha,\beta)| - \phi^2) ~|~ \cF] \leq (1 + o(1)) \frac{2 \phi}{\ln(|\cS_{u,\radius} \cup \cS_{v,\radius}|^{-1} p^{-1})} (|\cS_{u,\radius} \cup \cS_{v,\radius}| p)^{\phi}
\end{align*}
\end{lemma}

\proof
Let us define
\begin{align*}
Z_u(\alpha,\beta) &= \{a \in [n] ~s.t.~ a \in \cS_{u,\radius}, (a, \alpha, \beta) \in \cEpp\}, \\
Z_v(\alpha,\beta) &= \{b \in [n] ~s.t.~ b \in \cS_{v,\radius}, (b, \alpha, \beta) \in \cEpp\}.
\end{align*}
Furthermore, because $|\Mpp(\cdot, \cdot, \cdot)| \leq 1$ and $\|N_{u, s}\|_\infty \leq 1$, for any $a,b \in [n]$, it follows that 
\begin{align*}
&|N_{u,\radius}(a) N_{v,\radius}(b)  \Mpp(a, (\alpha, \beta)) \Mpp(b, (\alpha, \beta))| \leq \Ind_{(a \in Z_u(\alpha,\beta))}\Ind_{(b \in Z_v(\alpha,\beta))},
\end{align*}
which implies 
\[|T(\alpha,\beta)| \leq |Z_u(\alpha,\beta)| |Z_v(\alpha,\beta)| \leq |Z_u(\alpha,\beta) \cup Z_v(\alpha,\beta)|^2.\]
Note that $|Z_u(\alpha,\beta) \cup Z_v(\alpha,\beta)| \sim \text{Binomial}(|\cS_{u,\radius} \cup \cS_{v,\radius}|, p)$. It follows then that 
\begin{align*}
\Prob{|T(\alpha,\beta)| \geq z ~|~ \cF} 
&\leq \Prob{|Z_u(\alpha,\beta) \cup Z_v(\alpha,\beta)|^2 \geq z  ~|~ \cF} \\
&= \Prob{|Z_u(\alpha,\beta) \cup Z_v(\alpha,\beta)| \geq \lceil\sqrt{z}\rceil  ~|~ \cF} \\
&= \sum_{i = \lceil\sqrt{z}\rceil}^{|\cS_{u,\radius} \cup \cS_{v,\radius}|} \binom{|\cS_{u,\radius} \cup \cS_{v,\radius}|}{i} p^i (1-p)^{|\cS_{u,\radius} \cap \cS_{v,\radius}|-i} \\
&\leq (1-p)^{|\cS_{u,\radius} \cup \cS_{v,\radius}|} \sum_{i = \lceil\sqrt{z}\rceil}^{|\cS_{u,\radius} \cup \cS_{v,\radius}|} \left(\frac{|\cS_{u,\radius} \cup \cS_{v,\radius}| p}{1-p}\right)^i \\
&\leq (1-p)^{|\cS_{u,\radius} \cup \cS_{v,\radius}|} \left(\frac{|\cS_{u,\radius} \cup \cS_{v,\radius}| p}{1-p}\right)^{\lceil\sqrt{z}\rceil} \sum_{i = 0}^{\infty} \left(\frac{|\cS_{u,\radius} \cup \cS_{v,\radius}| p}{1-p}\right)^i \\
&\leq (1-p)^{|\cS_{u,\radius} \cup \cS_{v,\radius}|} \left(\frac{|\cS_{u,\radius} \cup \cS_{v,\radius}| p}{1-p}\right)^{\lceil\sqrt{z}\rceil} (1 + o(1)) \\
&\leq \left(|\cS_{u,\radius} \cup \cS_{v,\radius}| p\right)^{\lceil\sqrt{z}\rceil} (1 + o(1)) \\
&\leq \left(|\cS_{u,\radius} \cup \cS_{v,\radius}| p\right)^{\sqrt{z}} (1 + o(1)),
\end{align*}
where we used the fact that $|\cS_{u,\radius} \cup \cS_{v,\radius}| p = o(1)$.

We use the bound on the tail probabilities to show that
\begin{align*}
\E[\Ind_{|T(\alpha,\beta)| \geq \phi^2} (|T(\alpha,\beta)| - \phi^2) ~|~ \cF] 
&= \int_0^{\infty} \Prob{|T(\alpha,\beta)| \geq \phi^2 + z} dz \\
&\leq (1 + o(1)) \int_0^{\infty} \left(|\cS_{u,\radius} \cup \cS_{v,\radius}| p\right)^{\sqrt{\phi^2 + z}} dz \\
&\leq (1 + o(1)) \int_{\phi}^{\infty} 2 y \left(|\cS_{u,\radius} \cup \cS_{v,\radius}| p\right)^{y} dy \\
&=(1 + o(1)) \frac{2 \phi}{\ln(|\cS_{u,\radius} \cup \cS_{v,\radius}|^{-1} p^{-1})} (|\cS_{u,\radius} \cup \cS_{v,\radius}| p)^{\phi}. 
\end{align*}

\endproof
\subsection{Proof of Lemma \ref{lemma:dist}}

\begin{proof}
Now we are ready to bound the difference between $d(u, v)$ and $\hat{d}(u,v)$ for any $u, v \in [n]$. Recall, 
\begin{align}\label{eq:lem1.f1}
d(\theta_u,\theta_v) & = \| \Lambda^{2t+1} Q (e_u - e_v)\|^2 \nonumber\\
&= (e_u - e_v)^T Q^T \Lambda^{4t+2} Q (e_u - e_v) \\
&= e_u^T Q^T \Lambda^{4t+2} Q e_u + e_v^T Q^T \Lambda^{4t+2} Q e_v -  e_u^T Q^T \Lambda^{4t+2} Q e_v -  e_v^T Q^T \Lambda^{4t+2} Q e_u, \nonumber
\end{align}
and according to \eqref{eq:dist},
\begin{align}\label{eq:lem1.f2}
\dist(u,v) &= \frac{1}{\binom{|S|}{2} p^2} (Z_{uu} + Z_{vv} - Z_{uv} - Z_{vu})
\end{align}
 for $S  \equiv S_{u,s,\radius} = n \setminus (\cB_{u,\radius} \cup \cB_{v,\radius} \cup [n/2])$ and
\begin{align}
Z_{uv} &= \frac{1}{\binom{|S_{u,s,\radius}|}{2} p^2 |\cS_{u,t}| |\cS_{v,t}|} \sum_{\alpha < \beta \in S_{u,s,\radius}\times S_{u,s,\radius}} T_{uv}(\alpha, \beta).
\end{align}
By Lemma \ref{lem:growth}, event $\cA^3_{u, \radius}(\delta)$ holds with probability at least 
$1 - O\Big( n \exp\big(-\Theta(n^{2\extra}) \big) \Big)$. By Lemmas \ref{lemma:nhbrhd_vectors} and 
\ref{lemma:NFN_bound}, conditioned on $\cA^3_{u, \radius}(\delta)$, for
\[x = \frac{ 16 \lambda_{\max}^{2\radius-2} n^{\psi}} {(1-\delta) n^{\extra}} = o(1),\]
event $\Ap_{u, v, \radius}(x, \delta)$ holds with probability at least 
$1 - 4 r \exp(-n^{2\psi}(1-o(1)))$,
implying 
\begin{align*}
& \big|\tN_{u,\radius}^T Q^T \Lambda^2 Q \tN_{v, \radius} - e_u^T Q^T \Lambda^{2(2\radius + 1)} Q e_v \big|  \leq \frac{n^{\psi}}{n^{\extra}}  \left(16 B \lambda_{\max}^{2(\radius - 1)} \Bigg(\sum_{k=1}^r 2\lambda_k^{2(\radius+1)}\Bigg) (1-\delta)^{-1} (1+ o(1)) \right).
\end{align*}
By Lemma \ref{lemma:NMN_conc}, conditioned on $\Ap_{u, v, \radius}(x, \delta)$, with probability
$1 - 4\exp(-n^{2\psi}(1-o(1)))$, 
\begin{align*}
|Z_{uv} - \tN_{u,\radius}^T Q^T \Lambda^2 Q \tN_{v,\radius}|  &= O\left(\frac{n^{\psi}}{(|S_{u,s,\radius}|^2 p^2 |\cS_{u,\radius}| |\cS_{v,\radius}|)^{1/2}}\right) + O\left(\frac{n^{\psi}}{|S_{u,s,\radius}|^{1/2}}\right),
\end{align*}
where $|S_{u,s,\radius}| = \Theta(n) = \Omega(n^{2\extra})$, by event $\Ap_{u,v,\radius}(x,\delta)$ and $\radius \geq \frac{1}{4\extra}$,
\begin{align*}
|S_{u,s,\radius}|^2 p^2 |\cS_{u,\radius}| |\cS_{v,\radius}| 
= \Theta(n^2 n^{-3 + 2 \extra} n^{4\extra\radius})
= \Omega(n^{2 \extra}).
\end{align*}
To put it all together, for $\psi \in(0, \extra)$, with probability at least 
\[1 - O\Big( \exp(-n^{2\psi}(1-o(1)) ) \Big),\]
it holds that
\begin{align*}
|\dist(u,v) - d(\theta_u,\theta_v)| 
&= O\left(\frac{r \lambda_{\max}^{4t} n^{\psi}}{n^{\extra}}\right).
\end{align*}
This completes the proof of Lemma \ref{lemma:dist}. 
\end{proof}

\section{Proof of Lemma \ref{lemma:dist.pert}: perturbation analysis of distance}\label{sec:perturbation}

We establish the proof of Lemma \ref{lemma:dist.pert} here. To do so, we establish a perturbation property of 
$\dist$ here, which combined with Lemma \ref{lemma:dist} will result into the proof of Lemma \ref{lemma:dist.pert}. 

We study the perturbation in the $\dist$ estimate when each noisy observed entry is arbitrarily perturbed. Specifically, for any
$(u, v, w) \in [n]^3$, $M_1(u, v, w)$ is observed with probability $p$. If observed, according to \eqref{eq:meas.def}, 
$M_1(u, v, w) = F(u, v, w) + \epsilon_{uvw} = F_r(u,v,w) + \bpert_{uvw} +  \epsilon_{uvw} $, where $F_r$ is the best rank $r$ approximation to $F$. This expression shows that we can interpret the deviation from a rank $r$ model as a deterministic perturbation of $\bpert_{uvw}$, bounded in absolute value by $\bpert$. Note that $\bpert_{uvw}$ can be any arbitrary (or adversarial), unknown deterministic quantity satisfying $|\bpert_{uvw}|\leq \bpert$. 

Lemma \ref{lemma:dist_perturb} provides a bound on the perturbation in the distance estimate, $\dist$, that results from these entrywise perturbabtions of the observations. 
\begin{lemma} \label{lemma:dist_perturb}
Let $p = n^{-3/2+\extra}$ for $\extra \in (0,\frac12)$, 
$\radius$ as defined in \eqref{eq:t}, 
$\delta \in (0,\frac12)$, and
\[x = \frac{ 16 \lambda_{\max}^{2\radius-2} n^{\psi}} {(1-\delta) n^{\extra}}.\]
For any $u \in [n]$, recall the event
\begin{align}\label{eq:aprime.copy}
\Ap_{u, v, \radius}(x, \delta) & = \cap_{k=1}^r \big(\cA^4_{u,k,\radius}(x, \delta) \cap \cA^4_{v,k,\radius}(x, \delta)\big). 
\end{align}
Let 
event  $\Ap_{u, v, \radius}(x, \delta)$ hold. Let each observed entry of $M_1$ be perturbed by adding
arbitrary, deterministic quantity bounded by $\bpert \geq 0$. Then for any $u, v \in [n]^2$, the 
distance estimate $\dist(u,v)$ is perturbed by at most  
$O( \radius \bpert (1+\bpert)^{2\radius-1} +  \radius^2 \bpert^2 (1+\bpert)^{4\radius-2})$
with probability at least $1 - \exp\big(-\Omega(n^{2\extra})\big) - O(n^{-8})$.
\end{lemma}
\proof 
Recall definition of $\dist$ in \eqref{eq:dist}:
\begin{align*}
\dist(u,v) &= (Z_{uu} + Z_{vv} - Z_{uv} - Z_{vu}), \\ 
Z_{uv} &= \frac{1}{|\cV_B(u,v,t)| p^2 |\cS_{u,\radius}| |\cS_{v,\radius}|} \sum_{(\alpha, \beta) \in \cV_B(u,v,\radius)} T_{uv}(\alpha, \beta), \nonumber \\
\cV_B(u,v,\radius) &= \{(\alpha,\beta) \in \cV_B ~s.t.~ \alpha \notin \cB_{u,2\radius} \cup \cB_{v,2\radius}, \beta \notin \cB_{u,2\radius} \cup \cB_{v,2\radius}\}, \nonumber \\
T_{uv}(\alpha, \beta) &= \sum_{a \neq b \in [n]} N_{u,\radius}(a) N_{v,\radius}(b)  \Mpp(a, (\alpha, \beta)) \Mpp(b, (\alpha, \beta))
\end{align*}
We shall bound the perturbation on $Z_{uv}$. Similar bounds will follow for the other three terms which will conclude the main results.  
Our interest is in understanding how does $Z_{uv}$ change if each observed entry is changed by arbitrary quantity
bounded by $\bpert \geq 0$. This will induce a bound on the changes in $T_{uv}(\cdot, \cdot)$ which will help bound the change in $Z_{uv}$.
By assumption \eqref{eq:aprime.copy}, $\Ap_{u, v, \radius}(x, \delta)$ holds. 
Conditioned on event  $\Ap_{u, v,\radius}(x, \delta)$, all the information associated with $\cB_{u, 2\radius}$ and $\cB_{v, 2\radius}$ is revealed; 
however, information about $[n] \backslash (\cB_{u, 2\radius} \cup \cB_{v, 2\radius})$ is not. 
Let $\cF(u, v, \radius, x, \delta)$ denote all the information revealed such that event $\Ap_{u, v, \radius}(x, \delta)$ holds.

Under $\Ap_{u, v, \radius}(x, \delta)$, by definition $\cA^3_{u,\radius}(\delta)$ and $\cA^3_{v,\radius}(\delta)$ holds. 
This implies that for $S \equiv \cV_B(u,v,t)$,  $|S| = \frac{n}{2} - o(n) = \frac{n(1-o(1))}{2}$. 
Furthermore, it implies that $|\cS_{u,\radius}|$ and $|\cS_{v,\radius}|$ are both greater than or equal to 
$(1-\delta)^{2\radius} 2^{-3\radius-1} n^{2\extra \radius} (1-o(1))$. As shown in \eqref{eq:SpS_bd},
\begin{align}\label{eq:lb.xx}
|S|^2 p^2 |\cS_{u,\radius}| |\cS_{v,\radius}| 
&= \Omega(n^{2\extra})
\end{align}
results from the choice of $\radius \geq \frac{1}{4\extra}$, and the fact that 
$\delta$ and $\radius$ are both constants.

For given $\alpha \neq \beta \in \cV_B(u,v,\radius)$, $T_{uv}(\alpha, \beta)$ is summation over terms, indexed by $a \neq b \in [n]$, 
containing product $N_{u,t}(a) N_{v,\radius}(b)  \Mpp(a, (\alpha, \beta)) \Mpp(b, (\alpha, \beta))$. Now $N_{u,\radius}(a) = 0$ if 
$a \notin \cS_{u,\radius}$, $N_{v,\radius}(b) = 0$ if $b \notin \cS_{v,\radius}$. For $a \in  \cS_{u,\radius}$,  $N_{u,\radius}(a)$ 
is product of $2\radius$ terms, each bounded in absolute value by $1$: let $N_{u,\radius}(a) = \prod_{i=1}^{2\radius} w_i$ with
$|w_i|\leq 1$ for all $i \leq 2 \radius$. Let $\pert_i$ be arbitrary, deterministic quantity added to $w_i$ with $|\pert_i|\leq \bpert$ 
for $i \leq 2 \radius$. Then change in $N_{u,\radius}(a)$ is bounded as 
\begin{align}
\big|\prod_{i=1}^{2\radius} w_i - \prod_{i=1}^{2\radius} (w_i + \pert_i) \big| 
&\quad = \big|\sum_{S \subset [2\radius]: S \neq \emptyset} \prod_{i \in S} \pert_i \prod_{s \in [2\radius] \backslash S} w_i\big| \nonumber \\
&\quad \leq \sum_{S \subset [2\radius]: S \neq \emptyset} \prod_{i \in S} |\pert_i| \prod_{s \in [2\radius] \backslash S} |w_i| \nonumber \\
&\quad  \leq \sum_{S \subset [2\radius]: S \neq \emptyset} \bpert^{|S|} ~= \sum_{i=1}^{2\radius} {2\radius \choose i} \bpert^s \nonumber \\
&\quad  =~\bpert\Big( \sum_{i=0}^{2\radius-1} \frac{(2\radius)!}{(2\radius-i-1)! (i+1)!} \bpert^i \Big) \nonumber \\
&\quad  \leq 2\radius \bpert \Big(\sum_{i=0}^{2\radius-1}  \frac{(2\radius-1)!}{((2\radius-1)-i)! i!} \bpert^i\Big)  \nonumber \\
&\quad =~2\radius \bpert   \Big(\sum_{i=0}^{2\radius-1}  {2\radius-1 \choose i} \bpert^i\Big) \nonumber \\
&\quad  = 2\radius \bpert (1+\bpert)^{2\radius-1} \equiv \Delta(\radius, \bpert). 
\end{align}
That is, $N_{u,\radius}(a)$ changes by at most $\Delta(\radius, \bpert)$. Similarly $N_{v, \radius}(b)$ changes by at most 
$\Delta(\radius, \bpert)$. Therefore, $N_{u,\radius}(a) N_{v, \radius}(b)$ can change at most 
by $O(\Delta(\radius, \bpert) + \Delta(\radius, \bpert)^2)$. 

By definition $|\Mpp(a, (\alpha, \beta))|, |\Mpp(b, (\alpha, \beta))|  \leq 1$. Further, 
$\Mpp(a, (\alpha, \beta)) \Mpp(b, (\alpha, \beta)) \neq 0$ only if 
$\Ind_{((a, \alpha, \beta) \in \cEpp)} \Ind_{((b, \alpha, \beta) \in \cEpp)}  = 1$. 
Therefore,  we can bound change in the term $N_{u,\radius}(a) N_{v,\radius}(b)  \Mpp(a, (\alpha, \beta)) \Mpp(b, (\alpha, \beta))$
as $\Ind_{((a, \alpha, \beta) \in \cEpp)} \Ind_{((b, \alpha, \beta) \in \cEpp)} O(\Delta(\radius, \bpert) + \Delta(\radius, \bpert)^2)$.
Therefore, we can bound the change in $Z_{uv}$ by 
\begin{align}
& \frac{O(\Delta(\radius, \bpert) + \Delta(\radius, \bpert)^2) }{|S|^2 p^2 |\cS_{u,\radius}||\cS_{v,\radius}| } \Big(\sum_{a \in \cS_{u,\radius}, b \in \cS_{v, \radius}, \alpha, \beta \in S} 
\Ind_{((a, \alpha, \beta) \in \cEpp)} \Ind_{((b, \alpha, \beta) \in \cEpp)} \Ind_{a \neq b} \Big) \\
&= \frac{O(\Delta(\radius, \bpert) + \Delta(\radius, \bpert)^2) }{|S|^2 p^2 |\cS_{u,\radius}||\cS_{v,\radius}| } \sum_{\alpha, \beta \in S} X_{\alpha \beta} \label{eq:bound.zuv.1}
\end{align}
where 
$$X_{\alpha \beta} = \sum_{ a \in \cS_{u,\radius}, b \in \cS_{v, \radius}, a \neq b} \Ind_{((a, \alpha, \beta) \in \cEpp)} \Ind_{((b, \alpha, \beta) \in \cEpp)}.$$

To conclude the Lemma,  it will be sufficient to argue that
$ \sum_{\alpha, \beta \in S} X_{\alpha \beta} = O(|S|^2 p^2 |\cS_{u,\radius}||\cS_{v,\radius}|)$ with high probability given $\cF$.  We use a similar argument as the proof of Lemma \ref{lemma:NMN_conc}. Given 
$\cF \equiv \cF(u, v, \radius, x, \delta)$, $\{X_{\alpha \beta}\}_{\alpha,\beta \in S^2}$ are conditionally independent random variables. By the same argument as that in Lemma \ref{lem:tail_T}, it follows that
\begin{align}
\Prob{\cup_{\alpha,\beta} \{X_{\alpha,\beta} \geq \phi^2\}  ~|~ \cF} &\leq \binom{|S|}{2} \left(|\cS_{u,\radius} \cup \cS_{v,\radius}| p\right)^{\phi} (1 + o(1)) \label{eq:X_tailBd}\\
\E[\Ind_{X_{\alpha,\beta} \geq \phi^2} (X_{\alpha,\beta} - \phi^2) ~|~ \cF] 
&\leq (1 + o(1)) \frac{2 \phi}{\ln(|\cS_{u,\radius} \cup \cS_{v,\radius}|^{-1} p^{-1})} (|\cS_{u,\radius} \cup \cS_{v,\radius}| p)^{\phi}. \label{eq:X_Ediff}
\end{align}
We define 
$\tX_{\alpha \beta} = \min(X_{\alpha \beta}, \phi^2)$ for $\phi = \lceil 16/(1-2\extra)\rceil$ so that $|\tX_{\alpha \beta}  - \E[\tX_{\alpha \beta} ]| \leq \phi^2$.
By \eqref{eq:X_tailBd}, and the choice of $\phi$ along with the conditions from $\cF$ that guarantee $|S| = \Theta(n)$ and naively $|\cS_{u,\radius} \cup \cS_{v,\radius}| = O(n)$,
\begin{align}
\Prob{\sum_{\alpha,\beta} X_{\alpha,\beta} \neq \sum_{\alpha,\beta} \tX_{\alpha,\beta}  ~|~ \cF} 
&\leq \binom{|S|}{2} \left(|\cS_{u,\radius} \cup \cS_{v,\radius}| p\right)^{\phi} (1 + o(1)) = O\big( n^{-6}\big).\label{eq:bound.zuv.2}
\end{align}
By \eqref{eq:X_Ediff},
\begin{align} 
|\E[X_{\alpha,\beta} ~|~ \cF] - \E[\tX_{\alpha,\beta} ~|~ \cF]| 
	&\leq \E[\Ind_{X_{\alpha,\beta} \geq \phi^2} (X_{\alpha,\beta} - \phi^2) ~|~ \cF] \nonumber\\
	&\leq (1 + o(1)) \frac{2 \phi}{\ln(|\cS_{u,\radius} \cup \cS_{v,\radius}|^{-1} p^{-1})} (|\cS_{u,\radius} \cup \cS_{v,\radius}| p)^{\phi} \nonumber \\
	&= O\big( n^{-6}\big). \label{eq:bound.zuv.3}
\end{align}
By the same argument as that in Lemma \ref{lem:VarT}, it follows that 
\begin{align*}
\Var[\tX_{\alpha, \beta}  ~|~ \cF, \{\theta_i\}_{i\in [n]}]
	&\leq \Var[X_{\alpha, \beta}  ~|~ \cF, \{\theta_i\}_{i\in [n]}] \\
	&\leq 2 p^2 |\cS_{u,\radius}| |\cS_{v,\radius}| (1+ o(1)).
\end{align*}
By Bernstein's inequality, for $z = 2 n^{\psi} |S| p (|\cS_{u,\radius}| |\cS_{v,\radius}|)^{1/2}$ for $\psi \in (0, \extra)$,
\begin{align} 
\Prob{\sum_{\alpha, \beta \in S} \big(\tX_{\alpha, \beta} - \E[\tX_{\alpha, \beta}] \big) > z } 
 & \leq \exp\Big( -\frac{3 z^2}{2 z \phi^2  + 12 (1 + o(1)) |S|^2 p^2 |\cS_{u,\radius}| |\cS_{v,\radius}|} \Big) \nonumber\\
&= \exp( -n^{2\psi}(1-o(1)) ). \label{eq:bound.zuv.4}
\end{align}

By \eqref{eq:bound.zuv.1}, \eqref{eq:bound.zuv.2}, \eqref{eq:bound.zuv.3}, and \eqref{eq:bound.zuv.4},
given $\Ap_{u, v, \radius}(x, \delta)$ holds,
with probability $1 -\exp\big(-n^{2\extra}(1-o(1))\big) - O(n^{-6})$, 
the change in $Z_{uv}$ is bounded above by  
\begin{align*}
\frac{O(\Delta(\radius, \bpert) + \Delta(\radius, \bpert)^2) }{|S|^2 p^2 |\cS_{u,\radius}||\cS_{v,\radius}| } \sum_{\alpha, \beta \in S} X_{\alpha \beta} 
&=\frac{O(\Delta(\radius, \bpert) + \Delta(\radius, \bpert)^2) }{|S|^2 p^2 |\cS_{u,\radius}||\cS_{v,\radius}| } \sum_{\alpha, \beta \in S} \tX_{\alpha \beta} \\
&=\frac{O(\Delta(\radius, \bpert) + \Delta(\radius, \bpert)^2) }{|S|^2 p^2 |\cS_{u,\radius}||\cS_{v,\radius}| } (\E[\sum_{\alpha, \beta \in S} X_{\alpha \beta} ~|~ \cF] + 2 n^{\psi} |S| p (|\cS_{u,\radius}| |\cS_{v,\radius}|)^{1/2}).
\end{align*}
By \eqref{eq:lb.xx}, this choice of $2 n^{\psi} |S| p (|\cS_{u,\radius}| |\cS_{v,\radius}|)^{1/2}) = o(|S|^2 p^2 |\cS_{u,\radius}||\cS_{v,\radius}|)$ for $\psi \in (0, \extra)$. Finally we use the bound that $\E[\sum_{\alpha, \beta \in S} X_{\alpha \beta} ~|~ \cF] = \binom{|S|}{2} p^2 |\cS_{u,\radius}||\cS_{v,\radius}|$ to argue that with high probability the change in $Z_{uv}$ is bounded above by 
\[O( \Delta(\radius, \bpert) + \Delta(\radius, \bpert)^2 ) = O( \radius \bpert (1+\bpert)^{2\radius-1} +  \radius^2 \bpert^2 (1+\bpert)^{4\radius-2}).\]
This completes the proof of Lemma \ref{lemma:dist_perturb}.

\endproof

\subsection{Completing proof of Lemma \ref{lemma:dist.pert}}

Under the setup of Lemma \ref{lemma:dist.pert}, as argued in the proof of Lemma \ref{lemma:dist}, 
$\Ap_{u, v, \radius}(x, \delta)$, with appropriate choice of $x, \delta$ as considered in statement of 
Lemma \ref{lemma:dist_perturb}, holds with probability at least  $1 - 4 r \exp(-n^{2\psi}(1-o(1)))$. And
$\dist$ (without perturbation), is within $O\left(n^{-(\extra - \psi)}\right)$ for any pair of $u, v \in [n]$. 
By Lemma \ref{lemma:dist_perturb}, under event $\Ap_{u, v, \radius}(x, \delta)$, the $\dist$
is further perturbed by $O\left( \radius \bpert (1+\bpert)^{2\radius-1} +  \radius^2 \bpert^2 (1+\bpert)^{4\radius-2} \right)$
with probability at least $1-O\Big( \exp(-n^{2\psi}(1-o(1)))\Big) - O\Big(n^{-6}\Big)$. Putting these
together, we conclude the claim of Lemma \ref{lemma:dist.pert}.

    \section*{Acknowledgements}
    We gratefully acknowledge funding from the NSF under grants CCF-1948256 and CNS-1955997. Christina Lee Yu is also supported by an Intel Rising Stars Award.

\bibliographystyle{IEEEtran}

\begin{thebibliography}{10}
	\providecommand{\url}[1]{#1}
	\csname url@samestyle\endcsname
	\providecommand{\newblock}{\relax}
	\providecommand{\bibinfo}[2]{#2}
	\providecommand{\BIBentrySTDinterwordspacing}{\spaceskip=0pt\relax}
	\providecommand{\BIBentryALTinterwordstretchfactor}{4}
	\providecommand{\BIBentryALTinterwordspacing}{\spaceskip=\fontdimen2\font plus
		\BIBentryALTinterwordstretchfactor\fontdimen3\font minus
		\fontdimen4\font\relax}
	\providecommand{\BIBforeignlanguage}[2]{{%
			\expandafter\ifx\csname l@#1\endcsname\relax
			\typeout{** WARNING: IEEEtran.bst: No hyphenation pattern has been}%
			\typeout{** loaded for the language `#1'. Using the pattern for}%
			\typeout{** the default language instead.}%
			\else
			\language=\csname l@#1\endcsname
			\fi
			#2}}
	\providecommand{\BIBdecl}{\relax}
	\BIBdecl
	
	\bibitem{KeshavanMontanariOh10a}
	R.~H. Keshavan, A.~Montanari, and S.~Oh, ``Matrix completion from a few
	entries,'' \emph{IEEE Transactions on Information Theory}, vol.~56, no.~6,
	pp. 2980--2998, 2010.
	
	\bibitem{KeshavanMontanariOh10b}
	------, ``Matrix completion from noisy entries,'' \emph{Journal of Machine
		Learning Research}, vol.~11, no. Jul, pp. 2057--2078, 2010.
	
	\bibitem{Chatterjee15}
	S.~Chatterjee, ``Matrix estimation by universal singular value thresholding,''
	\emph{The Annals of Statistics}, vol.~43, no.~1, pp. 177--214, 2015.
	
	\bibitem{CandesRecht09}
	E.~Candes and B.~Recht, ``Exact matrix completion via convex optimization,''
	\emph{Communications of the ACM}, vol.~55, no.~6, 2009.
	
	\bibitem{CandesPlan10}
	E.~J. Candes and Y.~Plan, ``Matrix completion with noise,'' \emph{Proceedings
		of the IEEE}, vol.~98, no.~6, pp. 925--936, 2010.
	
	\bibitem{CandesTao10}
	E.~J. Cand{\`e}s and T.~Tao, ``The power of convex relaxation: Near-optimal
	matrix completion,'' \emph{IEEE Transactions on Information Theory}, vol.~56,
	no.~5, pp. 2053--2080, 2010.
	
	\bibitem{Recht11}
	B.~Recht, ``A simpler approach to matrix completion,'' \emph{Journal of Machine
		Learning Research}, vol.~12, no. Dec, pp. 3413--3430, 2011.
	
	\bibitem{NegahbanWainwright11}
	S.~Negahban and M.~J. Wainwright, ``Estimation of (near) low-rank matrices with
	noise and high-dimensional scaling,'' \emph{The Annals of Statistics}, pp.
	1069--1097, 2011.
	
	\bibitem{mazumder2010spectral}
	R.~Mazumder, T.~Hastie, and R.~Tibshirani, ``Spectral regularization algorithms
	for learning large incomplete matrices,'' \emph{Journal of machine learning
		research}, vol.~11, no. Aug, pp. 2287--2322, 2010.
	
	\bibitem{ChenWainwright15}
	Y.~Chen and M.~J. Wainwright, ``Fast low-rank estimation by projected gradient
	descent: General statistical and algorithmic guarantees,'' \emph{arXiv
		preprint arXiv:1509.03025}, 2015.
	
	\bibitem{sun2016guaranteed}
	R.~Sun and Z.-Q. Luo, ``Guaranteed matrix completion via non-convex
	factorization,'' \emph{IEEE Transactions on Information Theory}, vol.~62,
	no.~11, pp. 6535--6579, 2016.
	
	\bibitem{ge2016matrix}
	R.~Ge, J.~D. Lee, and T.~Ma, ``Matrix completion has no spurious local
	minimum,'' in \emph{Advances in Neural Information Processing Systems}, 2016,
	pp. 2973--2981.
	
	\bibitem{JainNetrapalliSanghavi13}
	P.~Jain, P.~Netrapalli, and S.~Sanghavi, ``Low-rank matrix completion using
	alternating minimization,'' in \emph{Proceedings of the forty-fifth annual
		ACM symposium on Theory of computing}.\hskip 1em plus 0.5em minus 0.4em\relax
	ACM, 2013, pp. 665--674.
	
	\bibitem{hardt2014understanding}
	M.~Hardt, ``Understanding alternating minimization for matrix completion,'' in
	\emph{Foundations of Computer Science (FOCS), 2014 IEEE 55th Annual Symposium
		on}.\hskip 1em plus 0.5em minus 0.4em\relax IEEE, 2014, pp. 651--660.
	
	\bibitem{goldberg92}
	D.~Goldberg, D.~Nichols, B.~M. Oki, and D.~Terry, ``Using collaborative
	filtering to weave an information tapestry,'' \emph{Commun. ACM}, 1992.
	
	\bibitem{song2016blind}
	D.~Song, C.~E. Lee, Y.~Li, and D.~Shah, ``Blind regression: Nonparametric
	regression for latent variable models via collaborative filtering,'' in
	\emph{Advances in Neural Information Processing Systems}, 2016, pp.
	2155--2163.
	
	\bibitem{li2020blind}
	Y.~{Li}, D.~{Shah}, D.~{Song}, and C.~L. {Yu}, ``Nearest neighbors for matrix
	estimation interpreted as blind regression for latent variable model,''
	\emph{IEEE Transactions on Information Theory}, vol.~66, no.~3, pp.
	1760--1784, March 2020.
	
	\bibitem{BorgsChayesLeeShah17}
	C.~Borgs, J.~Chayes, C.~E. Lee, and D.~Shah, ``Thy friend is my friend:
	Iterative collaborative filtering for sparse matrix estimation,'' in
	\emph{Advances in Neural Information Processing Systems}, 2017, pp.
	4715--4726.
	
	\bibitem{borgs2017iterative}
	C.~Borgs, J.~Chayes, D.~Shah, and C.~L. Yu, ``Iterative collaborative filtering
	for sparse matrix estimation,'' \emph{Operations Research}, 2021.
	
	\bibitem{liu2013tensor}
	J.~Liu, P.~Musialski, P.~Wonka, and J.~Ye, ``Tensor completion for estimating
	missing values in visual data,'' \emph{IEEE transactions on pattern analysis
		and machine intelligence}, vol.~35, no.~1, pp. 208--220, 2013.
	
	\bibitem{gandy2011tensor}
	S.~Gandy, B.~Recht, and I.~Yamada, ``Tensor completion and low-n-rank tensor
	recovery via convex optimization,'' \emph{Inverse Problems}, vol.~27, no.~2,
	p. 025010, 2011.
	
	\bibitem{tomioka2010estimation}
	R.~Tomioka, K.~Hayashi, and H.~Kashima, ``Estimation of low-rank tensors via
	convex optimization,'' \emph{arXiv preprint arXiv:1010.0789}, 2010.
	
	\bibitem{tomioka2011statistical}
	R.~Tomioka, T.~Suzuki, K.~Hayashi, and H.~Kashima, ``Statistical performance of
	convex tensor decomposition,'' in \emph{Advances in Neural Information
		Processing Systems}, 2011, pp. 972--980.
	
	\bibitem{jain2014provable}
	P.~Jain and S.~Oh, ``Provable tensor factorization with missing data,'' in
	\emph{Advances in Neural Information Processing Systems}, 2014.
	
	\bibitem{bhojanapalli2015new}
	S.~Bhojanapalli and S.~Sanghavi, ``A new sampling technique for tensors,''
	\emph{arXiv preprint arXiv:1502.05023}, 2015.
	
	\bibitem{yuan2016tensor}
	M.~Yuan and C.-H. Zhang, ``On tensor completion via nuclear norm
	minimization,'' \emph{Foundations of Computational Mathematics}, vol.~16,
	no.~4, pp. 1031--1068, 2016.
	
	\bibitem{BarakMoitra16}
	B.~Barak and A.~Moitra, ``Noisy tensor completion via the sum-of-squares
	hierarchy,'' in \emph{Conference on Learning Theory}, 2016.
	
	\bibitem{PotechinSteurer17}
	A.~Potechin and D.~Steurer, ``Exact tensor completion with sum-of-squares,''
	pp. 1619--1673, 2017.
	
	\bibitem{MontanariSun18}
	A.~Montanari and N.~Sun, ``Spectral algorithms for tensor completion,''
	\emph{Communications on Pure and Applied Mathematics}, vol.~71, no.~11, pp.
	2381--2425, 2018.
	
	\bibitem{xia2017polynomial}
	D.~Xia and M.~Yuan, ``On polynomial time methods for exact low-rank tensor
	completion,'' \emph{Foundations of Computational Mathematics}, pp. 1--49,
	2017.
	
	\bibitem{xia2017statistically}
	D.~Xia, M.~Yuan, and C.-H. Zhang, ``Statistically optimal and computationally
	efficient low rank tensor completion from noisy entries,'' \emph{The Annals
		of Statistics}, vol.~49, no.~1, 2021.
	
	\bibitem{yuan2017incoherent}
	M.~Yuan and C.-H. Zhang, ``Incoherent tensor norms and their applications in
	higher order tensor completion,'' \emph{IEEE Transactions on Information
		Theory}, vol.~63, no.~10, pp. 6753--6766, 2017.
	
	\bibitem{friedland2014nuclear}
	S.~Friedland and L.-H. Lim, ``Nuclear norm of higher-order tensors,''
	\emph{Mathematics of Computation}, vol.~87, no. 311, pp. 1255--1281, 2018.
	
	\bibitem{abbe2020entrywise}
	E.~Abbe, J.~Fan, K.~Wang, and Y.~Zhong, ``Entrywise eigenvector analysis of
	random matrices with low expected rank,'' \emph{Annals of Statistics},
	vol.~48, no.~3, pp. 1452--1474, 2020.
	
	\bibitem{chen2019spectral}
	Y.~Chen, J.~Fan, C.~Ma, and K.~Wang, ``Spectral method and regularized mle are
	both optimal for top-k ranking,'' \emph{Annals of statistics}, vol.~47,
	no.~4, p. 2204, 2019.
	
	\bibitem{zhong2018near}
	Y.~Zhong and N.~Boumal, ``Near-optimal bounds for phase synchronization,''
	\emph{SIAM Journal on Optimization}, vol.~28, no.~2, pp. 989--1016, 2018.
	
	\bibitem{cai2019subspace}
	C.~Cai, G.~Li, Y.~Chi, H.~V. Poor, and Y.~Chen, ``Subspace estimation from
	unbalanced and incomplete data matrices: l2,$\infty$ statistical
	guarantees,'' \emph{The Annals of Statistics}, vol.~49, no.~2, pp. 944--967,
	2021.
	
	\bibitem{ma2018implicit}
	C.~Ma, K.~Wang, Y.~Chi, and Y.~Chen, ``Implicit regularization in nonconvex
	statistical estimation: Gradient descent converges linearly for phase
	retrieval and matrix completion,'' in \emph{International Conference on
		Machine Learning}, 2018, pp. 3345--3354.
	
	\bibitem{DingChen20}
	L.~Ding and Y.~Chen, ``Leave-one-out approach for matrix completion: Primal and
	dual analysis,'' \emph{IEEE Transactions on Information Theory}, 2020.
	
	\bibitem{cai2019nonconvex}
	C.~Cai, G.~Li, H.~V. Poor, and Y.~Chen, ``Nonconvex low-rank tensor completion
	from noisy data,'' \emph{Advances in neural information processing systems},
	vol.~32, 2019.
	
	\bibitem{cai2020uncertainty}
	C.~Cai, H.~V. Poor, and Y.~Chen, ``Uncertainty quantification for nonconvex
	tensor completion: Confidence intervals, heteroscedasticity and optimality,''
	in \emph{International Conference on Machine Learning}.\hskip 1em plus 0.5em
	minus 0.4em\relax PMLR, 2020, pp. 1271--1282.
	
	\bibitem{shah2019iterative}
	D.~Shah and C.~L. Yu, ``Iterative collaborative filtering for sparse noisy
	tensor estimation,'' in \emph{2019 IEEE International Symposium on
		Information Theory (ISIT)}.\hskip 1em plus 0.5em minus 0.4em\relax IEEE,
	2019, pp. 41--45.
	
	\bibitem{de2000multilinear}
	L.~De~Lathauwer, B.~De~Moor, and J.~Vandewalle, ``A multilinear singular value
	decomposition,'' \emph{SIAM journal on Matrix Analysis and Applications},
	vol.~21, no.~4, pp. 1253--1278, 2000.
	
	\bibitem{balasubramanian2021nonparametric}
	K.~Balasubramanian, D.~Gitelman, and H.~Liu, ``Nonparametric modeling of
	higher-order interactions via hypergraphons.'' \emph{J. Mach. Learn. Res.},
	vol.~22, pp. 146--1, 2021.
	
	\bibitem{alon2016probabilistic}
	N.~Alon and J.~H. Spencer, \emph{The probabilistic method}.\hskip 1em plus
	0.5em minus 0.4em\relax John Wiley \& Sons, 2016.
	
	\bibitem{wainwrightbook}
	M.~J. Wainwright, \emph{High-dimensional statistics: A non-asymptotic
		viewpoint}.\hskip 1em plus 0.5em minus 0.4em\relax Cambridge University
	Press, 2019, vol.~48.
	
	\bibitem{hoeffding1963}
	W.~Hoeffding, ``Probability inequalities for sums of bounded random
	variables,'' \emph{Journal of American Statistical Association}, vol.~58, pp.
	13--30, 1963.
	
\end{thebibliography}

    
    \appendix[Useful Lemmas and Omitted Proofs]
    \section{Useful Lemmas and Omitted Proofs}

We present the Proof of Lemma \ref{lemma:nearest_neighbor} below.

\proof{[Lemma \ref{lemma:nearest_neighbor}]}
We assumed the algorithm has access to two fresh samples, where $M_1$ is used to compute $\hat{d}$, and $\Mppp$ is used to compute the final estimate $\hat{F}$. Alternatively one could effectively obtain two sample sets by sample splitting. For some $(a,b, c) \in \cEppp$, the observation $\Mppp(a, b, c)$ is independent of $\hat{d}$, and $\E[\Mppp(a, b, c)] = f(\theta_a, \theta_b, \theta_c)$.  Conditioned on $\cEppp$, by definition of $\hat{F}$ and by assuming properties \ref{ass:good_distances.1} and \ref{ass:good_distances.2}, it follows that
\begin{align*}
	\E[(\hat{F}(u, v, w) - f(\theta_u,\theta_v, \theta_w))^2] 
	&= \left(\frac{1}{|\cEppp_{uvw}|}\sum_{(a,b, c) \in \cEppp_{uvw}} f(\theta_a,\theta_b, \theta_c) - f(\theta_u,\theta_v, \theta_w) \right)^2 \\
	&\qquad + \frac{1}{|\cEppp_{uvw}|^2} \sum_{(a,b,c ) \in \cEppp_{uvw}} \Var[\Mppp(a, b, c)] \\
	&\stackrel{(a)}{\leq} \text{bias}^2(\eta + \Delta) + \frac{\sigma^2}{|\cEppp_{uvw}|}.
\end{align*}
Inequality $(a)$ follows from property \ref{ass:good_distances.1}
and property \ref{ass:good_distances.2} for all $3n$ tuples 
$\{(u, a): a \in [n]\}\cup \{(v, b): b \in [n]\} \cup \{(w, c): c \in [n]\}$: $|d(u,a) - \hat{d}(u, a)| \leq \Delta$ and 
$\hat{d}(u, a) \leq \eta \implies d(u, a) \leq \eta + \Delta$, similarly $d(v, b), d(w, c) \leq \eta + \Delta$. As per \eqref{eq:meas.def}, 
we have that $\Var[\Mppp(a, b, c)] \leq \sigma^2$ for all $(a, b, c) \in \cEppp$. 
Define ${\cV}_{uvw} = \{ (a,b,c) \in [n]^3 ~:  {d}(u,a) < \eta -\Delta, ~{d}(v, b) < \eta-\Delta, ~{d}(w, c) < \eta-\Delta\}$.
Assuming property \ref{ass:good_distances.3},
\begin{align*}
	|\cV_{uvw}| & = |\{a \in [n]: {d}(u,a) < \eta-\Delta\}| ~|\{b \in [n]: {d}(v, b) < \eta-\Delta\}| \\
	& \qquad \qquad ~ \times|\{c \in [n]: {d}(w, c) < \eta-\Delta\}| \\
	&
	\geq \left(\meas(\eta-\Delta) n\right)^3.
\end{align*}
By the Bernoulli sampling model, each tuple $(a, b, c) \in [n]^3$ belongs to $\cEppp$ with probability $p$ independently. By a straightforward application of Chernoff's bound, 
it follows that for any $\delta \in (0,1)$, 
\begin{align}
	\Prob{|\cEppp \cap \cV_{uvc}| \leq (1 - \delta) \left(\meas(\eta-\Delta) n\right)^3} 
	&\leq \exp\left(- \frac{\delta^2 p \left(\meas(\eta-\Delta) n\right)^3 }{2}\right).
	\label{eq:omega3_V}
\end{align}
Therefore, by assuming property \ref{ass:good_distances.2} for $3n$ tuples 
$\{(u, a): a \in [n]\}\cup \{(v, b): b \in [n]\} \cup \{(w, c): c \in [n]\}$, it follows that with probability 
at least $1 - \exp\left(- \frac{\delta^2 p \left(\meas(\eta-\Delta) n\right)^3 }{2}\right)$,
\begin{align*}
	|\cEppp_{uvw}|  & = |\{ (a,b,c) \in \cEppp ~: \hat{d}(u,a) < \eta, ~\hat{d}(v, b) < \eta, ~\hat{d}(w, c) < \eta\}| \\
	& \geq |\{ (a,b,c) \in \cEppp ~: {d}(u,a) < \eta-\Delta, ~{d}(v, b) < \eta-\Delta, ~{d}(w, c) < \eta-\Delta\}| \\
	& = | \cEppp \cap \cV_{uvw}| \\
	&\geq (1 - \delta) p \left(\meas(\eta-\Delta) n\right)^3.
\end{align*}
%
%
%
Define the event $\cH = \{  |\cEppp_{uvw}|  \geq (1 - \delta) p \left(\meas(\eta-\Delta) n\right)^3| \}$. It follows that 
\[\Prob{\cH^c} \leq \exp\left(-\frac{1}{2}\delta^2 p \left(\meas(\eta-\Delta) n\right)^3\right).\]
By definition, $F(u,v,w) = f(\theta_u, \theta_v,\theta_w) \in [0,1]$  for all $u, v, w \in [n]$. Therefore, 
\begin{align*}
	\E[(\hat{F}(u,v,w) - f(\theta_u,\theta_v,\theta_w))^2] 
	&\leq \E[(\hat{F}(u,v,w) - f(\theta_u,\theta_v, \theta_w))^2 ~\Big|~\cH ] + \Prob{\cH^c} \\
	&\leq \bias^2(\eta + \Delta) + \frac{1}{(1 - \delta) p \left(\meas(\eta-\Delta) n\right)^3 } \\
	&\quad+ \exp\left(-\frac{1}{2}\delta^2 p \left(\meas(\eta-\Delta) n\right)^3\right).
\end{align*}
We add an additional $3n\alpha_1 + \alpha_2$ in the final MSE bound: $3n \alpha_1$ for violation of 
property \ref{ass:good_distances.2} for any of the $3n$ tuples $\{(u, a): a \in [n]\}\cup \{(v, b): b \in [n]\} \cup \{(w, c): c \in [n]\}$, 
and $\alpha_2$ for violation of property \ref{ass:good_distances.3}.

To obtain the high-probability bound, note that $\Mppp(a, b, c)$ are independent across indices 
$(a, b,c) \in \cEppp$ as well as independent of observations in $\Omega_1$. 
Additionally, the model assumes that $F(a, b, c), \Mppp(a, b, c) \in [0,1]$, and $\E[\Mppp(a, b, c)] = F(a, b, c)$ 
for observed tuples $(a,b,c)$. By an application of Hoeffding's inequality for bounded, 
zero-mean independent variables, for any $\delta' \in (0,1)$ it follows that assuming property 
\ref{ass:good_distances.1}, property \ref{ass:good_distances.2} for $3n$ tuples 
$\{(u, a): a \in [n]\}\cup \{(v, b): b \in [n]\} \cup \{(w, c): c \in [n]\}$, 
and property \ref{ass:good_distances.3} hold, we have
\begin{align*}
	&\Prob{\left. \tfrac{\Big|\sum_{(a,b, c) \in \cEppp_{uvw}} (M(a, b, c) - F(a,b, c)) \Big|}{|\cEppp_{uvw}|}\geq \delta' ~\right|~ \cH} \leq \exp\left(-\delta'^2 (1 - \delta) p \left(\meas(\eta-\Delta) n\right)^3\right).
\end{align*}
Therefore, 
\[ |\hat{F}_{uvw} - f(\theta_u,\theta_v, \theta_w)| \leq \bias(\eta + \Delta) + \delta', \]
with probability at least 
\begin{align*}
	& 1 - \exp\left(-\frac{1}{2}\delta^2 p \left(\meas(\eta-\Delta) n\right)^3\right) -  \exp\left(-\delta'^2 (1 - \delta) p \left(\meas(\eta-\Delta) n\right)^3\right)  - 3n\alpha_1 - \alpha_2.
\end{align*}
This completes the proof of Lemma \ref{lemma:nearest_neighbor}.
\endproof

\begin{lemma}\label{lem:one}
The following inequalities hold: 
\begin{itemize} 
\item[(a)] For any $\rho \geq 2$ and integer $r \geq 1$, 
\begin{align*}
\sum_{s = 1}^r \rho^s \leq 2 \rho^r. 
\end{align*}
\item[(b)] For any $\rho \geq 2$ and non-negative integer $s$,
\begin{align*}
\rho^s & \geq s\rho.
\end{align*}
\item[(c)] Further, if $\exp(-a \rho) \leq \frac12$ for some $a > 0$, then
\begin{align*}
\sum_{s = 1}^r \exp(-a \rho^s) \leq 2 \exp(-a \rho)
\end{align*}
\end{itemize}
\end{lemma}

\proof
To prove (a), note that for any $\rho \geq 2$, 
\begin{align*}
\sum_{s = 1}^r \rho^s &\leq \rho^r \sum_{s = 1}^r \rho^{s-r} ~=~ \rho^r \sum_{s = 0}^{r-1} \rho^{-s} ~\leq~ \rho^r \sum_{s = 0}^{r-1} 2^{-s} ~\leq~ 2 \rho^r.
\end{align*}
To prove (b), first check that it trivially holds for $s = 0$ and $s=1$. The inequality holds for $s = 2$ iff $\rho \geq 2$. The inequality hold for $s$ iff $\rho \geq s^{1/(s-1)}$. We can verify that $s^{1/(s-1)}$ is a decreasing function in $s$, such that if the inequality holds for $s = 2$, it will also hold for $s \geq 2$.
To prove (c), further consider $\exp(-a \rho) \leq \frac12$,
\begin{align*}
\sum_{s = 1}^r \exp(-a \rho^s) & \leq \sum_{s = 1}^r \exp(-a s\rho) \\
&\leq \exp(-a \rho) \sum_{s = 1}^r \exp(-a \rho (s - 1)) \\
& \leq \exp(-a \rho) \sum_{s = 0}^{r-1} \exp(-a \rho s)\\
&\leq \exp(-a \rho) \sum_{s = 0}^{r-1} 2^{- s}\\
&\leq 2 \exp(-a \rho).
\end{align*}
\endproof

\begin{lemma} \label{lem:subexp}
If $\Prob{|X| \geq z} \leq c \exp(- \frac{z^2}{Q})$, then for all $\lambda \in \Reals$,
\[ \E[e^{\lambda X}] \leq \exp(\frac{\lambda^2 \nu^2}{2})\]
with $\nu = \sqrt{\frac{Q (1 + c^2 \pi)}{2}}$.
\end{lemma}

\proof{}

\begin{align*}
\E[e^{\lambda X}]
&= \int_0^{\infty} \Prob{e^{\lambda X} \geq Z} dZ \\
&= \int_{-\infty}^{\infty} \Prob{\lambda X \geq z} e^{z} dz \\
&\leq \int_{-\infty}^{\infty} c \exp(- \frac{z^2}{Q\lambda^2} + z) dz \\
&\leq c \exp(\frac{\lambda^2 Q}{4}) \int_{-\infty}^{\infty} \exp(- \frac{1}{\lambda^2 Q} (z - \frac{\lambda^2 Q}{2})^2)  dz \\
&\leq c \exp(\frac{Q \lambda^2}{4}) \sqrt{\pi \lambda^2 Q}.
\end{align*}
Using the fact that $\sqrt{x} < e^{x/5}$ for all $x \geq 0$, 
it follows that 
\begin{align*}
\E[e^{\lambda X}]
&\leq \exp(\frac{Q \lambda^2}{4} + \frac{c^2 \pi \lambda^2 Q}{4})
\end{align*}
Therefore, for all $\lambda \in \Reals$,
\[ \E[e^{\lambda X}] \leq \exp(\frac{(1+c^2 \pi)Q \lambda^2}{4} ).\]

\endproof

\begin{lemma} \label{lem:ustats}
Let $X_1,\dots, X_n$ be i.i.d. random variables taking values in $\cX$. Let $g: \cX \times \cX \to \mathbb{R}$
be a symmetric function. Consider U-statistics with respect to $g$ of $X_1,\dots, X_n$ defined as
\begin{align}
U & = \frac{1}{{n \choose 2}} \sum_{1\leq i < j\leq n} g(X_i, X_j). 
\end{align}
Let $\|g\|_\infty \leq b$ for some $b > 0$. Then, 
\begin{align}
\Prob{| U - \mathbb{E}[U]| > t} & \leq 2\exp\Big(- \frac{n t^2}{8 b^2}\Big).
\end{align}
\end{lemma}
The proof follows directly from an implication of Azuma-Hoeffding's inequality. For example, 
see \cite[Example 2.23]{wainwrightbook} for a proof.

\end{document}